%% file: main.tex
\documentclass[final]{alt2023}

\usepackage[%
		minnames=1,maxnames=99,maxbibnames=99,minalphanames=1,maxalphanames=99,
		style=alphabetic,
		sorting=none,
		sortcites=false, %
		doi=false,url=false,
		uniquename=init,
		giveninits=true,
		hyperref,natbib,
		backend=biber]{biblatex}
\renewbibmacro{in:}{%
	\ifentrytype{article}{}{\printtext{\bibstring{in}\intitlepunct}}}

\title[Limitations of Information-Theoretic Generalization Bounds]{Limitations of Information-Theoretic Generalization Bounds for Gradient Descent Methods
in Stochastic Convex Optimization}
\usepackage{times}

\usepackage{amsmath,amssymb,amsfonts}
\usepackage{graphicx}

\usepackage[T1]{fontenc}
\usepackage[utf8]{inputenc}
\usepackage{mathscinet}
\input{headers.tex}

\input{defs.tex}

\input{defs-sco.tex}

\newcommand{\SCOprob}{(\parspace,\dataspace,\losscvx)}
\newcommand{\parameterspace}{domain\xspace} %

\newcommand{\param}{parameter\xspace}
\newcommand{\params}{parameters\xspace}

\usepackage[hide=true,setmargin=true,marginparwidth=1in]{marginalia}

\usepackage[capitalize]{cleveref}

\begin{document}

\input{authors.tex}

\maketitle
\renewcommand{\thefootnote}{*}
\footnotetext{Mahdi Haghifam and Borja {Rodríguez-Gálvez} are equal-contribution authors.}
\renewcommand{\thefootnote}{\arabic{footnote}}
\footnotetext{\emph{This version of the paper corrects a mistake in the proof of \Cref{th:ecmi-failining-example}. No changes were required in the main text and the result, intuition, and proof sketch still hold. For details, please see \Cref{app:icmi}.}}

\begin{abstract}
To date, no ``information-theoretic'' frameworks for reasoning about generalization error have been shown to establish minimax rates for gradient descent in the setting of stochastic convex optimization. 
In this work, we consider the prospect of establishing such rates via several existing information-theoretic frameworks:
input-output mutual information bounds, conditional mutual information bounds and variants,
PAC-Bayes bounds, and recent conditional variants thereof.
We prove that none of these bounds are able to establish minimax rates.
We then consider a common tactic employed in studying gradient methods, whereby the final iterate is corrupted by Gaussian noise, producing a noisy ``surrogate'' algorithm. We prove that minimax rates cannot be established via the analysis of such surrogates.
Our results suggest that new ideas are required to analyze gradient descent using information-theoretic techniques.

\end{abstract}

\renewcommand{\XX}{\mathcal X}
\renewcommand{\YY}{\mathcal Y}
\renewcommand{\SS}{S_n}
\newcommand{\SSS}{\inspace}
\newcommand{\TTT}{\mathcal {T}}

\section{Introduction}
\label{sec:introduction}

In this work, we uncover limitations of information-theoretic techniques towards analyzing stochastic gradient descent. To do so, we extend existing information-theoretic frameworks for reasoning about generalization to the setting of stochastic convex optimization (SCO) \citep{shalev2009stochastic}. Despite the resulting bounds being provably tight, we develop an SCO problem in which the mutual information terms underlying these bounds are too large to demonstrate that   subgradient methods \citep{cauchy1847methode,robbins1951stochastic,bubeck2015convex} obtain minimax rates. We also consider the introduction of isotropic Gaussian noise to the final iterate
and demonstrate a fundamental tradeoff between optimization error and expected generalization error that never yields minimax rates. Our results also cast doubt on the effectiveness of using isotropic Gaussian noise to study  subgradient methods  in other settings, such as deep learning.

\sloppy Information-theoretic bounds are, %
by their nature, distribution- and algorithm-dependent. These bounds have shown some promises: for instance, these key properties enable information-theoretic frameworks to achieve numerically non-vacuous generalization guarantees for stochastic gradient Langevin dynamics (SGLD) with modern deep-learning datasets and architectures \citep{negrea2019information,haghifam2020sharpened,li2019generalization,pmlr-v162-banerjee22a}. 
Therefore, it is natural to wonder whether the underlying quantity---%
mutual information---offers a potentially unifying tool %
to reason about generalization.

Information-theoretic techniques have long been used to classify the hardness of learning problems in terms of lower bounds on minimax risk. The development of information-theoretic techniques to upper bound minimax risk is a more recent %
approach.
A %
stream of work has produced a variety of bounds on the generalization error of learning algorithms in terms of the (conditional) mutual information between the training data and certain statistics of the learned predictor. 
In the case of binary classification, the suprema of certain such bounds
match (known) minimax rates, where the suprema runs over data distributions. In the special case of an interpolating classifier achieving zero empirical risk, the risk is shown to equal a certain mutual information term and to be controlled---for polynomial or slower rates---by upper bounds obtained by conditioning \citep{haghifam2022isit}.

Despite these successful applications, much less is known about the optimality or limitations of these techniques beyond the setting of binary classification and 0--1 valued loss.
In this work, we turn our attention to stochastic convex optimization (SCO), a well-studied setting with known minimax rates, and look in particular at the analysis of stochastic gradient methods like stochastic gradient descent (SGD). 
In contrast to learning with a 0--1 valued loss, the minimax excess risk cannot be characterized in terms of uniform convergence of the generalization error \citep{shalev2010learnability}. 

To start, we develop a tight information-theoretic bound for SCO problems, analogous to those developed for classification. We focus on the convex--Lipshitz--bounded (CLB) subclass of SCO learning with gradient descent (GD). 
Our main result demonstrates that despite the bound being tight, it cannot achieve known minimax rates in the CLB setting for GD.

Next, we investigate whether the gap arises due to GD’s deterministic nature: \emph{can we close the gap by introducing randomness?} 
In other words, can we find a \emph{``surrogate'' algorithm} with good information-theoretic generalization guarantees, and such that this surrogate algorithm is close in generalization to the original one?
Such an approach was formalized in \citep{negrea2020defense,pmlr-v178-sefidgaran22a,bu2021population}, and %
appears frequently in the generalization literature, e.g.~\citep{neu2021information,wang2021generalization,harutyunyan2021information,LanCar2002,hellstrom2020nonvacuous,GR17,dziugaite2021role,neyshabur2018pac,zhang2020spread,chatterji2019intriguing,foret2020sharpness,wu2020adversarial,pour2022benefits}.
The most commonly-used surrogate is a ``Gaussian surrogate", which perturbs the output of the algorithm by adding a Gaussian random variable.
Surprisingly, we show that the limitations of information-theoretic analyses in the SCO setting are not eliminated even under the Gaussian surrogate. 

Our negative results for Gaussian surrogates cast some doubt on their use to study SGD in other settings, such as deep learning.
Information-theoretic techniques have shown some promise in this setting. 
Building off the seminal work of \citet{PensiaJogLoh2018},
information-theoretic generalization bounds were shown to yield numerically nonvacuous estimates for stochastic gradient Langevin dynamics (SGLD) when applied to optimizing overparametrized neural networks on nontrivial deep learning classification benchmarks \citep{negrea2019information,haghifam2020sharpened,li2019generalization,pmlr-v162-banerjee22a}.
And while existing information-theoretic techniques seemingly cannot be applied directly to stochastic gradient methods like SGD itself, \citet{neu2021information} showed how to obtain a (suboptimal) generalization bound for SGD using an information-theoretic bound for a noisy ``surrogate''  learning algorithm, designed to track the behavior of SGD. Our results
explain the suboptimality of this approach and motivate work understanding the power or limitations of other surrogates.

\subsection{Contributions}
\begin{enumerate}
    \item 
     We prove \emph{tight} generalization bounds based  on the input-output mutual information (IOMI) of \citet{RussoZou16} and \citet{XuRaginsky2017} and the conditional mutual information (CMI) of \citet{steinke2020reasoning}  for CLB subclasses of SCO problems, as well as their individual sample variations~\citep{negrea2019information,haghifam2020sharpened,bu2020tightening,rodriguez2020random,zhou2020individually} and evaluated CMI~\citep{steinke2020reasoning}.  Our generalization bounds may be of independent interest and can be used to obtain distribution- and algorithm-dependent generalization bounds for SCO problems beyond the worst-case guarantees.
    
    \item  We investigate whether we can \emph{directly} analyze the 
    generalization of GD with our information-theoretic generalization bounds.
    We provide a negative answer to this question by showing that neither the CMI nor IOMI frameworks can properly characterize the excess risk of GD in SCO problems in the minmax setting. We also extend our negative results to the alternative variations of IOMI and CMI, such as evaluated CMI \citep{steinke2020reasoning}, and individual sample bounds \citep{negrea2019information,haghifam2020sharpened,bu2020tightening,rodriguez2020random,zhou2020individually}. 
    \item We consider a surrogate algorithm based on a Gaussian perturbation of the final iterate of GD. %
    We show that the generalization of GD can be decomposed as the sum of the generalization of %
    the perturbed final iterate and a residual term that captures the sensitivity of the loss function to perturbations around %
    such iterate. 
  We consider a favorable setting where the parameters of the surrogate can be tuned based on the data distribution. 
  Nevertheless, we show that %
  there exists a sequence of CLB %
  problems that can be learned with GD but IOMI and CMI frameworks fail to capture learnability in the minimax sense. Our construction is inspired from the ideas by \citet{amir2021sgd} but with a completely different analysis.
  
 \item We complement our results by showing that our construction also implies the failure of high-probability PAC-Bayes bounds in characterizing learnability of the CLB subclass of SCO problems using GD in the minimax sense. 
 In particular, we prove that the \emph{classical} PAC-Bayes bound of \citet{McAllester1999} and the recently proposed \emph{conditional} PAC-Bayes bound of \citet{grunwald2021pac} are \emph{vacuous} in the minimax sense.
\end{enumerate}

\subsection{Related Work}

Recently, there has been a significant interest in understanding whether %
information-theoretic generalization bounds can characterize worst-case (minimax) rates for certain learning problems. 
For  binary classification, \citet{bassily2018learners,roishay} show that the IOMI and classical PAC-Bayes frameworks of \citep{RussoZou16,XuRaginsky2017,McAllester1999} provably fail to characterize the learnability of Vapnik--Chervonenkis  classes for which we have strong generalization guarantees. 
Then, \citet{steinke2020reasoning,grunwald2021pac,haghifam2021towards} show that CMI \citep{steinke2020reasoning} can be used to establish optimal bounds in the realizable setting. 
The results of \citep{steinke2020reasoning,grunwald2021pac,haghifam2021towards} show that existing IT bounds \emph{characterize} the minimax rates, \emph{without the need for surrogates}. See also \citep{harutyunyan2022formal,9834457,nokleby2021information}.
Our work is different from the prior work since we study limitations of information-theoretic generalization bounds in the context of gradient descent methods. Moreover, our results indicate that existing techniques fail to characterize the minimax rates for gradient descent methods in SCO problems. Our findings stand in stark contrast to the success of information-theoretic frameworks in capturing the learnability of VC classes.

\section{Preliminaries}

\subsection{Probability and Information Theory Notation}
\label{app:notation-prob-it}

Let $P,Q$ be probability measures on a measurable space.
For a $P$-integrable function $f$, let $P[f] = \int f \dee P$.
When $Q$ is absolutely continuous with respect to $P$, denoted $Q \ll P$, we
write $\rnderiv{Q}{P}$ for (an arbitrary version of) the 
Radon--Nikodym derivative (or density) of $Q$ with respect to $P$. 
The \defn{KL divergence} (or \defn{relative entropy}) of \defn{ $Q$ with respect to $P$},
denoted $\KL{Q}{P}$, is defined as $Q[ \log \rnderiv{Q}{P} ]$ when $Q \ll P$ and as
infinity otherwise.

For a random element $X$ in some measurable space $\SSS$,
let $\Pr[X]$ denote its distribution, which lives in the space
$\ProbMeasures{\SSS}$ of all probability measures on $\SSS$.
Given another random element, say $Y$ in $\mathcal{Y}$,
let
$\cPr{Y}{X}$ denote the conditional
distribution of $X$ given $Y$ (or, more formally, the $\sigma$-algebra induced by $Y$).
If $X$ and $Y$ are independent, denoted by $X\indep Y$, we have $\cPr{Y}{X} = \Pr[X]$ almost surely (a.s.). Moreover, we write $\cPr{Z}{(X,Y)}$ for the conditional distribution of the pair $(X,Y)$
given a random element $Z$. 
For an event, say $X \in A$, 
$\cPr{Y}{X \in A}$ denotes the event's conditional probability
given $Y$, 
which is defined to be the conditional expectation of the indicator
random variable $\indic{X\in A}$ given $Y$,
denoted $\cEE{Y}{\indic{X \in A}}$.
By the law of total expectation (a.k.a. chain or tower rule), $\EE \cEE{\cF} = \EE$ for any $\sigma$-algebra $\cF$.

The \defn{mutual information between $X$ and $Y$} 
is $
\minf{X;Y} = \KL{\Pr[(X,Y)]} { \Pr[X] \otimes \Pr[Y]},$
where $\otimes$ forms the product measure. Then,
the \defn{disintegrated mutual information between $X$ and $Y$ given $Z$} is $
 \dminf{Z}{X;Y} =  \KL{\cPr{Z}{(X,Y)}}{\cPr{Z}{X} \otimes \cPr{Z}{Y} }$, and the conditional mutual information is $
 \minf{X;Y\vert Z}=\EE{[\dminf{Z}{X;Y}]}$. 

 Let $\mu = \Pr[X]$ and let $\kappa(Y) = \cPr{Y}{X}$ a.s. 
If $X$ concentrates on a countable set $V$ with counting measure $\nu$, 
the \defn{(Shannon) entropy of $X$} is 
$
\entr{X}=- \mu [ \log \rnderiv{\mu}{\nu}] = - \sum_{x\in V} \Pr(X=x)\,\log \Pr (X=x )
$. The \defn{disintegrated entropy of $X$ given $Y$} is defined by 
$
\centr{Y}{X} = -\kappa(Y)[ \log \rnderiv{\kappa(Y)}{\nu} ],
$
while the \defn{conditional entropy of $X$ given $Y$} is $\entr{X\vert Y} =  \EE[\centr{Y}{X}]$.
Note that $\entr{X\vert Y}\leq \entr{X}$ \citep{cover2012elements}.

\subsection{Stochastic Convex Optimization}
\label{sec:sco}

A \defn{stochastic convex optimization} (SCO) problem is a triple $\SCOprob$, where $\parspace \subseteq \Reals^d$ is a convex set and $\losscvx(\cdot,z) : \parspace \to \Reals$ is a convex function for every $z \in \dataspace$ \citep{shalev2009stochastic}. Informally, given an SCO problem $\SCOprob$, the goal is to find an approximate minimizer 
of the %
\defn{population risk} 
$$\Popriskcvx{w} := \EE_{Z \sim \Dist} [\losscvx(w,Z)],$$
given only an i.i.d.\ sample $\trainset = \{Z_1, \ldots, Z_n\}$ drawn from an unknown distribution $\Dist$ on $\dataspace$. 

\fTBD{Mh: $\Dist_n$ later used for something else. I don't think it is necessary for this part.}
The \defn{empirical risk} of $w \in \parspace$ on a sample $\trainset\in \dataspace^{n}$ is 
$\Empriskcvx{w} := \frac{1}{n} \sum_{i \in [n]} f(w,Z_i)$,
where $\range{n}$ denotes the set $\{1,\dots,n\}$.
A \defn{learning algorithm} is a sequence 
$\Alg = (\Alg_n)_{n\ge 1}$ such that,
for every positive integer $n$, $\Alg_n$ maps $\SS$ to a (potentially random) element $W=\Alg_n(\SS)$ in $\parspace$. 
The \defn{expected generalization error} of $\Alg_n$ under $\Dist$ is
$
\nonumber
 \EGE = \EE \sbra[0]{\Popriskcvx{\Alg(\SS)}- \Empriskcvx{\Alg(\SS)} }.
$

We refer to $\parspace$ as the \defn{\parameterspace},
to its elements as \params, \fTBD{DR: Fix me with actual nomenclature that we use.}
to elements of $\dataspace$ as data,
and to $f$ as the \defn{loss function}.

Let $\mathcal{L}$ denote the class of all SCO problems. 
A subclass $\mathcal{C} \subseteq \mathcal L$ is \emph{learnable}
if, for every desired accuracy $\epsilon > 0$
and all sufficiently large number of samples $n$, 
$$
\underbrace{\sup_{\SCOprob \in \mathcal C} 
\inf_{\Alg \vphantom{\ProbMeasures{\dataspace}}} \sup_{\Dist \in \ProbMeasures{\dataspace}} \EE[\EER]}_{\text{minimax (expected) excess risk}} < \epsilon,
$$
where the infimum runs over algorithms.\footnote{Note that the initial $\sup \inf$ is, by skolemization, equivalent to $\inf \sup$, where now the algorithm takes as input both a description of the SCO problem $\SCOprob$ and the data $\SS$. We have chosen this presentation for simplicity.}

In general, the class $\mathcal{L}$ itself is not learnable~\cite[Chapter~12]{shalev2014understanding}. One important family of subclasses of $\mathcal{L}$ which are known to be learnable are the convex--Lipschitz--bounded (CLB) subclasses of SCO problems where, for constants $L,R \in (0,\infty)$, the loss function $\losscvx(\cdot,z)$ is $L$-Lipschitz for all data instances $z \in \dataspace$, and the \parameterspace $\parspace$ is closed and %
has finite diameter $R$ \citep[Chapter~12]{shalev2014understanding}. 
We denote each such class of SCO problems by $\clr$.  In the remainder of the paper, we assume, without loss of generality, that each such $\parspace$ satisfies $\parspace \subseteq \lbrace w : \lVert w \rVert_2 \leq R \rbrace$.

Let $W^\star_{\trainset}$ denote an arbitrary \emph{empirical risk minimizer} (ERM), i.e., an element of $\argmin_{w \in \parspace} \Empriskcvx{w}$.
Then, the expected excess risk,
$\EE[\Popriskcvx{\Alg_n(\trainset)}-\Popriskcvx{w^\star}]$,
can be written as the sum
\begin{align*}
 \EGE + \EE \big[ \Empriskcvx{\Alg_n(\trainset)} - \Empriskcvx{W^\star_{\trainset}} \big] + \EE \big[  \Empriskcvx{W^\star_{\trainset}} -  \Popriskcvx{w^\star}\big],
\end{align*}
of the  \emph{expected generalization error}, 
\emph{optimization error}, and \emph{approximation error}, respectively.

The third term satisfies $\EE[\Empriskcvx{W^\star_{\trainset}} -  \Popriskcvx{w^\star}]=\EE[\Empriskcvx{W^\star_{\trainset}} -  \Empriskcvx{w^\star}] \leq 0$ because $W^\star_{\trainset}$ is an ERM for the training set $\trainset$, and $w^\star$ is a constant. 
Thus, it often suffices to characterize the expected generalization error and optimization error to obtain tight control of the excess risk. 
For approaches based on iterative optimization, the optimization error
can, %
in many cases, be bounded by a convergence analysis \citep{bubeck2015convex}.  
Therefore, the problem of controlling expected excess risk %
frequently amounts to controlling the expected generalization error.
Nonetheless, there exist scenarios where the excess risk can vanish while the optimization and generalization errors do not, as shown in~\citep{koren2022benign} for some CLB problems learned with stochastic gradient descent (SGD).
	
CLB subclasses can be generically learned by suitably tuned instances of (projected) gradient descent (GD), a long studied algorithm~\citep{cauchy1847methode,bubeck2015convex}: 
For a convex and compact subset $\parspace \subseteq \Reals^d$, let $\proj : \Reals^d \to \parspace$ denote the Euclidean projection operator, given by $\proj(x)=\argmin_{y\in \parspace}\norm{y-x}_2$.
The GD algorithm, $\gdalgnodata = (\gdalgnodata_n)_{n \geq 1}$, is initialized at some feasible point $\ww_0 \in \parspace$ and then, for some number $T$ of iterations, proceeds to update the parameters iteratively according to 
$\ww_{t+1} = \proj \big(\ww_{t} - \eta_t g_t \big)$, where 
$\eta_t$ is a suitably chosen step-size and
$g_t\in \partial \Empriskcvx{\ww_t}$ is an element of the subdifferential of $\Empriskcvx{\ww_t}$.
While there are several variants, we will focus on the case where 
the output of the algorithm is the final iterate, 
i.e., $\gdalg = \ww_T$.

\subsection{Excess Risk of Gradient Descent}
For simplicity, we restrict the discussion to GD with a constant step size,
i.e., $\eta_t = \eta$ for all iterations $t \in \range{T}$.  
We present
known generalization and optimization error bounds for the CLB setting. 

In \citep{lastiterate},
the optimization error of the final iterate of GD in the CLB setting is shown to satisfy
\[
\label{eq:opt-error-gd}
\sup_{\SCOprob \in \clr}\sup_{\Dist \in \ProbMeasures{\dataspace}}\EE \big[ \Empriskcvx{\gdalg} - \Empriskcvx{W^\star_{\trainset}} \big] \leq \frac{R^2}{2\eta T}+ \frac{(\log(T)+2)\eta L^2}{2}.
\]
(See \cref{lem:gd-last-iterate} for a re-statement of this result in the context of the present paper).
A similar result also appears in \citep[][Thm.~5.3]{zhang2004solving}.
Recently, \citet[Thm.~3.2]{bassily2020stability} proved a generalization bound for GD,
\[
\label{eq:gen-error-gd}
\sup_{\SCOprob \in \clr} \sup_{\Dist \in \ProbMeasures{\dataspace}}\EGEgd \leq 4L^2 \sqrt{T} \eta + \frac{4L^2T\eta}{n}.
\]

Together, \Cref{eq:opt-error-gd,eq:gen-error-gd} yield the following bound on the excess risk,
\begin{small}
\begin{align}
\sup_{\SCOprob \in \clr} \sup_{\Dist \in \ProbMeasures{\dataspace}}&\EE[\Popriskcvx{\gdalg}-\Popriskcvx{w^\star}] \leq  4L^2 \eta\left(\sqrt{T} +  \frac{T}{n}\right) + \frac{R^2}{2\eta T}+ \frac{(\log(T)+2)\eta L^2}{2}.
\label{eq:excess-risk-gd}
\end{align}
\end{small}
For all $\alpha \geq 2$, \Cref{eq:excess-risk-gd} guarantees that GD achieves an excess risk in $\bigO{\nicefrac{LR}{\sqrt{n}}}$ for a number of iterations $T \in \bigTheta{n^\alpha}$ and a step-size $\eta \in \bigTheta{\nicefrac{R\sqrt{n}}{Ln^\alpha}}$. 
This, in fact, is the best achievable excess risk rate for the class $\clr$ in the distribution-free setting \citep{bubeck2015convex}.
In \citep{amir2021sgd,sekhari2021sgd}, it is shown that GD cannot attain this excess risk rate when the number of iterations satisfies $T \in o\mleft(n^2\mright)$.

\section{Main Questions and Overview of the Results}
\label{sec:problem-def}
The generalization error guarantee for GD in \cref{eq:gen-error-gd} is obtained using the algorithmic (uniform) stability framework of \citet{bousquet2002stability}. (Prior work  \citep{HardtRechtSinger2016} also relied on algorithmic stability.)
As shown above, a particular choice of the GD hyperparameters yields expected generalization error in $\bigO{\nicefrac{LR}{\sqrt{n}}}$. 
In this paper, we want to understand whether %
the same rate can be achieved using an information-theoretic framework for generalization. 
\emph{Are information-theoretic frameworks for generalization expressive enough to accurately estimate the generalization error of GD for SCO?} 

We begin by focusing on two frameworks for measuring the information complexity of a learning algorithm: 
\emph{input-output mutual information} (IOMI \citep{XuRaginsky2017,RussoZou15,RussoZou16}) and 
\emph{conditional mutual information} (CMI \citep{steinke2020reasoning}).
The IOMI of an algorithm $\Alg_n$ with respect to a data distribution $\Dist$, denoted $\iomi$, is defined to be the mutual information $\minf{\Alg_n(\SS);\SS}$ between the training data $\SS$ and the output of the algorithm, $\Alg_n(\SS)$. 
In order to define the CMI framework, consider $n \in \Naturals_{+}$ training data, let $U=(U_1,\dots,U_n) \dist \unif{\{0,1\}^{n}}$, and let $\supersample=(\tilde{Z}_{i,j})_{i \in \{0,1\},j\in \{n\}} \dist \Dist^{\otimes (2\times n)}$ be a $2\times n$ array of \iid~random elements in $\dataspace$, independent from $U$. 
Then $\supersample_U=(\tilde{Z}_{U_i,i})_{i=1}^{n}$ has the same distribution as $\SS$, and so we may assume, w.l.o.g., that $\SS=\supersample_U$ a.s. The CMI of the algorithm $\Alg_n$ with respect to the data distribution $\Dist$, denoted $\cmi$, is defined to be 
the conditional mutual information $\minf{\Alg_n(\SS);U\vert \supersample}$ between $\Alg_n(\SS)$ and $U$ given $\supersample$. In the remainder of the section, we write $\mi$ to refer to both $\iomi$ and $\cmi$.

As the first step towards answering our main question, we develop new generalization bounds in both the IOMI and CMI frameworks to handle the CLB subclass of SCO, and show that our upper bounds are tight. 
Existing information-theoretic generalization bounds often depend on properties of the loss function $f(w,z)$ for fixed $w\in \parspace$. 
For instance, the generalization bounds in \citep{XuRaginsky2017,BuZouVeeravalli2019,negrea2019information} depend on the tail of the random variable $f(w,Z)$ when $Z \sim \Dist$. 
In SCO, we often have no such control,
making it impossible to reason about these problems using existing generalization bounds.
Instead, in SCO, it is common for loss functions $f(w,z)$ to have regularity for fixed $z\in \dataspace$. 
In \cref{lemma:mi-and-cmi-bound-lipschitz}, we develop 
 new information-theoretic generalization bounds 
for the \clr subclass, proving that
$
\label{eq:gen-bounds-summary}
         \EGE \leq  \bigO{LR\sqrt{\nicefrac{\mi}{n}}}.
$
In  \cref{thm:bounds-tight}, we show that our bound is \emph{tight} up to constants.

Having obtained $\mi$ bounds for SCO problems, we ask whether 
they capture the generalization properties of GD well enough to obtain minimax rates.
In \cref{sec:failure-perturbation}, we provide a negative answer to this question, proving that for sufficiently large $n$
\[
\label{eq:gd-no-noise}
\sup_{\SCOprob \in \clr} \sup_{\Dist \in \ProbMeasures{\dataspace}} \migd  \in \Omega(n),
\]
which implies that neither the CMI nor IOMI frameworks can properly characterize minimax excess risk of GD in SCO problems. 
In \cref{sec:failure_alternatives}, we study variations of IOMI and CMI, such as evaluated CMI \citep{steinke2020reasoning} and individual-sample bounds \citep{bu2020tightening,negrea2019information,haghifam2020sharpened,rodriguez2020random,zhou2020individually}. We find that they also fail to characterize the generalization of GD algorithm. 

Since a direct analysis of GD via $\mi$ is not possible, we consider a ``surrogate'' analysis \citep{negrea2020defense}, where an excess risk bound is obtained by comparing the risk of GD to a different (surrogate) algorithm, for which one can obtain generalization guarantees. In other words, \emph{is GD ``close'' to an algorithm with small information complexity?}

We consider commonly used surrogate,
whereby one perturbs the final iterate by a Gaussian random variable (see, e.g.,  \citep{hellstrom2020nonvacuous,GR17,wang2021generalization,neu2021information,dziugaite2021role,neyshabur2018pac}).
More formally, let $\Alg_n(\trainset)=\sw$, where $\sw = \proj(\ww_T + \xi)$,  $W_T = \mathrm{GD}_n(\trainset)$, $\xi \sim \Normal(0,\sigma^2 \id{d})$, and $\xi \indep \trainset$. The generalization of GD can be related to the generalization of the Gaussian surrogate $\Alg_n$ via the inequality
\[
\label{eq:main-decompos-surrogate}
\EGEgd \leq \EGE + \EE[\respop(W_T)] + \EE[ \resemp(W_T)], \quad \text{where}
\]
\[
\label{eq:def-residual}
\respop(\ww_T) = \cEE{\trainset}{\sbra{|\Popriskcvx{\sw}-\Popriskcvx{\ww_T}|}} \ \ \textnormal{and} \ \ 
\resemp(W_T) = \cEE{\trainset}{\sbra{|\Empriskcvx{\sw}-\Empriskcvx{\ww_T}|}}
\]
are referred to as
\emph{residual terms} in the sequel.
In \cref{eq:def-residual}, the conditional expectations (given $\trainset$) marginalize over only the (independent) randomness of $\xi$.  
Intuitively, the residual terms measure the sensitivity of the population and empirical loss landscapes~\citep{neu2021information}. 
The sensitivity is measured around $W_T$ to perturbations by an isotropic Gaussian random vector with variance $\sigma^2$.

\begin{remark}
    In \cref{eq:main-decompos-surrogate}, one can drop the absolute values from the residual terms, i.e, second and third terms, to obtain
\[
\label{eq:bad-formulation}
\EGEgd = \EGE +  \EE\sbra{\Popriskcvx{\ww_T}-\Popriskcvx{\sw}} + \EE\sbra{\Empriskcvx{\sw}-\Empriskcvx{\ww_T}}.
\]
In this remark, we want to demonstrate how tautologies can arise if one directly studies \cref{eq:bad-formulation} instead of \cref{eq:main-decompos-surrogate}. 
Consider a surrogate that simply outputs a fixed \param from $\parspace$ (independent of the training set). For instance, let $\Alg_n(\trainset)=0$ ($\sw=0$). 
Then $\mi=0$ and $\Empriskcvx{\sw}=\Popriskcvx{\sw}$. Therefore, \cref{eq:bad-formulation} in this case is simplified to
\[
\EGEgd &= \EE\sbra{\Popriskcvx{\ww_T}} - \EE\sbra{\Empriskcvx{\ww_T}}  = \EGEgd, \label{eq:taut-bad-formulation}
\]
taking us back to the original problem. 

Next, we argue that even if we restrict the surrogate algorithm to the case of perturbation by Gaussian random variable, i.e., $\Alg_n(\trainset)=\sw$ where $\sw = \proj(\ww_T + \xi)$, we get an equally tautological statement from the decomposition in \cref{eq:bad-formulation}. In particular, we claim that letting $\sigma\to \infty$ takes us back to the original problem. Consider the IOMI framework. Since $\parspace$ is bounded, we have $\Var(\ww_T) \in \bigO{1}$. Then, using \citep[Thm.~4.6]{polyanskiy2014lecture} and the data-processing inequality for mutual information, we obtain $\minf{\proj(\ww_T + \xi);\trainset}\leq \minf{\ww_T+\xi;\trainset}\leq \minf{\ww_T+\xi;\ww_T} \in \bigO{1/\sigma^2}$ which tend to 0 as $\sigma$ diverges. Also, as $\sigma\to \infty$,  $\EE[\Empriskcvx{\sw}]\approx \EE[\Popriskcvx{\sw}]$. Therefore, in the case that $\sigma\to \infty$, by simplifying \cref{eq:bad-formulation}, we arrive at the same tautology as in \cref{eq:taut-bad-formulation}. Since $\cmi \leq \iomi$ for any learning problem  \citep[Thm.~2.1]{haghifam2020sharpened}, we have the same tautology even if we use the CMI framework.
\end{remark}

In this Gaussian surrogate setting, the question of whether $\mi$ bounds characterize $\EGE$ is equivalent to asking whether
\[
\label{eq:desired-result}
\sup_{\SCOprob \in \clr} \sup_{\Dist \in \ProbMeasures{\dataspace}}  \inf_{\sigma \geq 0}  \bigg \{  L R\sqrt{\frac{\mi}{n}} +  \EE[\respop(W_T)]  + \EE[ \resemp(W_T)] \bigg \}  \stackrel{?}{\in} \Theta \Big(\frac{LR}{\sqrt{n}} \Big).
\]
Answering this amounts to answering whether one can choose a value of $\sigma$
 \emph{with a full knowledge of the SCO problem and the data distribution}, such that the perturbed GD algorithm achieves the optimal rate via the generalization bound appearing in \cref{lemma:mi-and-cmi-bound-lipschitz}.
Alternatively, one can ask whether one can show, using the perturbation idea, that GD learns the subclass $\clr$ 
\emph{even with an arbitrary slow rate}, i.e., whether or not the LHS of \cref{eq:desired-result} converges to zero as the  number of the training samples diverges.

In order to gain insight on the Gaussian surrogate, consider extreme values of the variance of the perturbations.
Setting $\sigma = 0$ corresponds to a direct analysis of GD, and the result in \cref{eq:gd-no-noise} shows we cannot prove learnability using existing frameworks.
At the other extreme, one can show that $\mi \to 0$ as $\sigma\to \infty$, leaving us with a bound in terms of the sum of the residual terms alone. 
As the distance between $\sw$ and $\ww_T$ is maximal under such a perturbation, the sum of the residual terms is in $\Omega(1)$, once again failing to establish learnability.
The idea
behind introducing the surrogate algorithm $\Alg$ and adjusting the value of $\sigma$ is that it allows one to conceptually interpolate between these two extreme points in order to find an optimal bound on
GD's generalization error.%

Nevertheless, for the perturbed GD, we prove a negative result showing that 
\[
\label{eq:main-result-paper}
\sup_{\SCOprob \in \clr} \sup_{\Dist \in \ProbMeasures{\dataspace}}  \inf_{\sigma \geq 0}\bigg \{ L R\sqrt{\frac{\mi}{n}} +  \EE[\respop(W_T)] + \EE[ \resemp(W_T)] \bigg \} \in\Omega(1).
\]
Note that our negative result holds even if the perturbation's variance is allowed to depend on the data distribution $\Dist$ and the SCO problem $\SCOprob$.
While the distribution is unknown, the surrogate algorithm is a theoretical device and can be chosen with full knowledge of the data distribution to achieve the tightest possible bound. As such, we must control also the infimum.

In \cref{sec:pac-bayes}, we extend our results to PAC-Bayes bounds, which provide tail bounds on the generalization error of GD, with respect to the randomness in the data.
A similar surrogate decomposition as in \cref{eq:main-decompos-surrogate} relates \emph{disintegrated} generalization of GD to the generalization of $\Alg_n$ via
\[
\cEE{\trainset} \sbra{\Popriskcvx{\ww_T}- \Empriskcvx{\ww_T}} &\leq \cEE{\trainset} \sbra{\Popriskcvx{\sw}- \Empriskcvx{\sw}} + \resemp(\ww_T) + \respop(\ww_T),  \label{eq:pac-bayes-surrogate}
\]
where $\resemp(\ww_T)$ and $\respop(\ww_T)$ are defined in \cref{eq:def-residual}. 
The first term on the RHS of \cref{eq:pac-bayes-surrogate} can be analyzed using PAC-Bayes frameworks (see, e.g., \citep{LanCar2002,hellstrom2020nonvacuous,GR17,dziugaite2021role,neyshabur2018pac,chatterji2019intriguing,foret2020sharpness,wu2020adversarial}). 
In \cref{sec:pac-bayes}, we show
that, in the minimax sense, the classical and conditional PAC-Bayes frameworks of \citet{McAllester1999} and \citet{grunwald2021pac} provide a vacuous characterization of the RHS of \cref{eq:pac-bayes-surrogate} for all values of $\sigma$.

\section{Information-Theoretic Generalization Bounds for the CLB setting}
\label{sec:gen-bounds}
In SCO problems, generalization bounds for gradient methods can be obtained using the uniform stability framework \citep{HardtRechtSinger2016,bassily2020stability,feldman2018generalization,feldman2019high}. 
This framework provides an algorithm-dependent %
approach that has been used to obtain relatively strong generalization bounds for several convex optimization algorithms in the distribution-free setting. In this section, we extend the CMI and IOMI frameworks to the CLB setting and  provide \emph{algorithm-} and \emph{distribution-} dependent generalization bounds. 
\begin{theorem}
\label[theorem]{lemma:mi-and-cmi-bound-lipschitz}
Let $n \in \Naturals$, $\Dist \in \ProbMeasures{\dataspace}$ be a data distribution, and $\trainset \dist \Dist^{\otimes n}$. %
Consider an SCO problem $(\losscvx, \parspace, \dataspace) \in \clr$.
Then, for every learning algorithm $\Alg_n$ such that $\Alg_n(\trainset)\in \parspace$ \as, %
\begin{align*}
    \EGE \leq LR \sqrt{\frac{2\iomi}{n}} \quad \text{and} \quad \EGE \leq LR \sqrt{\frac{8\cmi}{n}}.
\end{align*}
\end{theorem}

The proof, based on~\citep{rodriguez2021tighter}, is in~\Cref{sec:proofs_ege_it}.
To better contextualize our generalization bounds in %
\cref{lemma:mi-and-cmi-bound-lipschitz}, we study their tightness.
For the trivial case where the output of a learning algorithm is independent of the training set, the %
bounds in
\cref{lemma:mi-and-cmi-bound-lipschitz} are tight. 
The theorem below states that the bounds are tight even when the learning algorithm depends on the training set.
\begin{theorem}
\label[theorem]{thm:bounds-tight}
 For every $n \in \Naturals, L \in \Reals_+, R\in \Reals_+$, there exists an SCO problem $(f, \parspace, \dataspace) \in \clr$, a data distribution $\Dist$ over $\dataspace$, and a learning algorithm $\Alg=(\Alg_n)_{n\geq 1}\in \parspace$ such that: (i) the expected generalization error of $\Alg_n$ satisfies $\EGE \geq\nicefrac{LR}{\sqrt{2n}}$, and (ii) %
 the upper bounds from %
    \cref{lemma:mi-and-cmi-bound-lipschitz} are $\EGE \leq \nicefrac{LR\sqrt{2}}{\sqrt{n}}$ and $\EGE \leq \nicefrac{LR\sqrt{8}}{\sqrt{n}}$, respectively.
\end{theorem}

See \cref{pf:bounds-tight} for the proof, which is inspired by~\citep[Sec. 5.1]{orabona2019modern}. 
\Cref{thm:bounds-tight} shows that there exists a learning algorithm in the CLB setting for which the bounds \cref{lemma:mi-and-cmi-bound-lipschitz} is tight. This implies that the bound in \cref{lemma:mi-and-cmi-bound-lipschitz} cannot be \emph{uniformly} improved for every learning algorithm in the CLB setting. Note, however, that there may exist a tighter bound for some learning algorithms.

\renewcommand{\Dist}{\mathcal{D}_{(n)}}
\section{Failure of Information-Theoretic Bounds for GD in the CLB Setting}
\label{sec:failure-perturbation}
An important feature of GD for SCO problems is that the sample complexity is %
\emph{dimension-independent}:
 For every SCO problem in $\clr$, if $L$ and $R$ do not grow with the \params' (ambient) dimension,  %
one needs $\bigO{1/\epsilon^2}$ samples to reach $\epsilon$ expected excess risk using GD, regardless of the dimension. 
In this section, we exploit this property to show that the \emph{distribution-free learnability} of SCO in the CLB setting using GD cannot be explained using the IOMI or CMI frameworks.

Let $\text{GD}(\trainset,\eta,T)$ denote the output of gradient descent, training on the training set $\trainset$ with learning rate $\eta$ for $T$ iterations, starting from a zero initialization.
\begin{theorem}
\label[theorem]{thm:gd-noisy}
Let $n\in \Naturals$, $T_{(n)}=2n^2$, $\eta_{(n)}=\frac{1}{n\sqrt{5n}}$, and $d_{(n)}=3T2^n/4$. 
Then, there exists a universal constant $N^\star \in \Naturals $ such that for every $n\geq N^\star$, there exist a sequence of SCO problems $\{(f_{(n)},\parspace_{(n)},\dataspace_{(n)})\in \mathcal{C}_{4,1}\}_{n\in \Naturals}$ where $\parspace_{(n)},\dataspace_{(n)} \in \Reals^{d_{(n)}}$, 
and a data distribution $\Dist$ over $\dataspace_{(n)}$ %
such that the following holds: 
For $\trainset \dist \Dist^{\otimes n}$, let $\ww_T = \text{GD}(\trainset,\eta_{(n)},T_{(n)})$ and $\Alg_n(\trainset) = \proj(\ww_T+\xi)$, where $\xi \dist \Normal(0,\sigma^2\id{d_{(n)}})$. 
Then, there exists $\mathsf{var}^\star_n > 0 $ such that if $\sigma^2\leq \mathsf{var}^\star_{(n)}$, then
$
\nonumber 
\iomi \in \Omega(n^3) ~ \text{and} ~ \cmi \in \Omega(n).
$
Also, if $\sigma^2>\mathsf{var}^\star_{(n)}$, then
$
\nonumber
 \EE[\resemp(\ww_T)]+ \EE[\respop(\ww_T)] \in \Omega(1).
$
As a result, %
\[
\nonumber
\inf_{\sigma\geq 0} \bigg\lbrace \sqrt{ \frac{ \min\{2 \iomi, 8 \cmi\} }{n} } +  \EE[\resemp(\ww_T)]+ \EE[\respop(\ww_T)] \bigg\rbrace \in \Omega(1),
\]
while the  generalization error of GD satisfies
$
\EE[|\Popriskcvx{\ww_T}-\Empriskcvx{\ww_T}|] \in \bigO{1/\sqrt{n}}.
$
\end{theorem}

\renewcommand{\Dist}{\mathcal D}

\begin{proof}
Here, we provide an overview of the proof. 
The formal proof can be found in \cref{pf:main-iomi}. Our construction is inspired from the construction in \citet{amir2021sgd}. 
\begin{itemize}[leftmargin=*]
    \item \textbf{Construction and Dynamics of GD}: 
Let $d \in \Naturals$ and $\dataspace=\{0,1\}^d$. Let the data distribution on input be $(\bernoulli(1/2))^{\otimes d}$, i.e., each coordinate is drawn independently and uniformly at random from $\bernoulli(1/2)$. 
Thus, the training set $\trainset \in \{0,1\}^{n\times d}$ is a matrix whose elements are drawn \iid~from $\bernoulli(1/2)$. Let $\lambda$ be a sufficiently small constant%
, and $\parspace$ be a ball of radius one in $\Reals^d$. 
We consider the following loss function $\losscvx: \parspace \times \dataspace \to \Reals$, 
$
\losscvx(w,z)= \sum_{i=1}^{d} z(i) w(i)^2 + \lambda  \inner{w}{z} + \max\{\max_{i\in \range{d}}\{w(i)\},0\}.$
We show that this function is convex and $4$-Lipschitz. As a result, the problem is in the CLB subclass. 
Next, we demonstrate that when the dimension is $d=3T2^n/4$, there are many columns in $\trainset$ such that \emph{all} the entries are zero. Following \citet{amir2021sgd}, we refer to such columns as \emph{bad coordinates}. 
Let $\badset \in \{0,1\}^d$ be a vector whose $i-$th coordinate is one if and only if $i$ is a bad coordinate.
We show that, with high probability, the number of bad coordinates is between $T/2$ and $T$. 
The result emerges from the observation that the dynamics of GD along the bad coordinates are completely \emph{different} compared to the good coordinates, therefore ``revealing'' which coordinates are bad.
To see this, consider the empirical risk
$\displaystyle \Empriskcvx{w} = \sum\nolimits_{i=1}^{d} \empmean(i) w(i)^2 + \lambda\inner{\empmean}{w} + \max\{\max\nolimits_{i\in \range{d}}\{w(i)\},0\}$,
where for $i\in \range{d}$, $\empmean(i)= \frac{1}{n} \sum_{j=1}^{n}z_j(i) \in [0,1]$ is the empirical mean of the points in the $i$-th column of $\trainset$.
By the definition of the bad coordinates we can write $\Empriskcvx{w} = \sum_{i \in \{i:\badset[i]=0\}} \empmean(i) w(i)^2 + \lambda\sum_{i\in \{i:\badset[i]=0\}}\empmean(i)w(i) + \max\{\max_{i\in \range{d}}\{w(i)\},0\}$. Note that the third term is not differentiable. 
We consider a specific first-order oracle proposed in \citep{amir2021sgd,bassily2020stability}. 
We show that to analyze the dynamics of GD for good coordinates, we only need to consider the first two terms. For good coordinates, the main observation here is that because of the \emph{norm-like} penalty from the first term, $|\ww_T(i)|$ is small. 
In contrast, for the bad coordinates the gradient that comes from the third term pushes $\ww_T(i)$ away from zero; in particular, for the bad coordinates we have $|\ww_T(i)|=\eta$ under the event  $T/2 \leq  \norm{\badset}_0 \leq T$. 
The other key property used in the proof is that $\norm{\ww_T} \in \bigTheta{1/\sqrt{n}}$ with high probability, meaning that the final iterate of GD is close to the origin.
\item \textbf{Lower Bound on the Residual Term}:
First, we prove that if $\sigma^2 \in \Omega(1/d)$, then the residual term is large.
Recall that $\norm{\ww_T} \in \bigTheta{1/\sqrt{n}}$, and
$\sw = \proj(\ww_T+\xi)$.
Consider $\EE\sbra{|\Popriskcvx{\sw}-\Popriskcvx{\ww_T}|}$, where the population risk is given by
$\displaystyle \Popriskcvx{w}=\nicefrac{1}{2} \norm{w}^2 + \nicefrac{\lambda}{2} \sum\nolimits_{i=1}^{d}w(i) + \max \{\max\nolimits_{i \in \range{d}}\{w(i)\},0 \}$.
Using concentration inequalities for Gaussian random variables, we show that $\norm{\sw}^2\approx \min\{\sigma^2 d + o\mleft(1\mright),1\}$, while $\norm{\ww_T}^2 \in o\mleft(1\mright)$. Using this argument we show that unless $\sigma^2 \in \bigO{1/d}$, the residual term grows with $n$. Since $d$ is exponentially large in $n$, we conclude that the variance of noise has to satisfy $\sigma^2 \in \bigO{2^{-n}}$.
\item \textbf{Lower Bound on $\cmi$ and $\iomi$:}
Here we show that $\cmi  
\in \Omega(n)$, which implies that $\iomi \in \Omega(n)$, since $\cmi \leq \iomi$ \citep[Thm.~2.1]{haghifam2020sharpened}.
In \cref{pf:main-iomi}, we prove the stronger result 
$\iomi \in \Omega(n^3)$. %
Step one is to establish that $
\cmi \geq n -(\entr{\badset\vert \sw,\supersample}+\entr{U\vert \sw,\supersample,\badset}) \geq n - (\entr{\badset\vert \sw} + \entr{U\vert \supersample,\badset})$ using standard properties of mutual information.
Next, we seek to upper bound $\entr{\badset \vert \sw}$ and $\entr{U\vert \supersample,\badset}$.
We do so using \emph{Fano's inequality} (\cref{lem:fano}) but in a way that differs from its conventional use.
The intuition behind using Fano's inequality is as follows:
if there exists an estimator that can be used to predict $\badset$ using $\sw$, then the conditional entropy $\entr{\badset \vert \sw_T}$ is small. 
The same also holds for predicting $U$ using $\sw,\supersample,\badset$.
The core of the proof then rests on designing two estimators:
one for estimating $\badset$ using $\sw$, and %
another one for estimating $U$ using $\supersample$ and $\badset$. 
We construct explicit estimators for each, and demonstrate that their probability of error is small.
Thus Fano's inequality implies that the entropy terms of interest are small.
To construct the first estimator, we use two important properties: (i) the variance of noise satisfies $\sigma^2 \in \bigO{2^{-n}}$, and (ii) for the good coordinates $|\ww_T(i)|$ is very small, while for the bad coordinates we have $|\ww_T(i)| \in \bigTheta{n^{-1.5}}$. 
The proposed estimator is based on comparing $|\sw(i)|$ with a threshold. 
We show that $\sigma^2$ is much smaller than $|\ww_T(i)|$ for the bad coordinates. As a result, the Gaussian noise does not \emph{perturb} the bad coordinates significantly.
Thus, the error probability of this estimator can be arbitrarily small as $n$ diverges. 
For constructing the second estimator, 
remember that: (i) by definition, in each column of $\supersample$ exactly one sample is chosen for the training set, and (ii) by the definition of the bad coordinates, we know that if $i \in \range{d}$ is a bad coordinate, then for all $Z \in \trainset$, we have $Z(i)=0$. 
Therefore, in every column of the supersample, either one or both of the samples have \emph{zeros in all of the bad coordinates}.
Our proposed estimator is as follows: whenever there is only one sample, the estimator can perfectly recover $U$ for that column. 
In the case of two samples, the estimator makes a random guess.
We show that the probability that there are two samples in a column such that both have zeros in all of the bad coordinates is $\Theta({2^{-n^2}})$. Therefore, the estimator makes an error with small probability. 
\end{itemize} 
\end{proof}

\begin{remark}
The sequence of SCO problems that witnesses that lower bound for the IOMI and CMI frameworks is \emph{the same}.
Hence, a tight generalization bound cannot be achieved for every SCO problem by considering the best framework for that problem out of the IOMI and CMI frameworks.
\end{remark}

\begin{remark}
\Cref{eq:excess-risk-gd} provides a general result for the excess risk guarantee of GD for every number of iterations $T$ and the step size $\eta$. GD obtains the excess risk and the generalization error guarantees of  $\bigO{\frac{LR}{\sqrt{n}}}$ by setting $T \in \bigTheta{n^\alpha}$ and $\eta \in \bigTheta{\frac{R\sqrt{n}}{Ln^\alpha}}$ for every $\alpha\geq 2$. In \cref{thm:gd-noisy}, we state the results only for $\alpha=2$.
However, the same construction can be used to prove the lower bounds in  \cref{thm:gd-noisy} for every $\alpha\geq 2$.
This observation shows a stronger failure: for every parameter setting under which GD attains the optimal excess risk, the upper bound in \cref{eq:desired-result} does not even converge to zero, i.e., $\bigOmega{1}$.
\end{remark}

\begin{remark}
A notable property of the construction in \cref{thm:gd-noisy} is that the Lipschitz constant of the loss function and the diameter of $\parspace$ do not grow with dimension. By a simple scaling, our result in \cref{thm:gd-noisy} implies the lower bound stated in \cref{eq:main-result-paper}.
\end{remark}

\section{Implications for PAC-Bayes Bounds}
\label{sec:pac-bayes}
In this section, we show that our construction that witnesses the lower bounds in \cref{thm:gd-noisy} reveals a limitation of PAC-Bayes bounds for learning SCO problems with GD. 
Using PAC-Bayes bounds to analyze the generalization of gradient methods via the surrogate algorithm that perturbs the final weight with a Gaussian random variable is a prevailing method in the literature 
\citep{LanCar2002,hellstrom2020nonvacuous,GR17,dziugaite2018data,dziugaite2021role,neyshabur2018pac,chatterji2019intriguing,foret2020sharpness,wu2020adversarial}.
This approach leads to non-vacuous estimates
of the generalization gap for non-convex problems such as training modern deep learning  models. 
However, we show that %
it fails for the CLB subclass of SCO problems.

We consider a classical PAC-Bayes bound \citep{mcallester1999some} and a recently-proposed conditional PAC-Bayes bound \citep{grunwald2021pac}. 
The main difference between the two is the \emph{measure of complexity} that characterizes generalization. 
We can represent an algorithm $\Alg_n$ with a posterior distribution $Q: \dataspace^n \to \ProbMeasures{\parspace}$.  
A complexity measure appearing in classical PAC-Bayes bounds is $
C_\textnormal{clas}(n) = \KL{Q(\trainset)}{\EE[Q(\trainset)]}$. 
The conditional PAC-Bayes bound relies on some additional structure. Let $\trainset=(Z_1,\dots,Z_n) \sim \Dist^{\otimes n}$ and  $\trainset'=(Z'_1,\dots,Z'_n) \sim \Dist^{\otimes n}$ such that $\trainset \indep \trainset'$. For every $u=(u_1,\dots,u_n) \in \{0,1\}^n$, define $\tilde{S}_u=((1-u_1) Z_1 + u_1 Z'_1,\dots,(1-u_n) Z_n + u_n Z'_n)$. The complexity measure for the conditional-PAC Bayes bound \citep{grunwald2021pac} is $C_\textnormal{cond}(n) = \cEE{\trainset}{[\KL{Q(\trainset)}{2^{-n}\sum_{u \in \{0,1\}^n}Q(\tilde{S}_u)}]}$. Next we present the known results that relate these complexity measures to the generalization gap. 
\begin{theorem}[\citealt{mcallester1999some,grunwald2021pac}]
\label[theorem]{thm:pac-bayes-classic}
Let $\trainset \sim \Dist^{\otimes n}$, $\delta \in (0,1)$, $L,R \in \Reals_{+}$. Assume that the range of the loss function $\losscvx$ lies in $[-LR,LR]$. Then, with probability at least $(1-\delta)$ (over the choice of $\trainset \sim \Dist^{\otimes n}$) for any posterior distribution $Q: \dataspace^n \to \ProbMeasures{\parspace}$ with $\ww \sim Q(\trainset)$,
\[
\nonumber
	\cEE{\trainset} \sbra{\Popriskcvx{W}- \Empriskcvx{W}}
		\in \bigO{LR\bigg(\frac{\min\{C_\textnormal{cond}(n) ,C_\textnormal{clas}(n)\}+\log(n/\delta) }{n}\bigg)^{\frac{1}{2}}}.
\]
\end{theorem}

Let $\complexity{n}$ denote %
either $C_\textnormal{clas}(n)$ or $C_\textnormal{cond}(n)$. Note that $\complexity{n}$ is a $\trainset$-measurable random variable. 
We next present our main result of this section showing the failure of PAC-Bayes bounds for learning SCO with GD.
\renewcommand{\Dist}{\mathcal{D}_{(n)}}
\begin{theorem}
\label[theorem]{thm:main-pacbayes}
Let $n\in \Naturals$, $T_{(n)}=2n^2$, $\eta_{(n)}=\frac{1}{n\sqrt{5n}}$, $d_{(n)}=3T2^n/4$, and $N^\star \in \Naturals$ be a universal constant. Then, there exists $\omega \in (0,1)$,  a sequence of SCO problems $\{(f_{(n)},\parspace_{(n)},\dataspace_{(n)})\in \mathcal{C}_{4,1}\}_{n\in \Naturals}$ where $\parspace_{(n)},\dataspace_{(n)} \in \Reals^{d_{(n)}}$, and a data distribution $\Dist$ over $\dataspace_{(n)}$ such that the following holds for all $n \geq N^\star$: 
For $\trainset \dist \Dist^{\otimes n}$, let $\ww_T = \text{GD}(\trainset,\eta_{(n)},T_{(n)})$  
and $\Alg_n(\trainset) = \proj(\ww_T+\xi)$, where $\xi \dist \Normal(0,\sigma^2\id{d_{(n)}})$. 
Then, for every $0<\delta< 1-\omega$, with probability at least $1-\delta-\omega$ 
over $\trainset \sim \Dist^{\otimes n}$,
\[
\nonumber
\inf_{\sigma \geq 0}\max\bigg\{\sqrt{\frac{ \complexity{n} + \log(\nicefrac{n}{\delta}) }{n}} , \resemp(\ww_T) + \respop(\ww_T)\bigg\} \in \Omega(1).
\]
\end{theorem}
\renewcommand{\Dist}{\mathcal D}
This result implies that a PAC-Bayes bound for the surrogate from  \cref{eq:pac-bayes-surrogate} yields a \emph{vacuous}  generalization bound  with \emph{constant probability}, i.e.,  independent of $n$.
\begin{remark}
The complexity term in PAC-Bayes bounds generally takes the form 
$\KL{Q(\trainset)}{P}$ for some element $P \in\ProbMeasures{\parspace}$.
The choice here, $P = \EE[Q(\trainset)]$, minimizes the complexity term in expectation.
Whether other choices might yield tighter high probability bounds is left open.
\end{remark}
\section{Failure of Information-Theoretic Alternatives to the IOMI and CMI Frameworks}
\label{sec:failure_alternatives}

In the previous sections, we showed how the IOMI and the CMI frameworks and their high-probability counterparts fail to characterize the behavior of GD in the CLB setting, even when they are strengthened with a surrogate analysis. 
In this section, we consider other alternatives and reinforcements of these two frameworks and show that they also fail to characterize the behavior of SCO problems in the CLB setting, albeit without considering any potential strengthening with a surrogate analysis. First, we introduce these alternatives and their motivation, then we adapt them to the CLB setting, and finally we show their failure. 
\subsection{Information-Theoretic Alternatives to the IOMI and CMI Frameworks}
\label{subsec:alternatives}
The IOMI and CMI frameworks are attractive due to algorithm- and distribution-dependence.
Nevertheless, they come with some drawbacks. 

\begin{itemize}
    \item[1] The IOMI may be infinite and the CMI may be $\Omega(n)$ for a variety of learning scenarios, e.g., deterministic algorithms.
    \item[2] \sloppy  IOMI and CMI may capture unnecessary information. 
    Note that we can write $\iomi = \sum_{i=1}^n \minf{\Alg(\trainset);Z_i} + \minf{Z^{i-1},Z_i|\Alg(\trainset)}$, where $Z^{i-1} \coloneqq (Z_1, \ldots, Z_{i-i})$.
    It is clear from this decomposition that IOMI not only captures the information that the output contains about individual samples, but also captures the ``artificial" dependencies among the samples, given the algorithm's output.
    The latter is not predictive of the generalization performance of the algorithm \citep{bu2020tightening}.
    An analogous problem arises in CMI,
    which includes the dependence of the indices and the samples after observing the algorithm's output \citep{rodriguez2020random}.
\end{itemize}

These problems can be avoided with an \emph{individual-sample} bound proposed in \citep{bu2020tightening}, replacing $\minf{\Alg(\trainset);S}$ with the average of $\minf{\Alg(\trainset);Z_i}$ for all $i \in \range{n}$.
This bound takes into account the information the output of the algorithm captures about each \emph{individual} sample $Z_i$, disregarding the generated dependency between the samples after observing the said output. 
Similarly, the individual-sample bound from \citep{rodriguez2020random, zhou2020individually} considers the information the output of the algorithm captures about each individual index $U_i$, disregarding the dependency between the indices and the samples.
These bounds are adapted to the CLB setting in the following theorem. The proof is in \cref{sec:proofs_ege_it}.

\begin{theorem}
\label[theorem]{th:ege_iomi_complete_and_cmi_ind}
Let $n \in \Naturals$, $\Dist \in \ProbMeasures{\dataspace}$ be a data distribution, and $\trainset \dist \Dist^{\otimes n}$. %
Consider an SCO problem $(\losscvx, \parspace, \dataspace) \in \clr$.
Then, for every learning algorithm $\Alg_n$ such that $\Alg_n(\trainset)\in \parspace$ \as, we have $ \EGE \leq \nicefrac{LR}{n} \sum_{i=1}^n \sqrt{2 \minf{\Alg(\trainset);Z_i}} $, and $\EGE \leq \nicefrac{2LR}{n} \sum_{i=1}^n \sqrt{2\minf{\Alg(\trainset);U_i|\tilde{Z}_{0,i},\tilde{Z}_{1,i}}}$.
\end{theorem}

\begin{remark}
    \label{rem:individual_tighter}
    As mentioned above, the individual-sample alternatives to the IOMI and CMI are tighter than the IOMI and CMI. This may be seen 
    by~\citep[Prop. 2]{bu2020tightening} and~\citep[Lemma 3]{rodriguez2020random} or~\citep[Lemma 2]{zhou2020individually}, where we have that
    \begin{equation*}
        \frac{1}{n} \sum_{i=1}^n \sqrt{2 \minf{W;Z_i}} \leq \sqrt{\frac{2 \iomi}{n}} \ \ \textnormal{and} \ \ \frac{1}{n} \sum_{i=1}^n \sqrt{2 \minf{W;U_i|\tilde{Z}_{1,i},\tilde{Z}_{2,i}}} \leq \sqrt{\frac{2 \cmi}{n}}.
    \end{equation*}
    Therefore, \Cref{th:ege_iomi_complete_and_cmi_ind} implies~\cref{lemma:mi-and-cmi-bound-lipschitz}. Moreover, we have that $\minf{\Alg(\trainset);U_i|\tilde{Z}_{1,i},\tilde{Z}_{2,i}} \leq \minf{\Alg(\trainset),Z_i}$~\citep[App. D.2.3]{rodriguez2021tighter}.
\end{remark}

Another drawback of IOMI and CMI frameworks is the following:
\begin{itemize}
    \item[3] 
    Both the IOMI and CMI of an algorithm depend on the joint distribution of the algorithm's output and other variables. In contrast, generalization error depends on the algorithm's output only through the losses it incurs.
    Therefore, it is possible to increase both the IOMI and the CMI by \emph{embedding} information about the training set in the output of a learning algorithm without affecting the algorithm's statistical properties~\citep{roishay,bassily2018learners}.
\end{itemize}

\citet{steinke2020reasoning} propose an alternative framework, \emph{evaluated} CMI, that considers the information about the data captured by the \emph{incurred loss} rather than the output itself.

\begin{definition}[Evaluated CMI, {\citep[Sec. 6.2.2]{steinke2020reasoning}}]
\label{def:ecmi}
Let $n\in \Naturals$. Let the supersample $\supersample$ and indices $U$ be as defined in \cref{sec:problem-def}. 
Let $\trainset=(Z_{U_i,i})_{i \in [n]}$, and $\lossveccvx \in \Reals^{2\times n}$ be the array with entries $\lossveccvx_{v,i}=f(\Alg_n(\trainset),Z_{v,i})$ for $v\in \{0,1\}$, $i\in \range{n}$.  The \defn{evaluated conditional mutual information of $\Alg$ with respect to $\Dist$},
denoted by $\ecmi$,
is the conditional mutual information $\minf{\lossveccvx ;U \vert \supersample}$. 
\end{definition}

\citet{haghifam2021towards} show that the eCMI can provide a sharp characterization of generalization for the realizable setting and 0--1 losses.
Below, we state a bound for the CLB setting based on eCMI. The proof can be found in \cref{sec:proofs_ege_it}.

\begin{theorem}
\label[theorem]{lemma:ecmi-bound-lipschitz}
Consider an SCO problem $(\losscvx, \parspace, \dataspace) \in \clr$.
Then, for every learning algorithm $\Alg_n$ such that $\Alg_n(\trainset)\in \parspace$ \as, we have
$
    \EGE \leq L R\sqrt{\frac{8 \ecmi}{n}}.
$
\end{theorem}

\begin{remark}
    \label{rem:evaluated_tighter}
    Similarly to~Remark~\ref{rem:individual_tighter}, note that the evaluated version of the CMI is tighter than the CMI itself, i.e., $\ecmi \leq \cmi$~\citep{steinke2020reasoning}.
\end{remark}

\begin{remark}
    Since these alternatives to the IOMI and CMI are tighter than the IOMI and CMI themselves (cf.%
    ~Remark~\ref{rem:individual_tighter} and~Remark~\ref{rem:evaluated_tighter}), the adaptation of these bounds to the CLB setting (\cref{th:ege_iomi_complete_and_cmi_ind,lemma:ecmi-bound-lipschitz}) are also tight in the sense of~\cref{thm:bounds-tight}.
\end{remark}

\subsubsection{Data-dependent alternatives and functional CMI}
\label{subsub:other_alternatives}

    \citet{negrea2019information} and~\citet{haghifam2020sharpened} introduced data-dependent alternatives to the IOMI and CMI frameworks that resulted in numerically non-vacuous generalization guarantees for stochastic gradient Langevin dynamics (SGLD) and its full-batch counterpart for modern deep-learning datasets and architecture. These bounds can also be adapted to the CLB setting by replicating~\Cref{th:ege_iomi_complete_and_cmi_ind} (i) considering \citep[Thm.~2]{rodriguez2021tighter} instead of \citep[Thm.~1]{rodriguez2021tighter} and noting that $\EE \big[\KL{\cPr{\trainset}{W}} {\cPr{\trainset \setminus Z_i}{W} } \big] = \minf{W;Z_i|\trainset \setminus Z_i}$; and (ii) considering \citep[Thm.~4]{rodriguez2021tighter} instead of \citep[Thm.~3]{rodriguez2021tighter} and noting that $\EE \big[\KL{\cPr{\supersample, U}{W}}{\cPr{\supersample, U \setminus U_i}{W}}\big] = \minf{W;U_i|\supersample, U \setminus U_i}$. This would yield the data-dependent EGE bounds
\begin{align}
    \EGE &\leq \frac{LR}{n} \sum_{i=1}^n \sqrt{2 \minf{\Alg(\trainset);Z_i|\trainset \setminus Z_i}} \quad \textnormal{and}
    \label{eq:ege_mi_subset} \\
    \EGE &\leq \frac{2LR}{n} \sum_{i=1}^n \sqrt{\minf{\Alg(\trainset),U_i|\supersample, U \setminus U_i}}.
    \label{eq:ege_cmi_subset}
\end{align}
Both~\cref{eq:ege_mi_subset} and~\cref{eq:ege_cmi_subset} are looser than the bounds in~\cref{th:ege_iomi_complete_and_cmi_ind}, by similar arguments to those in~Remark~\ref{rem:individual_tighter}~\citep[App. J]{rodriguez2020random}.

    \citet{harutyunyan2021information} introduced an alternative to the CMI for supervised learning problems that yield bounds that can be experimentally computed and are non-vacuous. %
    However, by the data processing inequality we have that this notion is looser than the evaluated CMI.

\subsection{Failure of the Alternatives}
\label{subsec:failure-alternatives}

We demonstrate that the individual sample and evaluated versions of CMI still fail in the CLB setting.
Based on the relative tightness of these alternative frameworks (see~%
Remark~\ref{rem:individual_tighter} and~Remark~\ref{rem:evaluated_tighter}), showing their failure implies failure of all the aforementioned alternatives to the IOMI and CMI frameworks.
In fact, it also proves the failure of (i) the data-dependent bounds from~\citep{negrea2019information} and~\citep{haghifam2020sharpened}, and (ii) functional-CMI of \citep{harutyunyan2021information}, adapted to the CLB setting (c.f.~\cref{subsub:other_alternatives}).

The following theorem states that the \emph{distribution-free learnability} of GD cannot be \emph{directly} proved using any of the alternatives to the IOMI and CMI framework described above.

\renewcommand{\Dist}{\mathcal{D}_{(n)}}
\begin{theorem}
\label[theorem]{th:ecmi-failining-example}
Let $n\in \Naturals$, $T_{(n)}=n^2$, $\eta_{(n)}=\frac{1}{n\sqrt{n}}$, and $d_{(n)}=2n^2$. Then, for every $n\geq 1$, there exists a sequence of SCO problems $\{(f_{(n)},\parspace_{(n)},\dataspace_{(n)})\in \mathcal{C}_{1,1}\}_{n\in \Naturals}$ where $\parspace_{(n)},\dataspace_{(n)} \in \Reals^{d_{(n)}}$, and data distribution $\Dist$ over $\dataspace_{(n)}$ such that the following holds: Let $\ww_T=\text{GD}(\trainset,\eta_{(n)},T_{(n)})$. Then, $ \ecmi \in \Omega(n)$, and $\sum_{i=1}^n \sqrt{2\minf{\Alg(\trainset);U_i|\tilde{Z}_{0,i},\tilde{Z}_{1,i}}}  \in \Omega(n)$,
while the  generalization error of GD satisfies
$
    \EE[|\Popriskcvx{\ww_T}-\Empriskcvx{\ww_T}|] \in \bigO{1/\sqrt{n}}.
$
\end{theorem}
\renewcommand{\Dist}{\mathcal{D}}

\begin{proof}%
Here, we provide an overview of the proof. 
The formal proof can be found in~\cref{pf:ecmi-fails}. Let $\dimcvx \in \Naturals$ and $\dataspace =  \lbrace \coorvec{i} : i \in \range{\dimcvx} \rbrace$, %
where 
$
\coorvec{i} = (0, \ldots, 0, 1,0, \ldots, 0)
$ with a $1$ at the $i$-th coordinate and $\lVert \coorvec{i} \rVert_2 = 1$.
Let the data distribution on the input be the uniform distribution, that is $\Dist = \textnormal{Uniform}(\dataspace)$. 
Consider a problem in the CLB class with a convex, $1$-Lipschitz loss function $\losscvx(w,z) = - \langle w, z \rangle$, and $\parspace = \lbrace w\in \Reals^d : \norm{w}_2 \leq 1 \rbrace$. 
With this loss, the weights $W_T$ returned by GD after $T$ iterations are a weighted sum of the instances $Z_i$. 
As in the \emph{birthday paradox}~\citep[Sec.~5]{mitzenmacher2017probability} problem, we can show that for large 
$\dimcvx$, e.g. $\dimcvx = 2n^2$, 
the probability that any two instances from the supersample $\supersample$ share the same non-zero coordinate is smaller than some constant probability $c$, which is independent of the number of samples. 
Let $E$ be an $\supersample$-measurable random variable that is one if and only if no pair of instances $\tilde{Z}_{u,i}$ and $\tilde{Z}_{v,j}$ (for all $i,j \in \range{n}$ and all $u,v \in \{0,1\})$ from the supersample $\supersample$ share the same coordinate.

\begin{itemize}[leftmargin=*]
    \item \textbf{Lower bound on $\minf{\Alg(\trainset);U_i|\tilde{Z}_{0,i},\tilde{Z}_{1,i}}$}: When $E=1$ one can completely identify which instance (the index $U_i$) was used for training by looking at the non-zero coordinates of $W_T$: if $\tilde{Z}_{0,i} = \coorvec{k}$ and $W_T(k) \neq 0$, then $U_i = 0$, and otherwise $U_i = 1$. Therefore, under ${E}=1$, we have that $\minf{\Alg(\trainset);U_i\vert\tilde{Z}_{0,i},\tilde{Z}_{1,i}} = \entr{U_i} = \textcolor{blue}{\log 2}$. 
    
    \item \textbf{Lower bound on $\ecmi$}: Similarly, when ${E}=1$, one can completely identify which instances (the indices $U$) were used for training by looking at the non-zero entries of the loss vector $\lossveccvx$: if $\lossveccvx_{0,i} \neq 0$, then $U_i = 0$, and otherwise $U_i = 1$. Therefore, under ${E}=1$, we have that $\ecmi = \entr{U} = n $. %
\end{itemize}
Finally, noting that this event has a constant probability, i.e. $\Pr(E = 1) \geq c$, completes the proof.

\end{proof}

\section{Open Questions}
In this work, we uncover the limitations of information-theoretic analyses of GD for the CLB subclass of SCO problems. We further show that these limitations remain even when a surrogate algorithm based on Gaussian perturbation is considered.
Our results prompt several directions for future research:

\begin{enumerate}
    \item \textbf{Optimal dependence of the information-theoretic bounds on the dimension}: One of the common properties between our constructions in \cref{thm:gd-noisy} and \cref{th:ecmi-failining-example} is that the dimension is much larger than the number of samples. In particular, we exploit the fact that the generalization guarantees of GD for SCO problems is dimension-independent in order to construct problem instances with large information complexity. In particular, it is straightforward to see that the lower bounds on IOMI and CMI that stem from our constructions in \cref{thm:gd-noisy} and \cref{th:ecmi-failining-example} depend on the dimension.   
    It is interesting to find the minimum dimension such that there exists an SCO problem for which the information-theoretic bounds fail to characterize learnability. For the direct analysis of GD we show  $\bigO{n^2}$  is sufficient (\cref{th:ecmi-failining-example}), while for the surrogate analysis exponential dependence, i.e.,  $\bigO{n^2 2^n}$ (\cref{thm:gd-noisy}), is sufficient where $n$ is the number of training samples. 
    
    \item \textbf{Gaussian perturbation for other subclasses of SCO problems}:
    In this work, we proved limitations of the surrogate algorithm based on the Gaussian perturbation for the CLB subclass of SCO problems. In particular, the loss function used in \cref{thm:gd-noisy} is a \emph{non-smooth} convex function. It is an open question to show that such limitations exist for the subclasses of SCO problems with smooth or strongly-convex loss functions. Notice that our results in~\cref{th:ecmi-failining-example} suggests that a \emph{direct} analysis still fails for the subclass of SCO problems with smooth loss functions as the loss function used for proving~\cref{th:ecmi-failining-example} is smooth. 
    \item \textbf{Instance-independent surrogates}: The notable property of Gaussian perturbation is that it is instance-independent, in the sense that its structure does not depend on the problem instance, and we only need to tune the variance based on the problem instance.  It is an open
    question to prove or refute the existence of a \emph{instance-independent surrogate} for analyzing the generalization of gradient descent methods for SCO problems using information-theoretic frameworks. An interesting starting point is investigating the prospect of using the Gibbs algorithm \citep{wang2016average,aminian2021exact} as a problem-independent surrogate.
    \item \textbf{Instance-dependent surrogates}:  We can also study the prospect of instance-dependent surrogates where the surrogate algorithm can depend on the problem instance. For this family of surrogates, the surrogate algorithm is chosen based on the data distribution, loss function, and the original learning algorithm. 
\end{enumerate}

\section*{Acknowledgments}
The authors would like to thank Gergely Neu and Thomas Steinke for helpful discussions. 

\section*{Disclosure of funding}
Mahdi Haghifam is supported in part by the Vector Institute and Doctoral Completion Award from University of Toronto.
Borja Rodríguez-Gálvez and Mikael Skoglund were funded in part by the Swedish research council under contract 2019-03606. Ragnar Thobaben was funded in part by the Swedish research council under contract 2021-05266. Part of the work was done while Mahdi Haghifam was an intern at Google Research-Brain Team and Borja Rodríguez-Gálvez was an intern at Apple.  
Daniel Roy is supported in part by an NSERC Discovery Grant and Accelerator.
This research was carried out, in part, while Daniel Roy and Gintare Karolina Dziugaite were visiting the Simons Institute for the Theory of Computing at UC Berkeley, and while Daniel Roy was visiting the Institute for Mathematical Research at ETH Z\"urich.

\printbibliography

\newpage

\appendix

\section{Proof of the information-theoretic bounds of {$\EGE$} in the CLB setting}
\label{sec:proofs_ege_it}

Before starting the proofs, note that the proof of %
Theorem~\ref{th:ege_iomi_complete_and_cmi_ind} implies~\cref{lemma:mi-and-cmi-bound-lipschitz} %
(c.f.~Remark~\ref{rem:individual_tighter} and~Remark~\ref{rem:evaluated_tighter}).

\subsection{Proof of~\cref{th:ege_iomi_complete_and_cmi_ind}: Individual-sample IOMI}
Consider~\citep[Thm.~1]{rodriguez2021tighter}, which controls $\EGE$ by means of the Wasserstein distance
\begin{equation*}
    \EGE \leq \frac{L}{n} \sum_{i=1}^n \EE \Big[ \mathbb{W}(\cPr{Z_i}{W},\Pr(W)) \Big].
\end{equation*}

Then, consider the fact that the Wasserstein distance is dominated by the total variation, that is, that $\mathbb{W}(\mu,\nu) \leq 2R \texttt{TV}(\mu,\nu)$ when the space where the distributions $\mu$ and $\nu$ are defined has diameter $R$ with respect to the specified metric~\citep[Thm.~6.15]{villani2009optimal}\footnote{In the particular case of this work, the metric considered for the Lipschitness of the function and the diameter of the space is the $\ell_2$ norm difference, but these theorems are not restricted to that.}. Applying Pinsker's~\citep[Thm.~6.5] {polyanskiy2014lecture} inequality to the total variation and Jensen's inequality afterwards, one recovers the desired bound in~\cref{th:ege_iomi_complete_and_cmi_ind}.

\subsection{Proof of~\cref{th:ege_iomi_complete_and_cmi_ind}: Individual-sample CMI}
Consider now~\citep[Thm.~3]{rodriguez2021tighter}, which again controls $\EGE$ by means of the Wasserstein distance
\begin{equation*}
    \EGE \leq \frac{L}{n} \sum_{i=1}^n \EE \Big[ \mathbb{W}(\cPr{U_i, \tilde{Z}_{0,i}, \tilde{Z}_{1,i}}{W},\cPr{\tilde{Z}_{0,i}, \tilde{Z}_{1,i}}{W}) \Big].
\end{equation*}

As in the proof above, considering the domination of the Wasserstein distance by the total variation together with Pinsker's and Jensen's inequality recovers the desired bound in~\cref{th:ege_iomi_complete_and_cmi_ind}.

\subsection{Proof of Theorem~\ref{lemma:ecmi-bound-lipschitz}}
By the Donsker-Varadhan lemma \citep[Prop.~4.15]{boucheron2013concentration} we have that
\begin{equation*}
     \minf{\lossveccvx,\supersample;U} \geq \EE[g(\lossveccvx,\supersample, U)] - \log \EE \bigg[ e^{g(\lossveccvx',\supersample',U)}\bigg]
\end{equation*}
for all measurable functions $g$ such that $g(\lossveccvx,\supersample,U)$ and $e^{g(\lossveccvx',\supersample',U)}$ have finite expectations~ \citep[Prop.~4.15]{boucheron2013concentration}, where $(\lossveccvx',\supersample')$ is an independent copy of $(\lossveccvx,\supersample)$ and where $\minf{\lossveccvx,\supersample;U} = \minf{\lossveccvx,U|\supersample} = \ecmi$. For the rest of the proof, let $\lossveccvxlower \in \Reals^{2\times n}$ be a realization of $\lossveccvx$. Consider now
\begin{equation*}
    g(\lossveccvxlower, \tilde{s}, u) = \frac{\lambda}{n} \sum_{i=1}^n (2u_i - 1) \Big( \lossveccvxlower_{0,i} - \losscvx(0,\tilde{z}_{0,i}) - \big(\lossveccvxlower_{1,i} - \losscvx(0,\tilde{z}_{1,i}) \big) \Big)
\end{equation*}
for some $\lambda > 0$, and note that $\EE[g(\lossveccvx,\supersample,U)] = \lambda \EGE$. Applying Donsker-Varadhan lemma~\citep[Prop.~4.15]{boucheron2013concentration} with this choice of $g$ yields
\begin{equation*}
    \ecmi \geq \lambda \EGE - \log \EE \bigg[ e^{ \frac{\lambda}{n} \sum_{i=1}^n (2U_i - 1) \Big( \lossveccvx_{0,i}' - \losscvx(0,\tilde{Z}_{0,i}') - \big(\lossveccvx_{1,i}' - \losscvx(0,\tilde{Z}_{1,i}') \big) \Big) } \bigg].
\end{equation*}

Studying random variables $(2U_i - 1) \Big( \lossveccvx_{0,i}' - \losscvx(0,\tilde{Z}_{0,i}') - \big(\lossveccvx_{1,i}' - \losscvx(0,\tilde{Z}_{1,i}') \big) \Big)$ reveals that they are $0$ mean and bounded in $[-2L R, 2 L R]$. We can thus apply Hoeffding's lemma~\citep[][Example~2.4]{wainwright2019high} to bound the cumulant generating function. Optimizing for $\lambda > 0$ and rearranging completes the proof.

\section{Proof of \cref{thm:bounds-tight}}
\label{pf:bounds-tight}
Let $d \in \Naturals$ be arbitrary. Let $\parspace$ be a ball of radius $R$ in $\Reals^d$. Consider an arbitrary $z_0 \in \parspace$ such that $\norm{z_0}=R$. The input space is $\dataspace=\{z_0/R,-z_0/R\}$. Also, let the data distribution $\Dist$ be $\Dist(z_0/R)=\Dist(-z_0/R)=1/2$, the loss function  $\losscvx:\parspace\times \dataspace \to \Reals$ be $\losscvx(w,z)=-L\inner{w}{z}$. It is straightforward to see that the loss function is convex and $L$-Lipschitz \footnote{The construction for this section is inspired by the lower bounds for online convex optimization in \citep[][Sec.5.1]{orabona2019modern}.}. 

Denote the training set $\trainset=(Z_1,\dots,Z_n)\dist \Dist^{\otimes n}$. Define a Rademacher random variable $\epsilon_i = 1$ if $Z_i=z_0/R$ and $\epsilon_i = -1$ if $Z_i=-z_0/R$. We can represent the training set as $\trainset=(\frac{z_0}{R}\epsilon_1,\dots,\frac{z_0}{R}\epsilon_n)$. The empirical risk for $w \in \parspace$ is given by $\displaystyle \Empriskcvx{w}=\frac{-L}{nR} \big \langle w,z_0\sum\nolimits_{i \in \range{n}}\epsilon_i \big \rangle$. It is straightforward to see that the ERM for this problem is 
\[
\nonumber
\argmin_{w\in \parspace}\Empriskcvx{w}=\Alg_n(\trainset)=\begin{cases}
  z_0 & \text{if}~\text{sign}(\sum_{i=1}^{n}\epsilon_i)=1\\
  -z_0 & \text{if}~\text{sign}(\sum_{i=1}^{n}\epsilon_i)=-1
\end{cases},
\]
where for $x \in \Reals$, $\text{sign}(x)=1$ if $x\geq0$ and $\text{sign}(x)=-1$ if $x<0$.

First, we provide a lower bound on the expected generalization error. The expected empirical risk of $\Alg_n$ is given by
\begin{align*}
\EE \big[ \min_{w\in \parspace} \Empriskcvx{w} \big] &=\EE \bigg[ \min_{w\in \parspace} -\frac{L}{Rn} \Big\langle  w,z_0\sum_{i=1}^{n} \epsilon_i \Big\rangle \bigg] \\
    &= - \frac{L}{Rn}\EE \bigg[\max_{w\in \{z_0,-z_0\}}  \Big\langle  w,z_0\sum_{i=1}^{n} \epsilon_i \Big\rangle \bigg] \\
    &=- \frac{L}{Rn}\EE \bigg[ \Big| \Big\langle  z_0,z_0\sum_{i=1}^{n} \epsilon_i \Big\rangle \Big| \bigg]\\
    &=- \frac{LR}{n}\EE \bigg[ \Big|\sum_{i=1}^{n} \epsilon_i \Big|\bigg],
\end{align*}
where we have used $\forall a,b\in \Reals$, $\max(a,b)=\frac{a+b}{2}+\frac{|a-b|}{2}$. Observe that $\Popriskcvx{w}=0$ for all $w \in \parspace$. Therefore, the expected generalization error is lower bounded by 
\[
\nonumber
\EGE &= -\EE\sbra{\min_{w\in \parspace} \Empriskcvx{w}} \nonumber\\
    &=    \frac{LR}{n}\EE \bigg[ \Big|\sum_{i=1}^{n} \epsilon_i \Big|\bigg] \nonumber \\
    &\geq \frac{LR}{\sqrt{2n}} \nonumber,
\]
where the last line follows from Khintchine--Kahane inequality \citep[][Thm.~D.9]{mohri2018foundations}. 

Next, we analyze the upper bounds based on 
\cref{lemma:mi-and-cmi-bound-lipschitz}. Observe that the following Markov chain holds: 
$$
\trainset \-- \text{sign}\Big(\sum_{i=1}^{n}\epsilon_i \Big) \-- \Alg_n(\trainset).
$$
By the data processing inequality we have 
\[
\nonumber
\iomi = \minf{\Alg_n(\trainset);\trainset}\leq \minf{\Alg_n(\trainset);\text{sign}\Big(\sum_{i=1}^{n}\epsilon_i\Big)}.
\]
We can upper bound the mutual information as
\[
\nonumber
\minf{\Alg_n(\trainset);\text{sign}\Big(\sum_{i=1}^{n}\epsilon_i\Big)} \leq \entr{\text{sign}\Big(\sum_{i=1}^{n}\epsilon_i\Big)}\leq 1,
\]
since $\text{sign}\Big(\sum_{i=1}^{n}\epsilon_i\Big)$ can take only two values. Therefore, we obtain $\iomi \leq 1$. As $\cmi \leq \iomi$ for any learning problem  \citep[Thm.~2.1]{haghifam2020sharpened}, we have $\cmi \leq 1$. Finally, the result follows by plugging the bounds on IOMI and CMI into %
\cref{lemma:mi-and-cmi-bound-lipschitz}.
\section{Proof of \cref{thm:gd-noisy}} 
\label{pf:main-iomi}
The outline of the proof is as follows. First, in \cref{subsec:construct}, we describe our construction. Then, we analyze the dynamics of GD on the problem in \cref{subsec:dynam-gd}. Using the properties of the final iterate of GD, proved in \cref{subsec:dynam-gd}, we proceed by showing in  \cref{subsec:large-var} that if the noise variance is greater than a threshold, then the residual term does not converge to zero as the number of samples grows. For the case that the noise variance is smaller than the threshold, we prove the failure of IOMI and CMI in \cref{subsec:small-var} and \cref{subsec:small-var-cmi}, respectively. 
\subsection{Construction}
\label{subsec:construct}
We begin the proof by describing a learning scenario that witnesses the lower bound (we drop the $n$ argument from the parameters to reduce notational clutter). Let $d \in \Naturals$ and $\dataspace=\{0,1\}^d$. Let the data distribution on input be $(\bernoulli(1/2))^{\otimes d}$, i.e., each coordinate is drawn independently and uniformly at random from $\bernoulli(1/2)$. In this section, we treat the training set $\trainset \in \{0,1\}^{n\times d}$ as a matrix. Note that each element of $\trainset$ is drawn \iid~from  $\bernoulli(1/2)$. For $i\in \range{d}$, we say the $i-$th coordinate is a \emph{bad coordinate} iff for all $j \in \range{n}$, $Z_{j}(i)=0$. In words, if $i-$th coordinate is a bad coordinate then all the entries in the $i-$th column of $\trainset$ is zero. 
Also, the convex \parameterspace space $\parspace$ is the Euclidean ball of radius one in $\Reals^d$. Note for $x\in \Reals^d$, $\proj(x)=x/\max\{\norm{x},1\}$.

We consider the convex function proposed in \citet{amir2021sgd}. Let $0<\lambda\leq \bigO{1/(n\sqrt{d})}$ be a positive constant which is determined later. Then we consider the following loss function $\losscvx: \parspace \times \dataspace \to \Reals$
\[
\label{eq:loss-function-main}
\losscvx(w,z)= \sum_{i=1}^{d} z(i) w(i)^2 + \lambda  \inner{w}{z} + \max \Big\{\max_{i\in \range{d}}\{w(i)\},0\Big\}.
\]
It is straightforward to show that the first two terms in \cref{eq:loss-function-main} is convex. Also, $\max\{\max_{i\in \range{d}}\{w(i)\},0\}$ is a convex function because it is maximum of convex (linear) functions  \citep[][Sec.3.2.3]{boyd2004convex}. Therefore, $\losscvx$ is convex  as it is sum of convex functions. Then, we show that each term in \cref{eq:loss-function-main} is Lipschitz. The first term, $\sum_{i=1}^{d} z(i) w(i)^2 \leq \norm{w}^2$ is $2-$Lipschitz by the boundedness of $\parspace$. The second term is $\lambda\sqrt{d}-$Lipschitz because $\norm{\grad(\lambda  \inner{w}{z})}=\lambda \norm{ z}\leq \lambda \sqrt{d}$. We use \cref{lem:lipschitz-max} to show that the last term in \cref{eq:loss-function-main} is $1$-Lipschitz. Therefore, $\losscvx(w,z)$ is $(3+\lambda \sqrt{d})-$Lipschitz. Note that $\lambda \in \bigO{1/(n\sqrt{d})}$, so the function in \cref{eq:loss-function-main} is $4-$Lipschitz for sufficiently large $n$.  

\subsection{Dynamics of GD}
\label{subsec:dynam-gd}
First of all, we want to note that the statements in this proof about random variables hold almost surely. We will skip such declarations for the remainder of the proof to aid readability. In this part, we aim to find the properties of the final iterates of the GD algorithm.  
Let $d= 0.75 T2^{n}$.
Let $\badset \in \{0,1\}^d$ denote a vector such that $\badset(i)=1$ if and only if $i$ is a bad coordinate. Let $\norm{\badset}_0$ denote the number of bad coordinates. Next, we provide a probabilistic estimate on $\norm{\badset}_0$. $\norm{\badset}_0=\sum_{i=1}^{d}\badset(i)$ follows the binomial distribution with the number of trial $d$ and the success probability of $2^{-n}$. The reason is the probability that all the points in a column is zero is given by $2^{-n}$. By the standard multiplicative Chernoff bound \citep[][Cor.4.6]{mitzenmacher2017probability} we have
\[
\label{eq:bad-coordinate-num}
\Pr({T}/{2}\leq \norm{\badset}_0 \leq T)\geq  1-2\exp(-{T}/{36}).
\]
Therefore, with probability at least $1-2\exp(-{T}/{36})$, the number of bad coordinates is between $T/2$ and $T$. 

Next step concerns understanding the dynamics of GD. The empirical risk for any $w \in \parspace$ is given by 
\[
\label{eq:emp-risk}
\Empriskcvx{w} = \sum_{i=1}^{d} \empmean(i) w(i)^2 + \lambda\inner{\empmean}{w} + \max \Big\{\max_{i\in \range{d}}\{w(i)\},0 \Big\}
\]
where for $i\in \range{d}$, $\empmean(i)=\frac{1}{n}\sum_{j=1}^{n}z_j(i) \in [0,1]$ is the empirical mean of the points in $i-$th column of $\trainset$.

\begin{lemma}
\label[lemma]{lem:dynamics-gd}
Under the event $\{T/2 \leq \norm{\badset}_0 \leq T\}$, let  $\mathcal{B} = \{i_1,\dots,i_{\norm{\badset}_0}\}\subseteq \range{d}$ contain the \textit{ordered} set of bad coordinates. Consider the GD process  $\ww_{t+1}=\ww_{t}-\proj(\ww_t - \eta \subgrad{(\Empriskcvx{\ww_t})})$ starting at $\ww_0=0$ where $\eta$ is the step size. For every $i \in \range{d}$ and $t \in \range{T}$ 
\begin{equation*}
    \ww_t(i) = \begin{cases}
        \frac{\lambda}{2} (-1+(1-2\eta \empmean(i))^t)  &  i \in  \range{d}\setminus \mathcal{B}\\
        -\eta                                         &  i \in  \{i_1,\dots,i_{\min\{ \norm{\badset}_0,t-1\}}\}\\
        0                                             &  i \in  \{i_{\min\{ \norm{\badset}_0,t-1\}+1},\dots,i_{\norm{\badset}_0}\}
    \end{cases}.
\end{equation*}
In particular, 
\begin{equation*}
    \ww_T(i) = \begin{cases}
        \frac{\lambda}{2} (-1+(1-2\eta \empmean(i))^T)  &  i \in  \range{d}\setminus \mathcal{B}\\
        -\eta                                         &  i \in   \mathcal{B}
    \end{cases},
\end{equation*}
and for all $i \in  \range{d}\setminus \mathcal{B}$,  $-\eta \lambda T \leq \ww_T(i) < 0$.
\end{lemma}
\begin{proof}
    First, we describe the first-order oracle proposed in \citet{amir2021sgd} and \citet{bassily2020stability}. Note that the first two terms in \cref{eq:loss-function-main} are differentiable. For the third term, i.e., $f_3(w)=\max\{\max_{i\in \range{d}}\{w(i)\},0\}$, which is not differentiable we consider the following first-order oracle. Let $\mathcal{I}(w) = \{j \in \range{d}\vert j\in \{\arg\max_{i \in \range{d}}{w(i)}\}\cap \{i \vert w(i)\geq 0\}\}$. Then, we claim that 
\[
\label{eq:subg-third}
\subgrad{f_3}(w) =
\begin{cases}
  0  & w=0 ~\text{or} ~\mathcal{I}=\emptyset \\
  \coorvec{\min\{\mathcal{I}(w)\}} &  w\neq 0 ~\text{and}~ \mathcal{I} \neq \emptyset.
\end{cases}
\]
where for $i \in \range{d}$, $\coorvec{i}=(\underbrace{0,\dots,0}_{\text{$i-1$ times} },1,\underbrace{0,\dots,0}_{\text{$d-i$ times} })$ ($i$-th coordinate vector). 

To prove that \cref{eq:subg-third} is a member of subgradient at $w$, we need to prove for all $w,v \in \Reals^d$, we have
$f_3(v)\geq f_3(w)+ \inner{\subgrad{f_3}(w)}{v-w}$. 

Consider the case $w=0$, then, since $\subgrad{f_3}(w)=0$, trivially we have $f_3(v)\geq f_3(w)=0$. Next, consider the case that $w\neq 0$ but $\mathcal{I}(w) =\emptyset$. This case holds if and only if for all $i \in \range{d}$, $w(i)<0$. Therefore, $\subgrad{f_3}(w)=0$ and the first-order convexity condition trivially holds. Finally, consider the case that $w\neq 0$ and $\mathcal{I}(w) \neq \emptyset$. Let $\hat{i}=\min\{\mathcal{I}(w)\}$, then
\[
\nonumber
f_3(w) + \inner{\coorvec{\hat{i}} }{v-w} &= w(\hat{i}) +  \inner{\coorvec{\hat{i}} }{v-w}\\
    &=  w(\hat{i}) +  v(\hat{i})-w(\hat{i}) \nonumber\\
    &= v(\hat{i}) \nonumber\\
    &\leq \max\Big\{\max_{i\in \range{d}}\{v(i)\},0\Big\}, \nonumber
\]
as was to be shown.

Then, we provide analysis of the dynamics of GD using the first-order oracle described above. We only describe the dynamics under the event $\{T/2\leq \norm{\badset}_0\leq T \}$. Let the (ordered) set of bad coordinates denoted by $\mathcal{B}=\{i_1,\dots,i_{\norm{\badset}_0}\}$. The main observation here is that we can re-write the \cref{eq:emp-risk} as follows
\[
\label{eq:emp-risk-seperated}
\Empriskcvx{w} = \sum_{i \in \range{d}\setminus \mathcal{B} } w(i)^2 \empmean(i) + \lambda \sum_{i\in \range{d}\setminus \mathcal{B} } w(i)\empmean(i) +  \max \Big\{\max_{i\in \range{d}}\{w(i)\},0\Big\}.
\]
This equation shows that the gradient comes from the first two terms does not change the bad coordinates of $w$. As we will show that $f_3$ does not provide gradient for bad coordinates, the dynamic of each good coordinate of $w$ is independent of other coordinates. Formally, we prove by induction that $\ww_1=-\eta \lambda \hat{\mu}$, and for $t\geq 2$, 
\begin{equation*}
    \ww_t(i) = \begin{cases}
        \frac{\lambda}{2} (-1+(1-2\eta \empmean(i))^t)  &  i \in  \range{d}\setminus \mathcal{B}\\
        -\eta                                         &  i \in  \{i_1,\dots,i_{\min\{ \norm{\badset}_0,t-1\}}\}\\
        0                                             &  i \in  \{i_{\min\{ \norm{\badset}_0,t-1\}+1},\dots,i_{\norm{\badset}_0}\}
    \end{cases}.
\end{equation*}
For the base case, by the GD algorithm's update rule we have $ \ww_1  =  \proj \big(\ww_{0} - \eta g_0 \big)= \proj \big(- \eta g_0 \big)$. Note that $g_0=\lambda \hat{\mu}$. Since $\lambda \in \bigO{1/(n\sqrt{d})}$, $- \eta g_0 \in \parspace$.

For the inductive step, assume that for some $k \in \range{T-1}$, the claim holds. We have $\ww_{k+1}=\proj \big(\ww_{k} - \eta g_k \big)$. First, for $i\in  \range{d}\setminus \mathcal{B}$, $ g_k(i) = 2 \ww_k(i)\hat{\mu}(i) + \lambda \hat{\mu}(i)$. Note that the gradient from the third term is zero for good coordinates as $\ww_k(i)<0$ for $i\in  \range{d}\setminus \mathcal{B}$. By a simple calculation, one can show that 
\[
\nonumber
\ww_{k} - \eta g_k =  \frac{\lambda}{2} \big(-1+(1-2\eta \empmean(i))^{k+1}\big). 
\]
Also, as $\lambda \in \bigO{1/(n\sqrt{d})}$, $\ww_{k} - \eta g_k \in \parspace$. 
Then, for the bad coordinates, consider two cases $\min\{\norm{\badset}_0\,k-1\}=k-1$ and $\min\{\norm{\badset}_0 , k-1\}=\norm{\badset}_0$.
Consider the first case, i.e., $\min\{\norm{\badset}_0,k-1\}=k-1$. 
Consider $i_{k} \in \mathcal{B}$. 
From \cref{eq:emp-risk-seperated}, the first two terms do not provide gradient for bad coordinates. 
Then, we claim that  $\subgrad{f_3}(\ww_{k})=\coorvec{i_{k}}$. 
The reason is that for all $i\in\{i_1,\dots,i_{k-1}\} \cup \range{d}\setminus \mathcal{B}$, $\ww_k(i)<0$ and $i\in\{i_k,\dots,\norm{\badset}_0\}$, $\ww_k(i)=0$. 
Therefore, the claim follows from \cref{eq:subg-third}. Therefore, for the first case, for all  $i \in \{i_1,\dots,i_{k}\}$, $\ww_{k+1}(i)=-\eta$, and  for all  $i \in \{i_{k+1},\dots,\norm{\badset}_0\}$, $\ww_{k+1}(i)=0$.

Consider the second case, $\min\{\norm{\badset}_0\,k-1\}=\norm{\badset}_0$. In this case, all coordinates of $\ww_k$ are less than zero. Therefore, the gradient from $f_3$ is zero, and the bad coordinates remain unchanged.
\end{proof}

Next, we provide a result regarding $\norm{\ww_T}$.
\begin{lemma}
\label[lemma]{lem:norm-out-gd}
Under the event $\{T/2 \leq \norm{\badset}_0\leq  T\}$, we have $ \frac{1}{2\sqrt{n}} \leq \norm{\ww_T}\leq \frac{1}{\sqrt{n}}$. 
\end{lemma}
\begin{proof}
Under the event  $\{T/2 \leq \norm{\badset}_0\leq  T\}$, \cref{lem:dynamics-gd} shows that  
\[
\nonumber
\norm{\ww_T}= \Big(\norm{B}_0 \eta^2 + \sum_{i \in \range{d}\setminus \mathcal{B}} \ww_T(i)^2\Big)^\frac{1}{2}.
\]
Since for good coordinates, $|\ww_T(i)|\leq \lambda \eta T$, we have the following upper bound $\norm{\ww_T} \leq \sqrt{T \eta^2 + d (\lambda \eta T )^2 }$. For a lower bound consider $\norm{\ww_T} \geq \sqrt{T \eta^2/2}$. Setting the parameters, we obtain $\norm{\ww_T} \leq \frac{1}{\sqrt{n}}$ and $\norm{\ww_T} \geq \frac{1}{2\sqrt{n}}$.
\end{proof}
\subsection{Noise with Large Variance Fails }
\label{subsec:large-var}
Consider the case that the variance of $\xi$ along each dimension is $\sigma^2$ and $\sigma \geq \frac{\beta^\star}{\sqrt{d}} $ where $\beta^\star = 0.1$. In particular, $\frac{\beta^\star}{\sqrt{d}}$ is the threshold for the variance.
First of all note that for all $w \in \parspace$
\[
\nonumber
\Popriskcvx{w}=\frac{1}{2} \norm{w}^2 + \frac{\lambda}{2} \sum_{i=1}^{d}w(i) + \max \Big\{\max_{i \in \range{d}}\{w(i)\},0 \Big\}.
\]
Therefore, 
\[
\label{eq:residual-term}
|\Popriskcvx{\sw_T} - \Popriskcvx{\ww_T}| = \Big|\frac{1}{2} (\norm{\sw_T}^2-\norm{\ww}^2) + \frac{\lambda}{2} \sum_{i=1}^{d}(\sw_T(i) - \ww_T(i)) +  \Xi_T \Big|,
\]
where $\Xi_T =  \max\{\max_{i \in \range{d}}\{\sw_T(i)\},0\} -  \max\{\max_{i \in \range{d}}\{\ww_T(i)\},0\}$.  Under the event $\{T/2 \leq \norm{\badset}_0 \leq T\}$, \cref{lem:dynamics-gd} shows that $\max\{\max_{i \in \range{d}}\{\ww_T(i)\},0\}=0$ since $\ww_T(i)<0$ for all $i\in \range{d}$. Therefore, $\Xi_T =  \max\{\max_{i \in \range{d}}\{\sw_T(i)\},0\} \geq 0$.

Because the Gaussian distribution is invariant under the rotation, we can assume that $\ww_T = (\norm{\ww_T},\underbrace{0,\dots,0}_{\text{$d-1$ times} })$ without loss of generality. Therefore,  \cref{eq:residual-term} is given by 
\[
\label{eq:residual-term-rot}
|\Popriskcvx{\sw_T} - \Popriskcvx{\ww_T}|  =  \Big|\frac{1}{2} (\norm{\sw_T}^2-\norm{\ww_T}^2) + \frac{\lambda}{2} (\sw_T(1) - \norm{\ww_T} + \sum_{i=2}^{d} \sw_T(i) ) +\Xi_T \Big|. 
\]
Let $V_T = \ww_T + \xi$. Let us represent $\xi = r \theta$ where $r = \norm{\xi}$ and $\theta = \nicefrac{\xi}{\norm{\xi}}$. By a simple calculation, one can obtain that 
\[
\label{eq:norm-before-proj}
\norm{V_T}^2=\norm{\ww_T}^2+r^2+2\norm{\ww_T} r \theta(1).
\]
\newcommand{\rmax}{\mathsf{r}_{\text{max}}}  
Define $\rmax=1-\norm{\ww_T}$. Note  $0 \leq \rmax\leq 1$ since $\ww_T\in \parspace$. By the tower rule for the expectation,
\begin{equation}
    \label{eq:towerrule-decompos}
    \begin{aligned}
    &\EE\sbra{|\Popriskcvx{\sw_T} - \Popriskcvx{\ww_T}|} \geq \EE\sbra{|\Popriskcvx{\sw_T} - \Popriskcvx{\ww_T}|\indic{T/2 \leq \norm{\badset}_0 \leq T } \indic{r \leq \rmax}}    \\
&+ \EE\sbra{|\Popriskcvx{\sw_T} - \Popriskcvx{\ww_T}|\indic{T/2 \leq \norm{\badset}_0 \leq T } \indic{r > \rmax}  \indic{\norm{V_T}\leq 1}}\\
&+\EE\sbra{|\Popriskcvx{\sw_T} - \Popriskcvx{\ww_T}|\indic{T/2 \leq \norm{\badset}_0 \leq T } \indic{r > \rmax}  \indic{\norm{V_T}> 1}}
\\ 
    \end{aligned}
\end{equation}

Under the event $\{T/2 \leq \norm{\badset}_0 \leq T \} $, we divide the sample space into three regions: \textbf{Region 1}: $\{r\leq \rmax\}$, \textbf{Region 2}: $\{r > \rmax\}\cap\{\norm{V_T}< 1\}$, and \textbf{Region 3}: $\{r > \rmax\}\cap\{\norm{V_T}\geq 1\}$. In what follows, we lower bound \cref{eq:towerrule-decompos} for each region separately.

\newcommand{\subscript}[2]{$#1 _ #2$}
\begin{enumerate}[wide, labelwidth=!, labelindent=0pt, label={[\subscript{R}{{\arabic*}}]}]
    \item \textbf{Region 1} $\{r\leq \rmax\}$: 
    
By the tower rule for the expectation,
\[
&\EE\sbra{|\Popriskcvx{\sw_T} - \Popriskcvx{\ww_T}|\indic{T/2 \leq \norm{\badset}_0 \leq T } \indic{r \leq \rmax}} \nonumber \\
&= \EE\sbra{\cEE{\ww_T,r}{\sbra{|\Popriskcvx{\sw_T} - \Popriskcvx{\ww_T}|}}\indic{T/2 \leq \norm{\badset}_0 \leq T } \indic{r \leq \rmax}} \nonumber 
\]
Under the event $\{r\leq \rmax\}$, it is straightforward to see that $\norm{V_T}\leq 1$. Therefore, $\sw_T=\proj(V_T)=V_T$, and \cref{eq:residual-term-rot} is given by 
\[
\nonumber
|\Popriskcvx{\sw_T} - \Popriskcvx{\ww_T}| = \Big|\frac{1}{2}r^2 + \norm{\ww_T}r\theta(1) + \frac{\lambda r}{2}\sum_{i=1}^{d}\theta(i) + \Xi_T \Big|.
\]
By the construction of the surrogate algorithm and \cref{lem:polar-gauss}, we know that $\theta$ is independent of $r$ and $\ww_T$. Then, we invoke the reverse triangle inequality, i.e., $|a-b|\geq |a|-|b|$ for $a,b\in \Reals$. Using $\theta \equaldist -\theta$, we have
\[
 &\cEE{r,\ww_T}{\sbra{\Big|\frac{1}{2}r^2 + r\norm{\ww_T}\theta(1) + \frac{ \lambda r}{2}\sum_{i=1}^{d}\theta(i) + \Xi_T \Big|}} \nonumber\\
 &\geq  \underbrace{\cEE{\ww_T,r}{\sbra{\Big|\frac{1}{2}r^2 + \norm{\ww_T}r\theta(1) + \Xi_T\Big|}} }_{\text{\textcircled{1}}} - \underbrace{\cEE{\ww_T,r}{\sbra{\Big|\frac{\lambda r}{2}\sum_{i=1}^{d}\theta(i)\Big| }}}_{\text{\textcircled{2}}}.\label{eq:expect-theta}
\]
We will analyze \textcircled{1} and \textcircled{2} separately. Note that $\theta(1) \dist \unif{[-1,1]}$. Thus,  with probability $\nicefrac{1}{2}$, $\theta(1) \in [0,1]$. Therefore,
\[
\label{eq:term1-inner}
 \text{\textcircled{1}}  \geq \cEE{W_T,r}{\sbra{ \Big|\frac{1}{2}r^2 + \norm{\ww_T}r\theta(1) + \Xi_T\Big| \indic{\theta(1) \in[0,1]}}}\geq \Big|\frac{r^2}{4} + \frac{\Xi_T}{2}\Big| \geq \frac{r^2}{4}, 
\]
where the last inequality follows from $\Xi_T\geq 0$.
 By the Cauchy-Schwartz, we have $\norm{\theta}_1\leq \sqrt{d}$ since $\norm{\theta}_2=1$. Therefore, 
 \[
\text{\textcircled{2}}  & \leq \frac{\lambda r}{2} \norm{\theta}_1 \nonumber\\ 
                        &\leq \frac{\lambda r}{2} \sqrt{d} \nonumber\\
                        &\leq \frac{\lambda \rmax}{2} \sqrt{d} \nonumber\\
                        &\leq \frac{\lambda }{2} \sqrt{d}, \label{eq:term2-inner}
 \]
where the third inequality follows since $r \leq \rmax$ and the the last step follows from $\rmax$ being less than one.
By \cref{eq:expect-theta}, \cref{eq:term1-inner}, and \cref{eq:term2-inner}, we finish lower bounding the inner expectation, 
\[
\label{eq:region1-lowerbound}
\cEE{W_T,r}{\sbra{|\Popriskcvx{\sw_T} - \Popriskcvx{\ww_T}}} \geq \frac{r^2}{4} - \frac{\lambda \sqrt{d}}{2} \geq \frac{r^2}{4}-\frac{1}{2n}.
\]
Here, the last inequality follows from setting $\lambda\leq \frac{1}{n\sqrt{d}}$.

   \item  \textbf{Region 2}: $\{r > \rmax\}$ and $\{\norm{V_T} < 1\}$:
   
   Since $\norm{V_T}< 1$ and $\sw_T = \proj(V_T)$, we have $\sw_T = V_T$ . Using \cref{eq:norm-before-proj}, we can write
   \[
    \nonumber
|\Popriskcvx{\sw_T} - \Popriskcvx{\ww_T}| &\geq  | \frac{1}{2}(r^2+2\norm{\ww_T}r\theta(1)) + \frac{\lambda r}{2}\sum_{i\in \range{d}} \theta(i) + \Xi_T|.
    \]
   Then, using the reverse triangle inequality, i.e., $|a-b|\geq |a|-|b|$ for $a,b\in \Reals$, and the facts that $|\theta(1)|\leq 1$ and $\Xi_T\geq 0$, we have
   \[
   \nonumber 
   |\Popriskcvx{\sw_T} - \Popriskcvx{\ww_T}| &\geq \Big|\frac{1}{2}r^2 + \Xi_T\Big| - \Big|r\norm{\ww_T}\theta(1) + \frac{\lambda r}{2} \sum_{i \in \range{d}} \theta(i)\Big|\\
    & \geq \Big|\frac{1}{2}r^2 + \Xi_T\Big| - r\norm{\ww_T}|\theta(1)| - \frac{\lambda r}{2} \Big|\sum_{i \in \range{d}} \theta(i)\Big| \nonumber\\
    &\geq  \frac{1}{2}r^2 + \Xi_T - r\norm{\ww_T} - \frac{\lambda r}{2} \Big|\sum_{i \in \range{d}} \theta(i)\Big| \nonumber\\
    & \geq  \frac{1}{2}r^2 - r\norm{\ww_T} - \frac{\lambda r}{2} \Big|\sum_{i \in \range{d}} \theta(i)\Big|\nonumber.
   \]
    By the Cauchy-Schwartz, we have $\norm{\theta}_1\leq \sqrt{d}$ since $\norm{\theta}_2=1$. Therefore, 
    \[
    \nonumber
      |\Popriskcvx{\sw_T} - \Popriskcvx{\ww_T}| &\geq \frac{1}{2}r^2 - r\norm{\ww_T} - \frac{\lambda r}{2} \sqrt{d}\\
                                                &\geq \frac{1}{2}r^2 - r\norm{\ww_T} - \frac{r}{2n}. \label{eq:region2-interm} 
    \]
    Here, the last line follows from $\lambda\leq \frac{1}{n\sqrt{d}}$. 
    
  Define $g:\Reals \to \Reals $ where $g(x)=x^2/2  -x(\norm{\ww_T}+1/(2n))$. Then, we have  $ \arg\min_{x\in \Reals} g(x) = \norm{\ww_T}+1/(2n)$.   From \cref{lem:norm-out-gd}, we know that $\norm{\ww_T}\leq 1/ \sqrt{n}$. Notice that for $n\geq 5$, we have $\norm{\ww_T} \leq 1/\sqrt{n} \leq 0.5(1-1/(2n))$ which gives us $\norm{\ww_T}+1/(2n)\leq 1-\norm{\ww_T}=\rmax$. Therefore, we conclude that $g$ is increasing for $x \geq \rmax$.  Note that the lower bound in \cref{eq:region2-interm} is $g(r)$ and using this observation we have $g(r)> g(\rmax)$ since in this region $r>\rmax$. Therefore, we can further lower bound \cref{eq:region2-interm} as 
   \[
   \nonumber  |\Popriskcvx{\sw_T} - \Popriskcvx{\ww_T}| &\geq \frac{3}{2}\norm{\ww_T}^2 - \Big(2-\frac{1}{2n}\Big)\norm{\ww_T} +\frac{1}{2}\Big(1-\frac{1}{n}\Big)\\
    &\geq \frac{1}{2} -\frac{2}{\sqrt{n}}.  \label{eq:region2-lowerbound}
   \]
    To prove the last step define $h:\Reals \to \Reals $ where $h(x)=\frac{3}{2}x^2 - \Big(2-\frac{1}{2n}\Big)x +\frac{1}{2}\Big(1-\frac{1}{n}\Big)$. It is straightforward to see that $h(x)$ is decreasing for $x\leq \nicefrac{1}{\sqrt{n}}$ when $n\geq \sqrt{5}$. Using this argument and some manipulations we can show the last step.
    
   \item  \textbf{Region 3}: $\{r > \rmax\}$ and $\{\norm{V_T}\geq 1\}$. 
   
   Since $\norm{V_T}\geq 1$ and $\sw_T = \proj(V_T)$, we have $\norm{\sw_T}=1$. Using this observation and  reverse  triangle inequality, i.e., $|a-b|\geq |a|-|b|$ for $a,b\in \Reals$, we can simplify \cref{eq:residual-term-rot} as
    \[
    \nonumber
|\Popriskcvx{\sw_T} - \Popriskcvx{\ww_T}| &\geq  \frac{1}{2}|1-\norm{\ww_T}^2 + 2\Xi_T| - \frac{\lambda}{2} \Big|\sum_{i\in \range{d}}\sw_T(i)-\norm{\ww_T}\Big|\\
    \nonumber
&  \geq \frac{1}{2}|1-\norm{\ww_T}^2 + 2\Xi_T| - \frac{\lambda}{2} (\norm{\sw_T}_1 + \norm{\ww_T}).
    \]
The last line follows from using the triangle inequality twice. By \cref{lem:norm-out-gd}, we have $\norm{\ww_T}\leq 1/\sqrt{n}$. Then, since $\Xi_T\geq 0$, we obtain
\[
|\Popriskcvx{\sw_T} - \Popriskcvx{\ww_T}| \geq \frac{1}{2}(1+2\Xi_T - \frac{1}{n}) - \frac{\lambda}{2} (\norm{\sw_T}_1 + \frac{1}{\sqrt{n}}). \nonumber
\]
Also, by the Cauchy-Schwartz, we have $\norm{\sw_T}_1\leq \sqrt{d}$ since $\norm{\sw_T}_2=1$. Therefore, setting $\lambda\leq \frac{1}{n\sqrt{d}}$
\[
\nonumber
|\Popriskcvx{\sw_T} - \Popriskcvx{\ww_T}|  &\geq \frac{1}{2}(1-\frac{1}{n}) +\Xi_T - \frac{1}{2n}-\frac{1}{2n^{1.5}\sqrt{d}}\\
&\geq \frac{1}{2}- \frac{1}{n} \nonumber\\
&\geq \frac{1}{2}- \frac{2}{\sqrt{n}} \label{eq:region3-lowerbound}.
\]
Here the last line follows from $\Xi_T\geq 0$ and some simple manipulations.
\end{enumerate}
Equipped with the lower bounds for each region we can conclude this part of the proof. Combining \cref{eq:towerrule-decompos} with \cref{eq:region1-lowerbound}, \cref{eq:region2-lowerbound}, and \cref{eq:region3-lowerbound}, we obtain
\[
 &\EE\sbra{|\Popriskcvx{\sw_T} - \Popriskcvx{\ww_T}|} \geq \EE\sbra{  \Big(\frac{r^2}{4} - \frac{1}{2n}\Big) \indic{T/2 \leq \norm{\badset}_0 \leq T } \indic{r \leq \rmax}} \nonumber \\
 &+   \Big(\frac{1}{2}- \frac{2}{\sqrt{n}}\Big) \EE\sbra{\cPr{\trainset}{r> \rmax} \indic{T/2 \leq \norm{\badset}_0 \leq T } } \label{eq:tower-rule-combined}.
\]
Assume we choose $n$ sufficiently large so that $\frac{(\beta^\star)^2}{16}-\frac{1}{2n}\geq 0$ (Notice that such $n$ always exists).  
We can further lower bound \cref{eq:tower-rule-combined} as
\[
&\EE\sbra{|\Popriskcvx{\sw_T} - \Popriskcvx{\ww_T}|}  \geq \EE\sbra{  \Big(\frac{r^2}{4} - \frac{1}{2n}\Big) \indic{\frac{T}{2} \leq \norm{\badset}_0 \leq T } \indic{\frac{\beta^\star}{2} \leq  r \leq \rmax}}  \nonumber\\
& + \EE\sbra{  \Big(\frac{r^2}{4} - \frac{1}{2n}\Big) \indic{\frac{T}{2} \leq \norm{\badset}_0 \leq T } \indic{ r < \frac{\beta^\star}{2}}} \nonumber \\
&+   \Big(\frac{1}{2}- \frac{2}{\sqrt{n}}\Big) \EE\sbra{\cPr{\trainset}{r > \rmax} \indic{\frac{T}{2} \leq \norm{\badset}_0 \leq T } } \nonumber \\
&\geq  \Big(\frac{(\beta^\star)^2}{16} - \frac{1}{2n}\Big) \EE\sbra{\cPr{\trainset}{\frac{\beta^\star}{2} \leq  r \leq \rmax} \indic{\frac{T}{2} \leq \norm{\badset}_0 \leq T } } \nonumber  \\
& -\frac{1}{2n}  \EE\sbra{ \indic{\frac{T}{2} \leq \norm{\badset}_0 \leq T } \indic{ r < \frac{\beta^\star}{2}}} +  \Big(\frac{1}{2}- \frac{2}{\sqrt{n}}\Big) \EE\sbra{\cPr{\trainset}{r > \rmax} \indic{\frac{T}{2} \leq \norm{\badset}_0 \leq T } } \nonumber\\
& \geq  \Big(\frac{(\beta^\star)^2}{16} - \frac{1}{2n}\Big) \EE\sbra{\cPr{\trainset}{\frac{\beta^\star}{2} \leq  r } \indic{\frac{T}{2} \leq \norm{\badset}_0 \leq T } }  -\frac{1}{2n}  \EE\sbra{ \indic{\frac{T}{2} \leq \norm{\badset}_0 \leq T } \indic{ r < \frac{\beta^\star}{2}}} \nonumber \\
 &\geq \Big(\frac{(\beta^\star)^2}{16} - \frac{1}{2n}\Big) \EE\sbra{\cPr{\trainset}{\frac{\beta^\star}{2} \leq  r } \indic{\frac{T}{2} \leq \norm{\badset}_0 \leq T } }  -\frac{1}{2n} \Pr(r < \frac{\beta^\star}{2}). \label{eq:before-last-step-residual}
\]
Here, we have used  $\frac{1}{2}- \frac{2}{\sqrt{n}}\geq \frac{(\beta^\star)^2}{16} - \frac{1}{2n}$ for $n\geq 14$ where $\beta^\star = 0.1$, and $\frac{r^2}{4}-\frac{1}{2n}\geq -\frac{1}{2n}$ for $r\geq 0$.

 Note that $\trainset \indep r$ by the construction of the surrogate algorithm. By assumption we have $\sigma \geq \frac{\beta^\star}{\sqrt{d}}$. Using the concentration bound from \cref{cor:norm-gauss} we obtain
\[
\nonumber
\Pr\Big(r \leq  \frac{\beta^\star}{2} \Big) \leq  \Pr\Big( r \leq \frac{\sigma \sqrt{d}}{2}  \Big) \leq 2\exp\Big(-\frac{9d}{64}\Big).
\]
Since $r$ and $\trainset$ are independent, by \cref{eq:bad-coordinate-num} we have
\[
 \EE\sbra{\cPr{\trainset}{\frac{\beta^\star}{2} \leq  r } \indic{\frac{T}{2} \leq \norm{\badset}_0 \leq T } }   &= \Pr\Big( \frac{\beta^\star}{2} \leq  r \Big) \Pr\Big( \frac{T}{2} \leq \norm{\badset}_0 \leq T \Big) \nonumber \\
 &\geq (1-2\exp(-9d/64)) \Pr \Big(\frac{T}{2} \leq \norm{\badset}_0 \leq T \Big)\nonumber\\
 &\geq (1-2\exp(-9d/64)) (1-2\exp(-T/36)) \label{eq:lowerprob-residual}.
\]
Therefore, we conclude this part by combining \cref{eq:before-last-step-residual} and \cref{eq:lowerprob-residual} to obtain the following lower bound:
\[
&\EE\sbra{|\Popriskcvx{\sw_T} - \Popriskcvx{\ww_T}|} \nonumber \\ 
&\geq \Big(\frac{(\beta^\star)^2}{16} - \frac{1}{2n}\Big) \big(1-2\exp(-9d/16)-2\exp(-T/36)\big) - \frac{1}{n}\exp(-\frac{9d}{64}).  \nonumber
\]
By setting the parameters, i.e., $T$ and $d$, we prove that for sufficiently large $n$
\[
\EE\sbra{|\Popriskcvx{\sw_T} - \Popriskcvx{\ww_T}|} \in \Omega(1) \nonumber,
\]
which was to be shown.

\subsection{Noise With Small Variance Fails: IOMI }
\label{subsec:small-var}
In \cref{subsec:large-var} we showed that if the variance of $\xi$ is greater than $\big(\frac{\beta^\star}{\sqrt{d}}\big)^2$, the distance between the population risk of the surrogate algorithm and the GD algorithm does not go to zero. In this part, we will show that if the variance of $\xi$ is smaller than $ \frac{\beta^\star}{\sqrt{d}}$ then, the mutual information term does not vanish as $n \to \infty$.

By the definition of the mutual information we can write $\minf{\sw_T;\trainset}= \entr{\trainset} - \entr{\trainset\vert \sw_T}$. Note that $\badset$ is a $\trainset$-measurable random variable. Therefore, we have 
\[
\nonumber
\entr{\trainset\vert \sw_T} = \entr{\trainset,\badset\vert \sw_T}.
\]
Then, by the chain rule for the discrete entropy $\entr{\trainset,\badset\vert \sw_T} = \entr{\badset\vert \sw_T} +\entr{\trainset\vert \sw_T,\badset}$. We claim that
\[
\nonumber
\entr{\trainset\vert \sw_T,\badset} \leq n\EE[(d-\norm{\badset}_0)].
\]
The reason is by conditioning on $\badset$, we know the exact values for the bad coordinates in $\trainset$. Therefore, the cardinally of the possible values for each data-point, conditioned on $\badset$, cannot be more that $2^{d-\norm{\badset}_0}$. Thus,
\begin{equation}
\begin{aligned}
\minf{\sw_T;\trainset}  &\geq \entr{\trainset}-\entr{\badset \vert \sw_T} -n(d-\EE[\norm{\badset}_0]) \\
&=  n \EE[\norm{\badset}_0] - \entr{\badset \vert \sw_T}, \nonumber
\end{aligned}
\end{equation}
where the last line follows from $\entr{\trainset}=nd$ because each element in $\trainset$ is drawn i.i.d. Also, note that $\EE[\norm{\badset}_0] = \EE[ \sum_{i=1}^{d}\badset(i)] = d \EE[\badset(1)]=d2^{-n}$
where in the last line we used the fact that each element of $\trainset$ is \iid, and each column is a bad coordinate with probability $2^{-n}$. Also, with the similar reasoning we obtain $\entr{\badset}=\entr{\badset(1),\dots,\badset(d)}=d\binaryentr{2^{-n}}$, where for $x\in [0,1]$ $\binaryentr{x}=-x\log(x)-(1-x)\log(1-x)$ is the binary entropy function.

Then, we invoke a version of Fano's inequality, provided in \cref{lem:fano}, to obtain
\[
\entr{\badset \vert \sw_T}\leq 1 + \proberror  \entr{\badset} \nonumber
\]
where $\proberror =\inf_{M:\parspace \to \{0,1\}^d} \Pr(M(\sw_T)\neq \badset)$. 
Using the well-known inequality $\binaryentr{x} \leq -x\log(x) + x$, we obtain 
\[
\nonumber
\entr{\badset}\leq d(n2^{-n}+2^{-n})=d(n+1)2^{-n}. 
\]
Therefore, 
\[
\minf{\sw_T;S} &\geq nd2^{-n} - (n+1)d2^{-n} \proberror - 1 \nonumber\\
                &= nd2^{-n} \Big(1- \frac{n+1}{n} \proberror \Big)-1\nonumber\\
                & \geq 1.5 n^3 (1- 2 \proberror)-1\label{eq:minf-simple},
\]
where the last line follows from setting $d=0.75 T2^n$, $T=2n^2$, and $(n+1)/n \leq 2$.

Next, we design an estimator $\Psi$ to \textit{decode} $\badset$ from $\sw_T$ and analyze its probability of error. Let $h = (\eta + \eta \lambda T)/2$. Then, the proposed estimator is given by
\[
\label{eq:def-estimator}
\Psi(w)(i) = 
\begin{cases}
  1   & \text{if} ~ |w(i)| \geq h\\
  0   & \text{if} ~ |w(i)| < h
\end{cases}
\]
for $i \in \range{d}$. In words, it compares each coordinate of $w$ with a given threshold, and if it is larger than $h$, then that coordinate declares as a bad coordinate.

Let $V_T = \ww_T + \xi$. Then,
\[
\nonumber
\proberror &\leq \Pr( \exists i \in [d] ~\text{s.t.}~ \Psi(\sw_T)(i)\neq \badset(i))\\
\label{eq:error-prob-main}
           &\leq \Pr( \{\exists i \in [d] ~\text{s.t.}~ \Psi(\sw_T)(i)\neq \badset(i)\} \cap \{\norm{V_T}\leq 1\}) + \Pr(\norm{V_T}\geq 1)
\]

First we show that $\Pr(\norm{V_T}\geq 1)$ is sufficiently small. From \cref{eq:norm-before-proj}, we have $\norm{V_T}\geq 1 = \norm{\ww_T}^2+r^2+2\norm{\ww_T} r \theta(1)$. Then, as shown in \cref{subsec:large-var} given that $\{r\leq \rmax=1-\norm{\ww_T}\}$, then $\norm{V_T}\leq 1$. Using this we obtain
\[
\nonumber 
 \Pr(\norm{V_T}\geq 1) &= \Pr(\norm{\ww_T}^2+r^2+2\norm{\ww_T} r \theta(1)\geq 1)\\
                        &\leq \Pr( r  \geq 1-\norm{\ww_T}) \nonumber.
 \]
Here $\xi = r \theta$ where $r = \norm{\xi}$ and $\theta = \nicefrac{\xi}{\norm{\xi}}$. Recall from \cref{lem:norm-out-gd} that under the event $\{T/2\leq \norm{\badset}_0\leq T\}$, $1/(2\sqrt{n}) \leq \norm{\ww_T}\leq 1/\sqrt{n}$. Therefore,   
\[
 &\Pr( r  \geq 1-\norm{\ww_T})\leq \EE\sbra{\cPr{\trainset}{r  \geq 1-\norm{\ww_T}} \indic{T/2\leq \norm{\badset}_0\leq T} } + 1-\Pr(T/2\leq \norm{\badset}_0\leq T) \nonumber\\
 &\leq \EE\sbra{\cPr{\trainset}{r  \geq 1-1/(2\sqrt{n})} \indic{T/2\leq \norm{\badset}_0\leq T} } + 1-\Pr(T/2\leq \norm{\badset}_0\leq T)\nonumber\\
  &\leq  \EE\sbra{\cPr{\trainset}{r  \geq 1-1/(2\sqrt{n})} \indic{T/2\leq \norm{\badset}_0\leq T} } + 2\exp(-T/36) \nonumber,
\]
where the last line follows from \cref{eq:bad-coordinate-num}. Observe that 
\[
\nonumber
\{r\geq 1-1/(2\sqrt{n})\} \subseteq \{r\geq 2 \beta^\star\},
\]
due to $\beta^\star=0.1$. Also as $\sigma \leq \beta^\star /\sqrt{d}$, we have
\[
\nonumber
\Pr(r\geq 2 \beta^\star) \leq \Pr(r\geq \sqrt{4 d(\sigma^\star)^2}) \leq  2\exp\Big(-\frac{9d}{16}\Big),
\]
where the last inequality comes from the concentration bounds for $r$ in \cref{cor:norm-gauss}. Since $r\indep \trainset$, we have
\[
\nonumber
\cPr{\trainset}{r  \geq 1-1/(2\sqrt{n})}  \leq    2\exp\Big(-\frac{9d}{16}\Big).
\]
Therefore, 
\[
\label{eq:norm-excced}
\Pr(\norm{V_T}\geq 1) \leq 2\exp\Big(-\frac{9d}{16}\Big) + 2\exp(-T/36).
\]
Since under the event $\norm{V_T}\leq 1$, $\sw_T = \proj(V_T)=\ww_T + \xi$,
\[
&\Pr( \{\exists i \in [d] ~\text{s.t.}~ \Psi(\sw_T)(i)\neq \badset(i)\} \cap \{\norm{V_T}\leq 1\}) \nonumber \\
&= \Pr( \{\exists i \in [d] ~\text{s.t.}~ \Psi(\ww_T +\xi)(i)\neq \badset(i)\} \cap \{\norm{V_T}\leq 1\})\nonumber\\
&\leq \Pr(\exists i \in [d] ~\text{s.t.}~ \Psi(\ww_T +\xi)(i)\neq \badset(i)\})\label{eq:error-prob-inside-ball}.
\]
By the definition of the error probability 
\[
 \Pr( \forall i \in \range{d} ~ \Psi(\ww_T +\xi)(i) = \badset(i)) \geq   \EE[\cPr{\trainset}{ \forall i \in \range{d} ~  \Psi(\ww_T +\xi)(i) = \badset(i)} \indic{T/2 \leq \norm{\badset}_0\leq T}]\label{eq:prob-error-estimator-step0}.
\]
Note that $\ww_T$ and $\badset$ are $\trainset$-measurable. Therefore, the inner probability is only over $\xi$. Also, let $\mathcal{B}=\{i_1,\dots,i_{\norm{\badset}_0}\}$ denote the set of bad coordinates. 
Using the closed-form expression in \cref{lem:dynamics-gd} for $\ww_T$ under the event $\{{T/2 \leq \norm{\badset}\leq T}\}$, we have
\[
\label{eq:prob-error-estimator-step1}
   \cPr{\trainset}{ \forall i \in \range{d} ~  \Psi(\ww_T +\xi)(i) = \badset(i)}= \left(\Pi_{i \in \mathcal{B}} \cPr{\trainset}{\eta + \xi(i)\geq h}\right) \Pi_{i \in \range{d} \setminus \mathcal{B}} \left( \cPr{\trainset}{-\ww_T(i) + \xi(i)\leq h} \right).
\]
This identity follows from $\xi \indep \trainset$ and each  coordinate of $\xi$ are i.i.d. As shown in \cref{lem:dynamics-gd}, under the event $\{{T/2 \leq \norm{\badset}\leq T}\}$, $0\leq -\ww_T(i)\leq \lambda \eta T$; therefore, $-\ww_T(i) < h$ for $i \in \range{d}\setminus \badset$. We can simplify \cref{eq:prob-error-estimator-step1} as
\[
\nonumber
 &\left(\Pi_{i \in \mathcal{B}} \cPr{\trainset}{\eta + \xi(i)\geq h}\right) \Pi_{i \in \range{d} \setminus \mathcal{B}} \left( \cPr{\trainset}{-\ww_T(i) + \xi(i)\leq h} \right)  \\
 &=\Big(1-Q\Big(\frac{\eta - \eta \lambda T}{2\sigma^\star}\Big)\Big)^{\norm{\badset}_0} \Pi_{i \in \range{d} \setminus \mathcal{B}} \Big(1-Q\Big(\frac{h+\ww_T(i)}{\sigma^{\star}}\Big)\Big)\nonumber
\]
where for $x \in \Reals$, $Q(x)=  \frac{1}{\sqrt{2\pi}} \int_{t\geq x}  \exp(-\frac{t^2}{2})\text{d}t$ is the tail distribution function of the Gaussian distribution with mean zero and variance one. Since $Q\Big(\frac{h+\ww_T(i)}{\sigma^{\star}}\Big) \leq Q\Big(\frac{h-\eta \lambda T}{\sigma^{\star}}\Big)=Q\Big(\frac{\eta-\eta \lambda T}{2\sigma^{\star}}\Big)$ for all $i\in \range{d}\setminus \mathcal{B}$, we can further lower bound as 
\[
   \cPr{\trainset}{ \forall i \in \range{d} ~  \Psi(\ww_T +\xi)(i) = \badset(i)} &\geq \left(\Pi_{i \in \mathcal{B}} \cPr{\trainset}{\eta + \xi(i)\geq h}\right) \Pi_{i \in \range{d} \setminus \mathcal{B}} \left( \cPr{\trainset}{\eta \lambda T + \xi(i)\leq h} \right) \nonumber\\
   & \geq \Big(1-Q\Big(\frac{\eta - \eta \lambda T}{2\sigma^\star}\Big)\Big)^{\norm{\badset}_0}  \Big(1-Q\Big(\frac{\eta - \eta \lambda T}{2\sigma^\star}\Big)\Big)^{d-\norm{\badset}_0} \label{eq:prob-error-estimator-step2}, 
\]
More precisely, since $ \eta \lambda T< h < \eta$, we have $\cPr{\trainset}{\eta + \xi(i)\geq h} = \cPr{\trainset}{\xi(i)\geq h - \eta} = 1 -  \cPr{\trainset}{\xi(i)\geq \eta -h}$ and $\cPr{\trainset}{\eta \lambda T + \xi(i)\leq h} = \cPr{\trainset}{\xi(i)\leq h - \eta \lambda T } = 1 -  \cPr{\trainset}{\xi(i)\geq h - \eta \lambda T}$. 

Therefore, we can use \cref{eq:prob-error-estimator-step0}, \cref{eq:prob-error-estimator-step1}, and \cref{eq:prob-error-estimator-step2} to obtain
\[
\Pr( \forall i \in \range{d} ~ \Psi(\ww_T +\xi)(i) = \badset(i))  &\geq  \Big(1-Q\Big(\frac{\eta - \eta \lambda T}{2\sigma^\star}\Big)\Big)^{d} ~ \Pr(T/2 \leq \norm{\badset}_0\leq T) \nonumber \\
&\geq   \Big(1-Q\Big(\frac{\eta - \eta \lambda T}{2\sigma^\star}\Big)\Big)^{d} ~ \Big(1-2\exp\Big(-\frac{T}{36}\Big)\Big),\nonumber
\]
where in the last line we have used \cref{eq:bad-coordinate-num}. 
Note that $\eta-\eta\lambda T\geq 0$ since $\lambda \in \bigO{1/(n\sqrt{d})}$. We can use the well-known inequality $(1-x)^n\geq 1-nx$ for $x\leq 1$ , $n\in \Naturals$ to obtain
\[
\nonumber
&1 - \Pr( \forall i \in \range{d} ~ \Psi(\ww_T +\xi)(i) = \badset(i))   \\
&\leq d Q\Big(\frac{\eta - \eta \lambda T}{2\sigma^\star}\Big) + 2\exp\Big(-\frac{T}{36}\Big) - 2d Q\Big(\frac{\eta - \eta \lambda T}{2\sigma^\star}\Big) \exp\Big(-\frac{T}{36}\Big) \nonumber \\
&\leq d Q\Big(\frac{\eta - \eta \lambda T}{2\sigma^\star}\Big) + 2\exp\Big(-\frac{T}{36}\Big) \nonumber,
\]
Then, we invoke the inequality $Q(x)\leq \frac{1}{2}\exp(-\frac{x^2}{2})$ for $x\geq 0$ \citep[][Ex.2.2]{wainwright2019high}, to further upper bound the last equation as follows:
\[
\label{eq:prob-error-estimator-step3}
1 - \Pr( \forall i \in \range{d} ~ \Psi(\ww_T +\xi)(i) = \badset(i)) \leq  \frac{d}{2} \exp\Big(-\frac{d (\eta - \eta \lambda T)^2 }{2(\beta^\star)^2}\Big) + 2\exp\Big(-\frac{T}{36}\Big).
\]

Finally, by combining \cref{eq:error-prob-main}, \cref{eq:norm-excced}, \cref{eq:error-prob-inside-ball}, and \cref{eq:prob-error-estimator-step3}, we obtain
\[
\proberror \leq \frac{d}{2} \exp\Big(-\frac{d (\eta - \eta \lambda T)^2 }{2(\beta^\star)^2}\Big) + 4\exp\Big(-\frac{T}{36}\Big) +  2\exp\Big(-\frac{9d}{16}\Big). \nonumber
\]
By setting the parameters and some simple manipulations, we obtain
\[
\label{eq:proberror-final}
\proberror \leq n^2 2^n \exp(- 2^n /n) + 6\exp(-n^2/18) .
\]
In \cref{fig:errorprob}, we plot the upper bound in \cref{eq:proberror-final}. As can be seen the upper bound is decreasing and smaller than $0.1$ for $n\geq 10$.
\begin{figure}[H]
    \centering
    \includegraphics[scale=0.45]{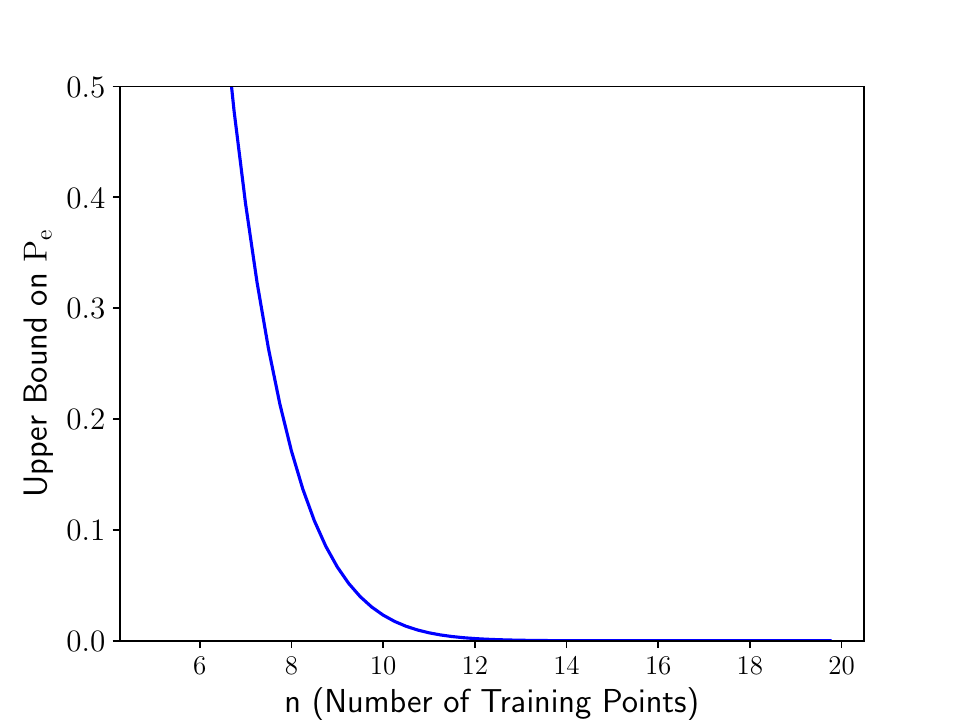}
    \caption{The upper bound in \cref{eq:proberror-final}}
    \label{fig:errorprob}
\end{figure}
Finally, combining \cref{eq:proberror-final} with \cref{eq:minf-simple}, we conclude that for $n\geq 10$ if $\sigma \leq \beta^\star/\sqrt{d}$, we have
\[
\nonumber
\minf{\sw_T;\trainset}\geq 1.2n^3 - 1,
\]
which was to be shown.

\subsection{Noise with Small Variance Fails: CMI}
\label{subsec:small-var-cmi}
In this part of the proof we aim to show that if the variance of the noise is smaller than  $\frac{(\beta^\star)^2}{d}$, then $\cmi$ grows linearly with $n$. We begin this part of the proof with a key lemma. 

We recall the definition of bad coordinates. 
For $i\in \range{d}$, we say the $i-$th coordinate is a \emph{bad coordinate} iff for all $j \in \range{n}$, $Z_{j}(i)=0$. In words, if $i-$th coordinate is a bad coordinate then all the entries in the $i-$th column of $\trainset$ is zero. 
Let $\badset \in \{0,1\}^d$ denote a vector such that $\badset(i)=1$ if and only if $i$ is a bad coordinate. Also $\norm{\badset}_0$ denotes the number of bad coordinates.

Next, we provide a result which shows that $U$ can be identified with high accuracy by having access to the supersample and bad coordinates. The intuition behind the result is as follows. 
Consider a decision making problem where by having access to $\badset$ and matrix of $\supersample$, we want to find which subset of the supersample contained in the training set, i.e., find $U$. 
First, note that, by definition, in each column of $\supersample$ exactly one sample is chosen for the training set. Also, by the definition of the bad coordinates, we know that if $i \in \range{d}$ is a bad coordinate, then for all $Z \in \trainset$, we have $Z(i)=0$. 
In the next theorem we show that the \emph{uncertainty} about $U$ is small conditioned on $\badset,\supersample$. The idea of the proof is to show that by only considering the bad coordinates we can \emph{distinguish} between the points in each column of the supersample.

\begin{lemma}
\label[lemma]{lem:cmi-badset}
$\entr{U\vert \badset,\supersample} \leq  n \EE[2^{-\norm{\badset}_0}]$.
\end{lemma}
\begin{proof}
Let  $\mathcal{B} = \{i_1,\dots,i_{\norm{\badset}_0}\}\subseteq \range{d}$ contains the \textit{ordered} set of bad coordinates.
 For every $k \in \range{n}$, define the following indicator random variable
\[
\nonumber
J_k = \indic{ \exists i \in \mathcal{B}~\text{s.t.}~ \tilde{Z}_{0,k}(i) \neq \tilde{Z}_{1,k}(i)   }
\]
Let $J=(J_1,\dots,J_n)\in \{0,1\}^n$. Note that $J$ is $(\supersample,\badset)$-measurable. 

The main observation here is that provided that $J_k=1$, then we can perfectly recover $U_k$. The reason is as follows: in each column of $\supersample$, exactly one sample is a member of the training set. Also, since we know $\badset$, the values of the bad coordinates are known for the points in the training set by the definition of bad coordinates. Therefore,  $J_k=1$ iff one of the point in the $k$-th column of $\supersample$ does not have zero on the indices in $\mathcal{B}$, which reveals the sample that is not in the training set. 
Therefore, as $J$ is $(\supersample,\badset)$-measurable, we can write
\begin{equation*}
\begin{aligned}
    \entr{U\vert \badset,\supersample} &= \entr{U\vert \badset,\supersample, J }\\
                                   &= \entr{(U)_{\{i \vert J_i=0\}},(U)_{\{i \vert J_i=1\}}\vert \badset,\supersample, J } \\
                                   &=  \entr{(U)_{\{i \vert J_i=0\}}\vert \badset,\supersample, J },
\end{aligned}
\end{equation*}
where the last line follows from $(U)_{\{i \vert J_i=1\}}$ being known from $J$. Since the cardinality of the support of $(U)_{\{i\vert J_i=0\}}$ is no more than $2^{n-\norm{J}_0}$, we obtain 
\[
\nonumber
 \entr{U\vert \badset,\supersample}\leq  n-\EE[\norm{J}_0].
\]
Then, we claim that
\[
\nonumber
\Pr(J_k=1) & = \EE[1-2^{-\norm{\badset}_0}]\nonumber.
\]
This claim conclude the proof since $\EE[\norm{J}_0]=\sum_{k=1}^{n}\EE[J_k]=\sum_{k=1}^{n}\Pr(J_k=1)$.

To prove the claim: $J_k=0$ iff, conditioned on the $U$ and $\badset$, for all $j$ such that $\badset_j=1$, $Z_{1-U_k,k}(j)=0$.
By the definition of the supersample, the points in the supersample are \iid~, independent of $U$, and drawn from $\bernoulli(1/2)$. Hence, 
\[
\nonumber
\Pr(J_k=0)=\EE[\cPr{U,\trainset}{J_k=0}]=\EE[2^{-\norm{\badset}_0}].
\]

\end{proof}
By the definition of mutual information, we have
\[
\nonumber
\cmi &= \entr{U\vert \supersample}- \entr{U\vert \sw_T,\supersample}\\
    &=\entr{U}- \entr{U\vert \sw_T,\supersample}\nonumber\\
    &=n- \entr{U\vert \sw_T,\supersample},\label{eq:cmi-simplify}
\]
where the second and third steps follow from $U \indep \supersample$ and $\entr{U}=n$, respectively. To analyze the second term in \cref{eq:cmi-simplify}, consider the following equality which comes from the chain rule:
\[
\nonumber
\entr{U,\badset\vert \sw_T,\supersample}&= \entr{U\vert \sw_T,\supersample}+\entr{\badset\vert U,\sw_T,\supersample}\\
                                &=\entr{\badset\vert \sw_T,\supersample}+\entr{U\vert \sw_T,\supersample,\badset}. \nonumber
\]
Notice that $\entr{B\vert U,\sw_T,\supersample}=0$ as $\badset$ is $(U,\supersample)$-measurable. Therefore, 
\[
\label{eq:cmi-decompose-secondterm}
\entr{U\vert \sw_T,\supersample} =\entr{\badset\vert \sw_T,\supersample}+\entr{U\vert \sw_T,\supersample,\badset}.
\]
To analyze the first term, note that conditioning cannot increase the entropy. Therefore, we have $\entr{\badset\vert \sw_T,\supersample}\leq \entr{\badset\vert \sw_T}$. Then, we invoke the Fano's inequality from \cref{lem:fano} to obtain
\[
\nonumber
\entr{\badset \vert \sw_T}\leq 1 + \proberror  \entr{\badset}.
\]
Here,  $\proberror =\inf_{M:\parspace \to \{0,1\}^d} \Pr(M(\sw_T)\neq \badset)$. Consider the estimator $\Psi$ proposed in \cref{eq:def-estimator}. We analyzed its probability of error in \cref{subsec:small-var} and obtained in \cref{eq:proberror-final} that
\[
\nonumber
\proberror  \leq n^2 2^n \exp(- 2^n /n) + 6\exp(-n^2/18).
\]
Note that $\entr{\badset}\leq d(n+1)2^{-n}\leq 2n^3$ for $n\geq 3$ as shown in \cref{subsec:small-var}. Therefore, 
\[
\label{eq:cmi-entr-firstterm-final}
\entr{\badset\vert \sw_T,\supersample} \leq \entr{\badset \vert \sw_T} \leq 2n^3(n^2 2^n \exp(- 2^n /n) + 6\exp(-n^2/18) ) + 1.
\]
Next, we analyze the second term in  \cref{eq:cmi-decompose-secondterm}. Using \cref{lem:cmi-badset} we have 
\[
\label{eq:cmi-entr-secondterm-1}
\entr{U\vert \sw_T,\supersample,\badset}  \leq \entr{U\vert \supersample,\badset} \leq   n\EE[2^{-\norm{\badset}_0}].
\]
The, consider
\[
\nonumber
\EE[2^{-\norm{\badset}_0}] = \EE[2^{-\norm{\badset}_0} \indic{T/2\leq \norm{\badset}_0\leq T}] +  \EE[2^{-\norm{\badset}_0} (\indic{\norm{\badset}_0< T/2} + \indic{\norm{\badset}_0> T} )].
\]
The second term can be upper bounded by $\Pr(\{\norm{\badset}_0< T/2\} \cup \{\norm{\badset}_0> T\})$, and this probability is less than $2\exp(-T/36)$ as shown in \cref{eq:bad-coordinate-num}. By simply upper bounding the first term by the worst-case realization, we can write 
\[
\nonumber
\EE[2^{-\norm{\badset}_0}] &\leq 
\EE[2^{-T/2} \indic{T/2\leq \norm{\badset}_0\leq T}] +  \Pr(\{\norm{\badset}_0< T/2\} \cup \{\norm{\badset}_0> T\})\\
\label{eq:bad-coordinate-exponent}
&\leq 2^{-T/2} + 2\exp(-T/36).
\]
Finally, by \cref{eq:cmi-entr-secondterm-1} and \cref{eq:bad-coordinate-exponent}, we obtain
\[
\label{eq:cmi-entr-secondterm-final}
\entr{U\vert \sw_T,\supersample,\badset} \leq n(2^{-T/2} + 2\exp(-T/36)).
\]
The last step is combining \cref{eq:cmi-decompose-secondterm}, \cref{eq:cmi-entr-firstterm-final}, and \cref{eq:cmi-entr-secondterm-final} to lower bound $\cmi$ as 
\[
\nonumber
\cmi &= n-\entr{U\vert \sw_T,\supersample}\\
     &\geq n - \sbra{n2^{-n^2} + n\exp(-n^2/18) +  2n^5 2^n \exp(- 2^n /n) + 12n^3\exp(-n^2/18)  + 1} \label{eq:final-upper-bound-cmi}
\]
\cref{fig:cmi-entr} shows the upper bound on $n-\cmi$ in \cref{eq:final-upper-bound-cmi} as a function of $n$. As seen for $n\geq 16$, $\cmi \geq n-1.1 $, and the lower bound on $\cmi$ is increasing.
\begin{figure}[H]
    \centering
    \includegraphics[scale=0.45]{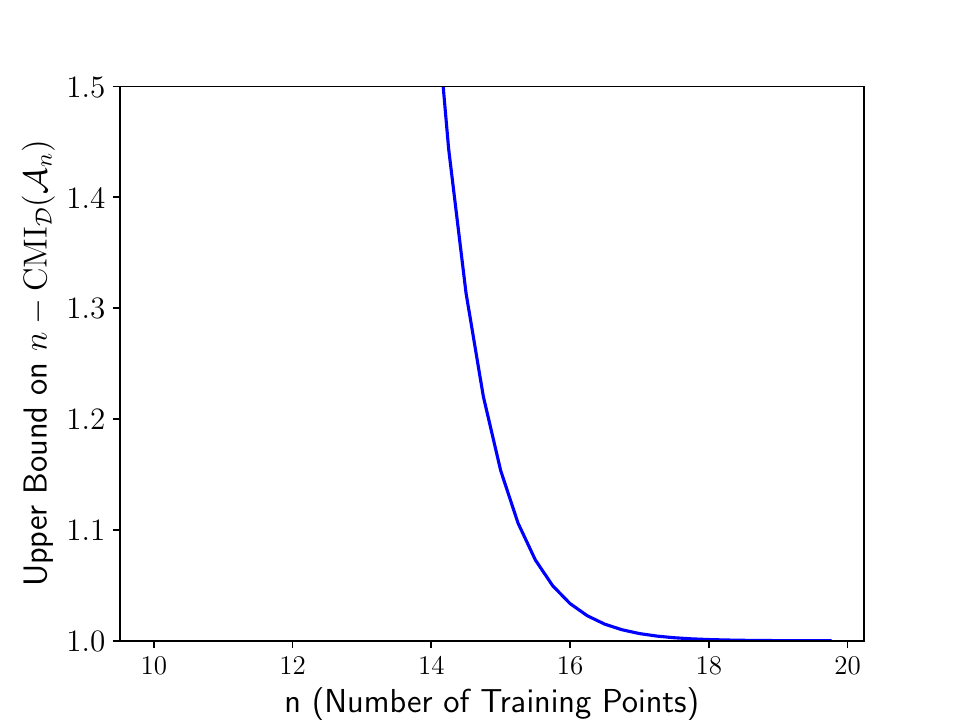}
    \caption{Upper bound on $n-\cmi$ in \cref{eq:final-upper-bound-cmi}.}
    \label{fig:cmi-entr}
\end{figure}
Hence, we obtain 
\[
\nonumber
\cmi\geq \Omega(n),
\]
which was to be shown.

\section{Proof of \cref{thm:main-pacbayes} }
The construction for proving this theorem is exactly the same as in \cref{thm:gd-noisy}.
\label{appx:pac-bayes-pf}
\subsection{Lower Bound on the Residual}
For the case that $\sigma^2 \leq \mathrm{var}_{(n)}^\star$, we showed in \cref{thm:gd-noisy} that for sufficently large $n$, $\EE[\respop(\ww_T) +\resemp(\ww_T)]=M_{\text{res}} \in \Omega(1)$. Since the loss function is $4$-Lipschitz and the space has radius of $1$, we have
\[
\nonumber
\respop(\ww_T) + \resemp(\ww_T)\leq 2L \norm{\ww_T-\sw}\leq 4LR=16 \quad \as
\]
Then, we invoke \cref{lem:reverse-markov} with $m=\tilde{m}=16$ and $a=M_{\text{res}}/2$ to obtain 
\[
\label{eq:hp-lowerbound-res}
\Pr(\respop(\ww_T) + \resemp(\ww_T) > M_{\text{res}}/2) \geq \frac{M_{\text{res}}}{32}.
\]
\subsection{Lower Bound on the Conditional-PAC Bayes Bound}
First of all,  \cref{lem:mixture-num} implies that $\cEE{\trainset}{[\KL{Q(\trainset)}{\frac{1}{2^n}\sum_{u \in \{0,1\}^n}Q(\tilde{S}_u)}]}$ is bounded by $n$ a.s. 
In \cref{thm:gd-noisy} we showed that given $\sigma^2 \geq \mathrm{var}_{(n)}^\star$ for sufficiently large $n$, we have 
$$
\EE[\cEE{\trainset}{[\KL{Q(\trainset)}{\frac{1}{2^n}\sum_{u \in \{0,1\}^n}Q(\tilde{S}_u)}]}]=\cmi \geq 0.2n.
$$
Then, we use \cref{lem:reverse-markov}, with the following parameters: $\tilde{m}=m=n$ and $a=0.1 n$ to obtain  
\[
\label{eq:hp-lowerbound-cmi}
\displaystyle \Pr\Big( \cEE{\trainset}{[\KL{Q(\trainset)}{\frac{1}{2^n}\sum_{u \in \{0,1\}^n}Q(\tilde{S}_u)}]} > 0.1 n\Big)\geq \frac{\cmi-0.1n}{n-0.1n}\geq \frac{1}{9}.
\]
\subsection{Lower Bound on the Classical PAC-Bayes Bound}
For every $s\in \{0,1\}^{n\times d}$, let $Q(s)$ denote the posterior, and $P=\EE[Q(\tilde{\trainset})]$ denote the prior. By construction, the training set $\trainset$ takes all the values in $\{0,1\}^{n\times d}$ uniformly at random. Therefore, by \cref{lem:mixture-num}, we have  
\[
\nonumber
\KL{Q(\trainset)}{P}=\KL{Q(\trainset)}{\frac{1}{2^{nd}}\sum_{s \in \{0,1\}^{n\times d}}Q(s)}\leq nd \quad  \text{a.s.}
\]
Consider the estimator $\Psi: \parspace \to \{0,1\}^d$ in \cref{eq:def-estimator}. For every $s \in \{0,1\}^d$, let $\hat{Q} :\{0,1\}^{n\times d} \to \ProbMeasures{\{0,1\}^d}$ be the pushforward of $Q(s)$ through $\Psi$. Similarly, we can define $\hat{P} \in \ProbMeasures{\{0,1\}^d}$ as the pushforward of $P$ using $\Psi$. 

By the data-processing inequality for the KL divergence \citep{polyanskiy2014lecture}, we have
\[
\label{eq:dataprocessing-kl}
\KL{Q(S)}{P} \geq \KL{\hat{Q}(S)}{\hat{P}} \quad \as
\]
We claim that $\hat{P}=\EE[\hat{Q}(\trainset)]$. By a slight abuse of notation, for every $b\in \{0,1\}^d$, let $\hat{P}(b)$ denote the probability assigned to $b$ by $\hat{P}$. Also, for every $s \in \{0,1\}^{n\times d}$ and a (measurable) set $A \subseteq \parspace$ let $Q(s)(A)$ be the measure assigned to set $A$ by $Q(s)$.  Similarly, we can define $P(A)$. Equipped with these notations, we can write
\begin{align*}
\hat{P}(b)&= \int_{w\in \parspace} P(dw) \indic{\Psi(w)=b}\\
          &=  \int_{w\in \parspace} \frac{1}{2^{nd}} \sum_{s \in \{0,1\}^{n \times d}} Q(s)(dw) \indic{\Psi(w)=b}\\
          &=\frac{1}{2^{nd}} \sum_{s \in \{0,1\}^{n \times d}} \int_{w\in \parspace} Q(s)(dw) \indic{\Psi(w)=b}.
\end{align*}
Here, the second step is by the definition of the prior, and the last step follows from Fubini's theorem. Notice that the expression in the last step is $\EE[\hat{Q}(\trainset)]$ as was to be shown.

Recall the definition of the bad coordinates. Define the \emph{bad coordinate profile} of $s \in \{0,1\}^{n\times d}$ as a binary vector of length $d$ such that its $i-$coordinate is one if and only if $i$ is a bad coordinate, and it is zero otherwise. For every $b \in \{0,1\}^d$, define set 
$$
\mathcal{S}_b=\{s\in \{0,1\}^{n \times d} \vert  \text{bad coordinate profile of $s$ is $b$ } \}.
$$
By construction, each coordinate is a bad coordinate independently with probability $2^{-n}$. Therefore
\[
\label{eq:prob-badcoord-profile}
\Pr(\trainset \in  \mathcal{S}_b) = 2^{-n\norm{b}_0} (1-2^{-n})^{d-\norm{b}_0}.
\]

In what follows, for every $b \in \{0,1\}^d$ that satisfies $T/2 \leq \norm{b}_0\leq T$, we provide an upper bound on $\KL{\hat{Q}(s)}{\hat{P}}$ given $s \in \mathcal{S}_b$. We can write
\[
\nonumber
\KL{\hat{Q}(s)}{\hat{P}} &= \KL{\hat{Q}(s)}{\EE[\hat{Q}(\trainset)]}\\
&\leq \KL{\hat{Q}(s)}{ \Pr(\trainset \in \mathcal{S}_b ) \cEE{\trainset\in \mathcal{S}_b }{[\hat{Q}(\trainset)]} + \Pr(\trainset \not\in \mathcal{S}_b ) \cEE{\trainset\not\in \mathcal{S}_b }{[\hat{Q}(\trainset)}] }.  \nonumber
\]
The last line follows from the law of total expectation. Then, we invoke \cref{lem:mixture-bound}, to obtain 
\[
&\KL{\hat{Q}(s)}{ \Pr(\trainset \in \mathcal{S}_b ) \cEE{\trainset\in \mathcal{S}_b }{[\hat{Q}(\trainset)]} + \Pr(\trainset \not\in \mathcal{S}_b ) \cEE{\trainset\not\in \mathcal{S}_b }{[\hat{Q}(\trainset)]}  }  \nonumber \\
&\leq -\log(\Pr(\trainset \in \mathcal{S}_b )) +  \KL{\hat{Q}(s)}{ \cEE{\trainset \in \mathcal{S}_b}{[\hat{Q}(\trainset)]}}. \nonumber
\]
First, we analyze $\log(\Pr(\trainset \in \mathcal{S}_b ))$. By \cref{eq:prob-badcoord-profile}, we have
\[
\nonumber
-\log(\Pr(\trainset \in \mathcal{S}_b )) = n\norm{b}_0 + (d-\norm{b}_0) \log\Big(\frac{1}{1-2^{-n}}\Big). 
\]
Since $T/2 \leq \norm{b}_0\leq T$, we have $n\norm{b}_0  \leq nT$. Then, using the inequality $-\log(1-x)\leq \frac{x}{1-x}$ for $x\leq 1$, we obtain  $-\log(1-2^{-n}) \leq \nicefrac{2^{-n}}{1-2^{-n}}$.   Therefore, 
\begin{align*}
    (d-\norm{b}_0) \log\Big(\frac{1}{1-2^{-n}}\Big) &\leq d \log\Big(\frac{1}{1-2^{-n}}\Big)\\
                                            & \leq  \frac{d 2^{-n}}{1-2^{-n}}\\
                                            & \leq 2d 2^{-n}.
\end{align*}
Finally setting $d=0.75 T 2^n$, we obtain the following upper bound
\[
\label{eq:log-prob-term}
-\log(\Pr(\trainset \in \mathcal{S}_b )) \leq \frac{5}{2} nT.
\]
Next, we provide an upper bound on $\KL{\hat{Q}(s)}{ \cEE{\tilde{S} \in \mathcal{S}_b}{[\hat{Q}(\tilde{S})]}}$. In \cref{eq:prob-error-estimator-step1}, we analyzed the error probability of the estimator $\Psi$ conditioned on the training set. In particular, we proved that for every training set whose number of bad coordinates is between $T/2$ and $T$, we have almost surely
\[
\nonumber
\cPr{\trainset}{ \exists i \in \range{d} ~  \Psi(\ww_T +\xi)(i) \neq b(i)} &\leq n^2 2^n \exp(-2^n/n)\\
            &\triangleq p_{\text{error}}.   \nonumber
\]
It implies that for all $s \in \mathcal{S}_b$ with $T/2\leq \norm{b}_0\leq T$, $\hat{Q}(s)(b)\geq 1- p_{\text{error}}$ and $\sum_{b'\neq b}\hat{Q}(s)(b')\leq p_{\text{error}}$. For notational convenience let $\cEE{\trainset \in \mathcal{S}_b}{[\hat{Q}(\trainset)]}\triangleq Q_b$. By the definition of the KL divergence, we can write
\[
\KL{\hat{Q}(s)}{\hat{Q}_b} &= \sum_{b' \in \{0,1\}^d} \hat{Q}(s)(b') \log\Big(\frac{\hat{Q}(s)(b')}{\hat{Q}_b(b')}\Big) \nonumber \\
&= \hat{Q}(s)(b) \log\Big(\frac{\hat{Q}(s)(b)}{\hat{Q}_b(b)}\Big)  + \sum_{b' \in \{0,1\}^d,b'\neq b}  \hat{Q}(s)(b') \log\Big(\frac{\hat{Q}(s)(b')}{\hat{Q}_b(b')}\Big). \label{eq:kl-bad-coordinates}
\]
Since for all $s \in \mathcal{S}_b$, $\hat{Q}(s)(b)\geq 1- p_{\text{error}}$, we have $ \hat{Q}_b(b)\geq 1-p_{\text{error}}$. Therefore, we have
\[
\nonumber
\hat{Q}(s)(b) \log\Big(\frac{\hat{Q}(s)(b)}{\hat{Q}_b(b)}\Big) &\leq \hat{Q}(s)(b) \log\Big(\frac{\hat{Q}(s)(b)}{1- p_{\text{error}}}\Big)\\
&\leq -\log(1-p_{\text{error}}) \label{eq:kl-term-hp-1}.
\]
The last step follows from $0 \leq \hat{Q}(s)(b) \leq 1$. Conditioned on $\trainset \in \mathcal{S}_b$, the distribution of the training set is uniform over the set $\mathcal{S}_b$. Using this observation, for every $b' \in \{0,1\}^d$, we can write
\[
\nonumber
\log\Big(\frac{\hat{Q}(s)(b')}{\hat{Q}_b(b')}\Big) & =  \log\Big(\frac{\hat{Q}(s)(b')}{\frac{1}{|\mathcal{S}_b|}\sum_{b' \in \mathcal{S}_b}\hat{Q}(s)(b')} \Big)  \nonumber \\
&\leq \log(|\mathcal{S}_b|) \nonumber\\
&\leq \log(2^{nd})\nonumber.
\]
Therefore, we have 
\[
\nonumber
\sum_{b' \in \{0,1\}^d,b'\neq b}  \hat{Q}(s)(b') \log\Big(\frac{\hat{Q}(s)(b')}{\hat{Q}_b(b')}\Big) &\leq nd \sum_{b' \in \{0,1\}^d,b'\neq b}  \hat{Q}(s)(b') \\
&\leq nd p_{\text{error}}. \label{eq:kl-term-hp-2}
\]
By \cref{eq:log-prob-term}, \cref{eq:kl-bad-coordinates}, \cref{eq:kl-term-hp-1}, and \cref{eq:kl-term-hp-2}, we obtain
\[
-\log(\Pr(\trainset \in \mathcal{S}_b )) +  \KL{\hat{Q}(s)}{ \cEE{\tilde{S} \in \mathcal{S}_b}{[\hat{Q}(\tilde{S})]}} &\leq \frac{5}{2}nT  -\log(1-p_{\text{error}}) + nd p_{\text{error}} \nonumber\\
&\leq \frac{5}{2}nT + \frac{p_{\text{error}}(nd+1)}{1-p_{\text{error}}}. \nonumber
\]
Setting the parameters, we can see that $\frac{p_{\text{error}}(nd+1)}{1-p_{\text{error}}} \leq 1$ for $n\geq 8$. 

Thus, we obtain that for every $s \in \mathcal{S}_b$ such that $T/2\leq \norm{b}_0\leq T$, we have
\[
\label{eq:good-ub-sb}
\KL{\hat{Q}(s)}{\hat{P}} \leq \frac{5}{2}nT + 1,
\]
for $n\geq 8$.

Note that the upper bound in \cref{eq:good-ub-sb} provides a \emph{uniform} upper bound for every $s \in \mathcal{S}_b$ such that $T/2\leq \norm{b}_0\leq T$. Therefore, by a simple contraposition we have
\[
\nonumber
&\{s \in \{0,1\}^{n \times d} \vert \text{the number of bad coordinates of $s \in \left\{T/2,\dots,T \right\}$  } \} \nonumber \\
 &\subseteq \{s \in \{0,1\}^{n \times d} \vert  \KL{\hat{Q}(s)}{\hat{P}} \leq \frac{5}{2}nT + 1\}. \nonumber
\]
By considering the complement of the above statement we obtain
\[
\nonumber
 & \{s \in \{0,1\}^{n \times d} \vert  \KL{\hat{Q}(s)}{\hat{P}} > \frac{5}{2}nT + 1\} \nonumber\\
 &\subseteq \{s \in \{0,1\}^{n \times d} \vert \text{the number of bad coordinates of $s \notin \left\{T/2,\dots,T \right\}$  } \}. \nonumber 
\]
Therefore, we have 
\[
\Pr( \KL{\hat{Q}(\trainset)}{\hat{P}} > \frac{5}{2}nT + 1) &\leq 1-\Pr({T}/{2}\leq \norm{\badset}_0 \leq T) \nonumber\\
                                                            &\leq  2\exp(-{T}/{36}) \nonumber. 
\]
Here, the line follows from \cref{eq:bad-coordinate-num}.

Next, we provide a lower bound on $\EE[\KL{\hat{Q}(\trainset)}{\hat{P}}]$. Let random variable $\badset$ denote the bad coordinate profile of $\trainset$. Notice that $\EE[\KL{\hat{Q}(\trainset)}{\hat{P}}]=\minf{\trainset;\hat{\badset}}$ where $\hat{\badset}$ is the estimate of $\badset$ using the estimator $\Psi$. We have $\minf{\hat{\badset};\trainset} = \entr{\trainset} - \entr{\trainset\vert \badset}$. By construction, $ \entr{\trainset}  = nd$. Since $\badset$ is a function of $\trainset$, we have $ \entr{\trainset,\badset \vert \hat{\badset}}=\entr{\trainset \vert \hat{\badset}}$. Then, by the chain rule for the entropy we can write $\entr{\trainset,\badset \vert \hat{\badset}} = \entr{\badset \vert \hat{\badset}} + \entr{\trainset \vert \badset, \hat{\badset}}$. By conditioning on $\badset$, we know the exact values for the bad coordinates in $\trainset$. Therefore, the cardinally of the possible values for each data-point, conditioned on $\badset$, cannot be more that $2^{d-\norm{\badset}_0}$. Therefore, we have $\entr{\trainset \vert \badset, \hat{\badset}} \leq n(d-\EE[\norm{\badset}_0])$ which gives us $\entr{\badset \vert \hat{\badset}} + \entr{\trainset \vert \badset, \hat{\badset}}\leq \entr{\badset \vert \hat{\badset}} + n(d-\EE[\norm{\badset}_0])$. By Fano's inequality in \cref{lem:fano}, we have $\entr{\badset \vert \hat{\badset}}\leq 1+ \Pr(\hat{\badset}\neq \badset) \entr{\badset}$.   Therefore, we obtain 
\[
\nonumber
\EE[\KL{\hat{Q}(\trainset)}{\hat{P}}] &=  \minf{\trainset;\hat{\badset}} \nonumber\\
&\geq n \EE[\norm{\badset}_0] -1 - \Pr(\hat{\badset}\neq \badset) \entr{\badset} \nonumber \\
\nonumber &\geq nd2^{-n} - (n+1)d2^{-n} \Pr(\hat{\badset}\neq \badset) - 1\\
\nonumber &\geq 1.5 n^3 (1-2\Pr(\hat{\badset}\neq \badset))-1.
\]
Here, we used the following facts. $\EE[\norm{\badset}_0] = \EE[ \sum_{i=1}^{d}\badset(i)] = d \EE[\badset(1)]=d2^{-n}$ since each element of $\trainset$ is \iid and each column is a bad coordinate with probability $2^{-n}$. Also, with the similar reasoning we obtain $\entr{\badset}=\entr{\badset(1),\dots,\badset(d)}=d\binaryentr{2^{-n}}$, where for $x\in [0,1]$ $\binaryentr{x}=-x\log(x)-(1-x)\log(1-x)$ is the binary entropy function. Also, we have used the the well-known inequality $\binaryentr{x} \leq -x\log(x) + x$. Then, our analysis of the error probability of the estimator $\Psi$ in \cref{subsec:small-var} implies that for $n\geq 10$, the following lower bound holds:
\[
\nonumber
\EE[\KL{\hat{Q}(\trainset)}{\hat{P}}] \geq  1.2n^3 - 1.
\]
In the next step, we invoke \cref{lem:reverse-markov} with the following parameters $\hat{m}=nd$, $m=\frac{5}{2}nT + 1=5n^3+1$, and $a = 0.6 n^3 - 0.5$ to write
\[
&\Pr(\KL{\hat{Q}(\trainset)}{\hat{P}} \geq 0.6 n^3 - 0.5 ) \nonumber \\
&\geq \frac{\EE[\KL{\hat{Q}(\trainset)}{\hat{P}}]-a - (nd-(\frac{5}{2}nT + 1))\Pr(X\geq \frac{5}{2}nT + 1)}{5n^3 +1 - a} \nonumber \\
&\geq  \frac{0.6 n^3 - 0.5 - 3n^3 2^n \exp(-n^2/18)}{4.4 n^3 + 1.5}\nonumber.
\]
By numerical evaluations, we can see that the lower bound is greater than $0.1$ for $n\geq 16$. 

From \cref{eq:dataprocessing-kl}, we have 
\[
\Pr(\KL{Q(\trainset)}{P} > 0.6 n^3 - 0.5 ) &\geq \Pr(\KL{\hat{Q}(\trainset)}{\hat{P}} \geq 0.6 n^3 - 0.5 ) \nonumber \\
&\geq 0.1, \label{eq:hp-lowerbound-mi}
\]
for $n\geq 16$ as was to be shown.
\subsection{Concluding the Proof}
In summary, in \cref{eq:hp-lowerbound-res}, \cref{eq:hp-lowerbound-cmi}, and \cref{eq:hp-lowerbound-mi}, we have shown there exist constants $\alpha_1 \in \mathbb{R}_+$, $\alpha_2 \in \mathbb{R}_+$, $\alpha_3 \in \mathbb{R}_+$, $\beta_1 \in (0,1)$, and $\beta_2 \in (0,1)$ such that for sufficiently large $n$,
\begin{enumerate}
    \item $\Pr\Big(\respop(\ww_T) + \resemp(\ww_T) > \alpha_1~\text{or}~\frac{\cEE{\trainset}{[\KL{Q(\trainset)}{\frac{1}{2^n}\sum_{u \in \{0,1\}^n}Q(\tilde{S}_u)}]}}{n} > \alpha_2\Big)\geq 1-\beta_1$.
    \item $\Pr\Big(\respop(\ww_T) + \resemp(\ww_T) > \alpha_1~\text{or}~\frac{\KL{Q(\trainset)}{\EE[Q(\trainset)]}}{n} > \alpha_3 \Big)\geq 1-\beta_2$.
\end{enumerate}
For notational convenience, let $\mathrm{Bad~Event}_1$ and  $\mathrm{Bad~Event}_2$ denote the first and second event above.

Next, we show how this result implies the failure of PAC-Bayes bounds. Consider the decomposition of the generalization error of GD with respect to the surrogate
\[
\nonumber
\cEE{\trainset} \sbra{\Popriskcvx{\ww_T}- \Empriskcvx{\ww_T}} &\leq \cEE{\trainset} \sbra{\Popriskcvx{\sw}- \Empriskcvx{\sw}} + \resemp(\ww_T) + \respop(\ww_T), 
\]
Let $\complexity{n}$ denote both $C_\textnormal{clas}(n)\triangleq \KL{Q(\trainset)}{\EE[Q(\trainset)]}$ and $C_\textnormal{cond}(n)\triangleq\cEE{\trainset}{[\KL{Q(\trainset)}{\frac{1}{2^n}\sum_{u \in \{0,1\}^n}Q(\tilde{S}_u)}]}$. Let $\delta < 1-\max\{\beta_1,\beta_2\}$. Assume we instantiate the PAC-Bayes bounds with the confidence of $1-\delta$. Then, by a simple application of the union bound we have
\[
\Pr\bigg(\bigg\{\cEE{\trainset} \sbra{\Popriskcvx{\sw}- \Empriskcvx{\sw}} &%
\in \bigO{LR\sqrt{\frac{ C_\textnormal{clas}(n) + \log(\nicefrac{n}{\delta}) }{n}}%
}\bigg\}  \nonumber \\
& ~\text{and}~\mathrm{Bad~Event}_1\bigg) \geq 1-\delta-\beta_1,  \nonumber\\
\Pr\bigg(\bigg\{\cEE{\trainset} \sbra{\Popriskcvx{\sw}- \Empriskcvx{\sw}} &%
\in \bigO{LR\sqrt{\frac{ C_\textnormal{cond}(n)+ \log(\nicefrac{n}{\delta}) }{n}}%
}\bigg\} \nonumber\\ 
&~\text{and}~\mathrm{Bad~Event}_2\bigg) \geq 1-\delta-\beta_2. \nonumber
\]
Thus, we conclude that with probability at least $1-\delta-\max\{\beta_1,\beta_2\}$ (over the randomness in the training set) for every $\sigma$ we have
\[
\nonumber
\max\{LR\sqrt{\frac{\complexity{n} + \log(\nicefrac{n}{\delta}) }{n}} , \resemp(\ww_T) + \respop(\ww_T)\} \in \Omega(1),
\]
as was to be shown.
\section{Proof of \cref{th:ecmi-failining-example}}
\label{pf:ecmi-fails}
Let $\dimcvx \in \Naturals$ and $\dataspace = \lbrace \coorvec{i} : i \in \dimcvx \rbrace$, that is, the set of all coordinate vectors in $\lbrace 0, 1 \rbrace^\dimcvx$, where $$\coorvec{i} = (\underbrace{0, \ldots, 0}_{i-1 \textnormal{ times}}, 1, \underbrace{0, \ldots, 0}_{\dimcvx-i \textnormal{ times}}).$$
Let the data distribution on the input be the uniform distribution, that is $\Dist = \textnormal{Uniform}(\dataspace)$. Then, we consider the simple convex, $1$-Lipschitz loss function $\losscvx(w,z) = - \langle w, z \rangle$. Moreover, we consider that the weights $w$ are in a unit ball on $\Reals^\dimcvx$, that is $\parspace = \lbrace w : \norm{w} \leq 1 \rbrace$. Therefore, the problem is in the CLB class.

Next, we analyze the dynamics of GD. The empirical loss is given by 
\begin{equation*}
    \Empriskcvx{w}=-\inner{w}{\empmean},
\end{equation*}
where $\empmean$ is the empirical mean of the instances in the training set, i.e., $\empmean = \frac{1}{n} \sum_{i=1}^n Z_i$. Also, we have that $\partial \Empriskcvx{w} = - \empmean$ for all $w \in \parspace$. Considering the update rule of GD, i.e. $W_{t+1} = \proj(W_t + \eta \empmean)$, one can show by induction that
\begin{equation}
    \label{eq:dynamic-gd-linear}
    \ww_t = 
    \begin{cases}
        \eta t \empmean  & \eta t \norm{\empmean}\leq 1 \\
        \frac{\empmean}{\norm{\empmean}} &\text{ Otherwise}
    \end{cases}.
\end{equation}

Now consider the $\supersample$-measurable random variable $E$ that is equal to one if and only if all the data instances in the supersample are distinct. That is
\begin{equation}
\label{eq:event-distinct}
    E = \indic{ \tilde{Z}_{u,i} \neq \tilde{Z}_{v,j} \textnormal{ for all } i,j \in \range{n} \textnormal{ and all } u,v \in \{ 0, 1 \}}.
\end{equation}

As in the \emph{birthday paradox problem}~\citep[Sec~5]{mitzenmacher2017probability}, we may bound the probability that $E=1$ as follows
\begin{align}
    \Pr(E=1) &= \prod_{k = 0}^{2n -1} \Big(1 - \frac{k}{\dimcvx} \Big)
    \nonumber \\
    &\geq \Big(1 - \frac{2n-1}{\dimcvx} \Big)^{2n-1}, \label{eq:birthday_dim}
\end{align}
This way, we may engineer a dimension $\dimcvx$ for which $\Pr(E=1) \geq c$ for all $n \geq 1$, where $c$ is a constant probability, independent of $n$. Solving for~\cref{eq:birthday_dim} results in
\begin{equation*}
    \dimcvx \geq \frac{2n - 1}{1 - c^{1/(2n-1)}}.
\end{equation*}
For instance, for $c = 0.1$, a dimension $\dimcvx = 2n^2$ suffices, and therefore $\Pr(E=0) \leq 0.9$.
Now, we are ready to study what happens to both the individual conditional mutual information $\minf{W_T;U_i|\tilde{Z}_{0,i},\tilde{Z}_{1,i}}$ and the evaluated mutual information $\ecmigd$ in this particular setting.

\subsection{Individual conditional mutual information}
\label{app:icmi}

Note that the individual CMI may be written as follows
\begin{align}
    \minf{W_T;U_i|\tilde{Z}_{0,i},\tilde{Z}_{1,i}}  &= \entr{U_i \vert \tilde{Z}_{0,i},\tilde{Z}_{1,i}} - \entr{U \vert W_T,\tilde{Z}_{0,i},\tilde{Z}_{1,i}} \nonumber \\
    &= \entr{U_i} - \entr{U_i \vert W_T,\tilde{Z}_{0,i},\tilde{Z}_{1,i}} \nonumber \\
    &= \log 2  - \entr{U_i \vert W_T,\tilde{Z}_{0,i},\tilde{Z}_{1,i}}, \label{eq:birthday_icmi_bound}
\end{align}
where the second and third equations follow from $U_i \perp (\tilde{Z}_{0,i},\tilde{Z}_{1,i})$ and $H(U_i) = \log 2$, respectively. 
Then, we may use Fano's inequality to bound $\entr{U_i \vert W_T,\tilde{Z}_{0,i},\tilde{Z}_{1,i}}$ and obtain the desired result.\footnote{To achieve this, a previous version of the paper instead relied on the false equivalence $\entr{U_i \vert W_T,\tilde{Z}_{0,i},\tilde{Z}_{1,i}} = \entr{U_i \vert W_T,\tilde{Z}_{0,i},\tilde{Z}_{1,i},E}$. This is not correct since $E$ is not $(W_T,\tilde{Z}_{0,i},\tilde{Z}_{1,i})$-measurable. The proof in this version instead relies on the connection between mutual information and Fano's inequality.} More precisely, Fano's inequality states that 
    \begin{align*}
        \entr{U_i \vert W_T,\tilde{Z}_{0,i},\tilde{Z}_{1,i}}
        &\leq \binaryentr{\Pr(U_i \neq \hat{U}_i)},
        \label{eq:birthday_ent_icmi_bound}
    \end{align*}
    for every estimator $\hat{U}_i(W_T, \tilde{Z}_{0,i}, \tilde{Z}_{1,i})$ and where $\binaryentr{\cdot}$ is the binary entropy. Notice that $\hat{U}_i$ is a function of $W_T, \tilde{Z}_{0,i},$ and $\tilde{Z}_{1,i}$. Therefore, showing that $\Pr(U_i \neq \hat{U}_i) < 0.5$ is a constant independent of $n$ ensures that
    \begin{equation}       
        \minf{W_T;U_i|\tilde{Z}_{0,i},\tilde{Z}_{1,i}}  \geq \log 2  - \binaryentr{\Pr(U_i \neq \hat{U}_i)} \in \Omega(1)
        \label{eq:icmi_birthday_bounded}
    \end{equation}
    and completes the proof.

From~\cref{eq:dynamic-gd-linear}, we can see that the non-zero coordinates of $W_T$ are precisely the coordinates of the training samples. That is, if $\tilde{Z}_{U_i,i} = \coorvec{k}$, then $W_T(k) \neq 0$. 
Therefore, under the event $E = 1$ defined in \cref{eq:event-distinct}, one can precisely determine if sample $\tilde{Z}_{0,i}$ or sample $\tilde{Z}_{1,i}$ was used for training after observing $W_T$ since the samples are all distinct. 
In other words, one can completely determine $U_i$ from $(W_T, \tilde{Z}_{0,i}, \tilde{Z}_{1,i})$. More precisely, consider a realization in which $\tilde{Z}_{0,i} = \coorvec{k}$ and $\tilde{Z}_{1,i} = \coorvec{l}$. Then, the estimator $\hat{U}_i(W_T, \tilde{Z}_{0,i}, \tilde{Z}_{1,i})$ is defined  as $\hat{U}_i(W_T, \tilde{Z}_{0,i}, \tilde{Z}_{1,i}) = 0$ if $W_T(k) \neq 0$ and $W_T(l) = 0$; $\hat{U}_i(W_T, \tilde{Z}_{0,i}, \tilde{Z}_{1,i}) = 1$ if $W_T(k) = 0$ and $W_T(l) \neq 0$; otherwise in the case that $W_T(k) \neq 0$ and $W_T(l) \neq 0$ , let  $\hat{U}_i(W_T, \tilde{Z}_{0,i}, \tilde{Z}_{1,i})$ be a Bernoulli random variable with parameter $\nicefrac{1}{2}$ independent of $\supersample$ and $U$. This estimator has a probability of error equal to $0$ given the event $E=1$. Therefore, the probability of error is
\begin{align*}
    \Pr (U_i \neq \hat{U}_i) &= \Pr(E = 0)\cPr{E=0}{U_i \neq \hat{U}_i} + \Pr(E = 1) \cPr{E=1}{U_i \neq \hat{U}_i} \\
    &= \Pr(E = 0)\cPr{E=0}{U_i \neq \hat{U}_i} \\
    &\leq 0.9 \cdot \cPr{E=0}{U_i \neq \hat{U}_i},
\end{align*}
where the last line follows from the construction. Next consider the following random variables 
\begin{equation*}
    G_i = \indic{ \tilde{Z}_{0,i} \neq \tilde{Z}_{1,i} \textnormal{ and } \tilde{Z}_{1-U_i, i} \neq \tilde{Z}_{U_j,j}  \textnormal{ for all } j \neq i \in \range{n} },
\end{equation*}
which describe the situation where the given samples $\tilde{Z}_{0,i}$ and $\tilde{Z}_{1,i}$ are distinct and the sample that is not chosen is also distinct from all other samples in the dataset $S$, even when some of these samples are equal between themselves or to the chosen sample $\tilde{Z}_{U_i,i}$ (e.g. when $E=0$). Therefore, given the event $E=0$ and $G_i = 1$, the estimator $\hat{U}_i$ still has a probability of error equal to zero. Hence, similar to before we may bound the probability of error of the estimator as
\begin{align*}
    \Pr (U_i \neq \hat{U}_i) &\leq 0.9 \cdot \Big( \cPr{E=0}{G_i = 0}\cPr{E = 0, G_i = 0}{U_i \neq \hat{U}_i} + \cPr{E = 0}{G_i = 1} \cPr{E = 0, G_i = 1}{U_i \neq \hat{U}_i}\Big) \\
    &\leq 0.9 \cdot \cPr{E = 0, G_i = 0}{U_i \neq \hat{U}_i}.
\end{align*}

Next, we claim that under the event where $E=0$ and $G_i=0$, the estimator $\hat{U}_i$ is a Bernoulli random variable with parameter $1/2$. Consider a realization in which $U_i = u$, $\tilde{Z}_{0,i} = \coorvec{k}$, and $\tilde{Z}_{1,i} = \coorvec{l}$. Then, we claim that under the event $E=0$ and $G_i=0$, $W_T(k)\neq 0$ and $W_T(l)\neq 0$. The reason is under this event, the following cases may happen: 1) $\tilde{Z}_{0,i}= \tilde{Z}_{1,i}$ or 2) $\tilde{Z}_{0,i} \neq \tilde{Z}_{1,i}$ but there exists another sample in the training set which is equal to $\tilde{Z}_{1-u,i}$. It is easy to see that in these two cases $W_T(k)\neq 0$ and $W_T(l)\neq 0$.

Thus, we conclude that, given $E=0$ and $G_i=0$, we have that $\cPr{E=0,G_i=0}{U_i \neq \hat{U}_i} = \nicefrac{1}{2}$. This is true since $\hat{U}_i$ is a Bernoulli random variable with parameter $\nicefrac{1}{2}$ independent of $U$ and $\tilde{S}$. Therefore, we have that $\Pr (U_i \neq \hat{U}_i) \leq 0.45$, which completes the proof as per \cref{eq:icmi_birthday_bounded}.

\subsection{Evaluated conditional mutual information}

Note that the evaluated CMI may be written as follows
\begin{align}
    \ecmigd &= \minf{\lossveccvx;U\vert \supersample} \nonumber \\
    &= \entr{U\vert \supersample} - \entr{U\vert \supersample,\lossveccvx} \nonumber \\
    &= \entr{U} - \entr{U \vert \lossveccvx, \supersample} \nonumber \\
    &= n \log 2  - \entr{U \vert \lossveccvx, \supersample}, \label{eq:birthday_ecmi_bound}
\end{align}
where the third and fourth equations follow from $U \perp \supersample$ and $H(U) = n \log 2$, respectively.

Then, as in the previous subsection, the proof relies in the fact that $U$ can be completely determined by the loss vector $\lossveccvx$ under the event $E = 1$. More precisely, note that $\lossveccvx_{u,i} = \losscvx(W_T, \tilde{Z}_{u,i}) = - \langle W_T, \tilde{Z}_{u,i} \rangle$. Also, remember from the previous subsection that the non-zero coordinates of $W_T$ are precisely the non-zero coordinates of the samples that are used for training. Therefore, under the event $E=1$, $\lossveccvx_{u,i} = 0$ if and only if $\tilde{Z}_{u,i}$ was not used for training and therefore $Z_i = \tilde{Z}_{1-u,i}$.  Hence, one can completely determine $U$ from $\lossveccvx$ or, equivalently, $\EE \sbra{\centr{\lossveccvx, \supersample, E}{U} \indic{E=1}} = 0$. We may use this fact to bound $\entr{U \vert \lossveccvx, \supersample}$ and obtain the desired result. Namely, 
\begin{align}
   \entr{U \vert \lossveccvx, \supersample}
   &= \entr{U \vert \lossveccvx, \supersample,E} \nonumber \\
   &= \EE \sbra{\centr{\lossveccvx, \supersample,E}{U} \indic{E=1}} + \EE \sbra{\centr{\lossveccvx, \supersample,E}{U} \indic{E=0}}, \nonumber \\
   &\leq n \cdot 0.9 \log 2 \label{eq:birthday_ent_ecmi_bound}
\end{align}
where the first line follows since $E$ is $\supersample$-measurable, and the last inequality follows from upper bounding $\centr{\supersample, \lossveccvx,G}{U}$ by $n \log 2$ and the facts that $\EE \sbra{\centr{\lossveccvx, \supersample, E}{U} \indic{E=1}} = 0$ and $\Pr(E=0) \leq 0.9$. 

Finally, combining~\cref{eq:birthday_ecmi_bound} and~\cref{eq:birthday_ent_ecmi_bound} results in
\begin{equation*}
    \ecmigd  \geq n \log 2  - n \cdot 0.9 \log 2 \in \Omega(n),
\end{equation*}
and completes the proof.

\section{Helper Lemmata}
\begin{lemma}[{\citealp[][Ex.~3.3.7]{vershynin2018high}}]
\label[lemma]{lem:polar-gauss}
Let $X \dist \Normal(0,\id{d})$. Let us represent $X=R\theta$ where $R=\norm{X}$ and $\theta=X/\norm{X}$. Then, $R$ and $\theta$ are independent random variables. Also, $\theta$ is uniformly distributed on the Euclidean sphere $\unitsphere$ with the center at the origin.
\end{lemma}

\begin{lemma}[{\citealp[][Lemma~1]{laurent2000adaptive}}]
Consider random vector $X \dist \Normal(0,\id{d})$. Then, 
\[
\nonumber
\Pr\bigg(\sum_{i=1}^{d} a(i) X(i)^2 &\geq \norm{a}_1 + 2\norm{a}_2 \sqrt{t} + 2 \norm{a}_{\infty} t \bigg) \leq \exp(-t) \textnormal{ and}\\ 
\nonumber
\Pr\bigg(\sum_{i=1}^{d} a(i) X(i)^2 &\geq \norm{a}_1 - 2\norm{a}_2 \sqrt{t}\bigg) \leq \exp(-t)
\]
\end{lemma}

\begin{corollary}
\label[corollary]{cor:norm-gauss}
Let $\sigma \in \Reals$, $\delta \in (0,1)$, $d\in \Naturals$, and  $d \geq \log \frac{2}{\delta} $. Consider $X \dist \Normal(0,\sigma^2 \id{d})$, then 
\[
&\Pr\Big( d\sigma^2(1-2\sqrt{\frac{\log(\nicefrac{2}{\delta})}{d}}) \leq \norm{X}^2 \leq d\sigma^2(1+4\sqrt{\frac{\log(\nicefrac{2}{\delta})}{d}} ) \Big) \geq 1-\delta, \nonumber \\
&\Pr\big( \norm{X}\leq \sqrt{(1-\alpha) d \sigma^2} \big)\leq 2\exp\Big(-\frac{d \alpha^2}{4}\Big) ~ \text{for $\alpha \in [0,1]$, and} %
\nonumber
\\
&\Pr( \norm{X} \geq \sqrt{(1+\beta) d \sigma^2} )\leq 2\exp\Big(-\frac{d \beta^2}{16}\Big) ~ \text{for $\beta \geq 0$.}
\nonumber
\]
\end{corollary}
\begin{lemma}[{\citealp[][Thm.~2.10.1]{cover2012elements}}]
\label[lemma]{lem:fano}
Let $X$ and $Y$ be discrete random variables. Then
\[
\nonumber
\entr{X\vert Y} \leq \binaryentr{\proberror}+ \proberror \entr{X}\leq 1+ \proberror \entr{X},
\]
where $\proberror = \Pr(\Psi(Y)\neq X)$ for any (possibly randomized) estimator $\Psi$ of $X$ using $Y$  (See also \citealt{fano1952class}).
\end{lemma}

\begin{lemma}
\label[lemma]{lem:lipschitz-max}
Let $d\in \Naturals_{+}$. Let $g: \Reals^d \to \Reals$ be defined as $g(x)=\max\{\max_{i\in \range{d}}\{x(i)\},0\}$. Then, $g$ is $1-$Lipschitz.
\end{lemma}
\begin{proof}
Let $x\in \Reals^d$ and $\Delta\in \Reals^d$. Let $\arg\max_{i\in \range{d}}{\{x(i)+\Delta(i)\}}=i^\star$ and $\arg\max_{i\in \range{d}}{\{x(i)\}}=j^\star$ (break ties arbitrary). Then,
\begin{equation*}
  g(x+\Delta)- g(x) =  \begin{cases}
  -x(j^\star)\leq 0  & x(i^\star)+\Delta(i^\star)\leq 0 ~ \text{and} ~  x(j^\star)>0 \\
  x(i^\star) + \Delta(i^\star) - x(j^\star) < \Delta(i^\star) & x(i^\star)+\Delta(i^\star)> 0 ~ \text{and} ~  x(j^\star)>0\\
  0   & x(i^\star)+\Delta(i^\star)\leq 0 ~ \text{and} ~ x(j^\star)\leq 0\\
  x(i^\star)+\Delta(i^\star)\leq \Delta(i^\star)    & x(i^\star)+\Delta(i^\star)> 0 ~ \text{and} ~ x(j^\star)\leq 0
\end{cases}
\end{equation*}
The last case follows because $x(i^\star)\leq x(j^\star)\leq 0$, therefore, $x(i^\star)+\Delta(i^\star)\leq \Delta(i^\star)$. Thus, $|g(x+\Delta)- g(x)|\leq \norm{\Delta}$, as was to be shown. 
\end{proof}

\begin{lemma}
\label[lemma]{lem:gd-last-iterate}
Let $\losscvx$ be a convex and $L-$Lipschitz loss function, and $\parspace$ be a convex and compact \parameterspace space with bounded diameter $R$. Let $\{w_t\}_{t\in \range{T}}$ denote the output of GD algorithm with a constant step size $\eta$. Then,  we have
\[
\nonumber
\losscvx(w_T) - \min_{w\in \parspace}\losscvx(w) \leq \frac{R^2}{2\eta T}+ \frac{(\log(T)+2)\eta L^2}{2}
\]
\end{lemma}
\begin{proof}
Let $g_t \in \partial f(w_t)$. From \citep[][Thm.~2]{lastiterate},
\[
\nonumber
\losscvx(w_T)-\min_{w\in \parspace}\losscvx(w) \leq \frac{1}{T}\sum_{t=1}^{T}(\losscvx(w_t)-\min_{w\in \parspace}\losscvx(w))+\frac{1}{2}\sum_{k=1}^{T-1}\frac{1}{k(k+1)}\sum_{t=T-k}^{T}\eta \norm{g_t}^2.
\]
Since $\norm{g_t}\leq L$, the second term can be upper bounded by $\frac{\eta L^2}{2}\sum_{k=1}^{T-1}\frac{1}{k}$. Then, by the well-known bounds on the Harmonic numbers we have $\frac{\eta L^2}{2}\sum_{k=1}^{T-1}\frac{1}{k}\leq \frac{\eta L^2}{2}(\log(T-1)+1)\leq \frac{\eta L^2}{2}(\log(T)+1)$. For the first term, from \citep[][Thm.~3.2]{bubeck2015convex}, we have $ \frac{1}{T}\sum_{t=1}^{T}(\losscvx(w_t)-\min_{w\in \parspace}\losscvx(w))\leq \frac{R^2}{2\eta T}+\frac{\eta L^2}{2}$. Combining these two upper bounds proves the lemma.
\end{proof}
\begin{lemma}
\label[lemma]{lem:reverse-markov}
Let $X$ be a random variable, $\tilde{m} \geq 0$ be a constant such that $0\leq X \leq \tilde{m}$ a.s. Let $m \in \Reals$ be such that $0< m \leq \tilde{m}$. Then, for every $0\leq a < m$,  we have
\[
\nonumber
\Pr(X > a) \geq \frac{\EE[X]-a - (\tilde{m}-m)\Pr(X > m)}{m - a}.
\]
\end{lemma}
\begin{proof}
The following holds almost surely:
\[
\nonumber
X \leq a \indic{X\leq a} + m \indic{a < X\leq m} + \tilde{m} \indic{ m < X}. 
\]
Taking an expectation concludes the proof.
\end{proof}
\begin{lemma}[{\citealp[][Lem.~2]{9531956}}]
\label[lemma]{lem:mixture-bound}
Let $M \in \Naturals$ and $\mathcal{Y}$ be a measurable space. Let also $P \in \ProbMeasures{\mathcal{Y}}$ and $Q_i \in \ProbMeasures{Y}$ for all $i \in \range{M}$ be probability measures. If $\alpha_i \in (0,1)$ such that $\sum_{i=1}^M \alpha_i = 1$, 
\begin{equation*}
    \KL{P}{\sum_{i=1}^M \alpha_i Q_i} \leq \min_{i \in \range{M}} \Big \{ \KL{P}{Q_i} -\log(\alpha_i) \Big \}.
\end{equation*}
\end{lemma}
\begin{lemma}
\label[lemma]{lem:mixture-num}
Let $\mathcal{Y}$ be a measurable space. Let $M \in \Naturals$ and $P_i \in \ProbMeasures{\mathcal{Y}}$ for $i \in \range{M}$ be $M$ probability measures. Then, for every $i \in \range{M}$, we have
\[
\nonumber
\KL{P_i}{\sum_{j=1}^{M}\frac{1}{M} P_j} \leq \log(M).
\]
\end{lemma}
\begin{proof}
A direct application of \cref{lem:mixture-bound} gives us the result.
\end{proof}

\end{document}

%% file: headers.tex
\usepackage[utf8]{inputenc} 
\usepackage{amsmath, amssymb,bm, cases, mathtools}%
\usepackage{verbatim}
\usepackage{graphicx}\graphicspath{{figures/}}
\usepackage{multicol}
\usepackage{mathrsfs} 
\usepackage{algorithm}
\usepackage{algorithmicx}
\usepackage[noend]{algpseudocode}
\usepackage{array,booktabs,arydshln}%

\usepackage[T1]{fontenc}
\usepackage{inconsolata}
\usepackage{dsfont}

\usepackage{listings} %

\usepackage{enumitem}
\usepackage{booktabs}       %
\usepackage{nicefrac}       %

\usepackage{lineno}
\usepackage{bbm}

\usepackage{datetime}

\DeclareMathAlphabet\EuRoman{U}{eur}{m}{n}
\SetMathAlphabet\EuRoman{bold}{U}{eur}{b}{n}

\renewcommand{\appendixname}{}

\usepackage{xparse}
\usepackage{xstring}
\usepackage{xspace}
\usepackage{mleftright}

%% file: defs.tex
\def\[#1\]{\begin{align}#1\end{align}}
\def\*[#1\]{\begin{align*}#1\end{align*}}

\newcommand{\minf}[1]{I(#1)}

\newcommand{\entr}[1]{\mathrm{H}(#1)}

\newcommand{\centr}[2]{\mathrm{H}^{#1}(#2)}

\newcommand{\cPr}[2]{\Pr^{#1}(#2)}

\newcommand{\complexity}[1]{\mathsf{complexity}(#1)}

\newcommand{\trainset}{S_n}

\newcommand{\parspace}{\mathcal{W}}
\newcommand{\Dist}{\mathcal D}
\newcommand{\dataspace}{\mathcal Z}
\newcommand\optparen[1]{\ifthenelse{\equal{#1}{}}{}{(#1)}}

\newcommand{\dist}{\ \sim\ }

\newcommand{\Naturals}{\mathbb{N}}

\newcommand{\Reals}{\mathbb{R}}

\newcommand{\as}{\textrm{a.s.}}

\newcommand{\grad}{\nabla}

\newcommand{\AND}{\wedge}
\newcommand{\OR}{\vee}
\newcommand{\NOT}{\neg}

\newcommand{\dee}{\mathrm{d}}

\newcommand{\inspace}{\mathcal X}

\DeclareMathOperator*{\newlim}{\mathrm{lim}\vphantom{\mathrm{infsup}}}
\DeclareMathOperator*{\newmin}{\mathrm{min}\vphantom{\mathrm{infsup}}}
\DeclareMathOperator*{\newmax}{\mathrm{max}\vphantom{\mathrm{infsup}}}
\DeclareMathOperator*{\newinf}{\mathrm{inf}\vphantom{\mathrm{infsup}}}
\DeclareMathOperator*{\newsup}{\mathrm{sup}\vphantom{\mathrm{infsup}}}
\renewcommand{\lim}{\newlim}
\renewcommand{\min}{\newmin}
\renewcommand{\max}{\newmax}
\renewcommand{\inf}{\newinf}
\renewcommand{\sup}{\newsup}

\newcommand{\ProbMeasures}[1]{\mathcal{M}_1(#1)}

\renewcommand{\Pr}{\mathbb{P}}
\def\EE{\mathbb{E}}

\newcommand{\defn}[1]{\textit{#1}}

\newcommand{\cF}{\mathcal F}

\newcommand{\equaldist}{\overset{d}{=}}

\newcommand{\norm}[1]{\lVert #1 \rVert}

\newcommand{\iid}{i.i.d.}

\newcommand{\KLname}{\mathrm{KL}}

\newcommand{\KL}[2]{\KLname(#1 \| #2)}

\newcommand{\Normal}{\mathcal N}

\newcommand{\Var}{\text{Var}}

\newcommand{\XX}{\Reals^{\idim}}
\newcommand{\YY}{K}

\newcommand{\id}[1]{\mathbb{I}_{#1}}

\newcommand{\Alg}{\mathcal{A}}

\newcommand{\unif}[1]{\text{Unif}(#1)}

\newcommand{\bernoulli}{\text{Ber}}

\newcommand{\lcrx}[4][{-1}]{
	\IfEq{#1}{-1}{\mleft #2 {{{{#3}}}} \mright #4}{
   	\IfEq{#1}{0}{#2 {{{{#3}}}} #4}{
	\IfEq{#1}{1}{\bigl #2 {{{{#3}}}} \bigr #4}{
	\IfEq{#1}{2}{\Bigl #2 {{{{#3}}}} \Bigr #4}{
	\IfEq{#1}{3}{\biggl #2 {{{{#3}}}} \biggr #4}{
	\IfEq{#1}{4}{\Biggl #2 {{{{#3}}}} \Biggr #4}{
    \GenericWarning{"4th argument to lcrx must be -1, 0, 1, 2, 3, or 4"}
    }}}}}}}

\newcommand{\inner}[3][{-1}]{\lcrx[#1] \langle {{#2},{#3}} \rangle}

\newcommand{\sbra}[2][{-1}]{\lcrx[#1] [ {#2} ] }

\newcommand{\rnderiv}[2]{\frac{\text{d} #1}{\text{d} #2}}

\newcommand{\cEE}[1]{\EE^{#1}}
\newcommand{\ww}{W}

\newcommand{\indep}{\mathrel{\perp\mkern-9mu\perp}}

\newcommand{\dminf}[2]{I^{#1} (#2)}

\newcommand{\indic}[1]{ \mathds{1}[#1]}

\newcommand{\supersample}{\tilde{S}} %

\newcommand{\range}[1]{ [#1] }

\newcommand{\binaryentr}[1]{\mathrm{H}_{b}(#1)}

\newcommand{\bigO}[1]{\mathcal O\mleft(#1\mright)}

%% file: defs-sco.tex
\renewcommand{\defn}[1]{\emph{#1}}

\newcommand{\EGE}{\ensuremath{\mathrm{EGE}_{\Dist}(\Alg_n)}}

\newcommand{\iomi}{\ensuremath{\mathrm{IOMI}_{\Dist}(\Alg_n)}\xspace}

\newcommand{\cmi}{\ensuremath{\mathrm{CMI}_{\Dist}(\Alg_n)}\xspace}

\newcommand{\ecmi}{\ensuremath{e\mathrm{CMI}_{\Dist}(\losscvx(\Alg_n))}\xspace} %

\newcommand{\ecmigd}{\ensuremath{e\mathrm{CMI}_{\Dist}(\losscvx(\mathrm{GD}_n))}\xspace}

\newcommand{\Empriskcvx}[1]{\hat{\mathrm{F}}_{S_n}(#1)}
\newcommand{\gdalg}{\mathrm{GD}_n(\trainset)}
\newcommand{\gdalgnodata}{\mathrm{GD}}
\newcommand{\Popriskcvx}[1]{\mathrm{F}_{\Dist}(#1)}
\newcommand{\losscvx}[0]{f} %
\newcommand{\proj}[0]{\Pi_{\parspace}} %
\newcommand{\bigOmega}[1]{\Omega\mleft(#1\mright)} %
\newcommand{\bigTheta}[1]{\Theta\mleft(#1\mright)} %
\newcommand{\dimcvx}[0]{d} %
\newcommand{\sw}{\tilde{W}} %
\newcommand{\unitsphere}{\ensuremath{\mathrm{S^{(d-1)}}}} %
\newcommand{\badset}{\mathsf{B}} %
\newcommand{\proberror}{\mathsf{P_e}} %
\newcommand{\empmean}{\hat{\mu}} %
\newcommand{\subgrad}[1]{\partial #1}
\newcommand{\coorvec}[1]{\mathsf{e}({#1})}
\newcommand{\lossveccvx}{F}
\newcommand{\lossveccvxlower}{\vec{\losscvx}}
\newcommand{\EER}{\ensuremath{\Popriskcvx{\Alg_n(\trainset)}-\inf_{w\in \parspace}\Popriskcvx{w}}} 
\DeclareMathOperator*{\argmin}{arg\,min} %
\newcommand{\clr}{\ensuremath{\mathcal{C}_{L,R}}\xspace} %
\newcommand{\EGEgd}{\ensuremath{\mathrm{EGE}_{\Dist}(\mathrm{GD}_n)}}

\newcommand{\mi}{\ensuremath{\mathrm{IM}_{\Dist}(\Alg_n)}\xspace}
\newcommand{\migd}{\ensuremath{\mathrm{IM}_{\Dist}(\mathrm{GD}_n)}\xspace}
\newcommand{\respop}{\ensuremath{\Delta_{\sigma}}\xspace}
\newcommand{\resemp}{\ensuremath{\hat{\Delta}_{\sigma}}\xspace}

%% file: authors.tex
\newcommand\blfootnote[1]{%
  \begingroup
  \renewcommand\thefootnote{}\footnote{#1}%
  \addtocounter{footnote}{-1}%
  \endgroup
}
\altauthor{%
 \Name{Mahdi Haghifam$^*$} \Email{mahdi.haghifam@mail.utoronto.ca} \\
 \addr University of Toronto, Vector Institute
 \AND
 \Name{Borja {Rodríguez-Gálvez}$^*$} \Email{borjarg@kth.se}\\
 \addr KTH Royal Institute of Technology
 \AND
 \Name{Ragnar Thobaben} \Email{ragnart@kth.se}\\
 \addr KTH Royal Institute of Technology
 \AND
 \Name{Mikael Skoglund} \Email{skoglund@kth.se}\\
 \addr KTH Royal Institute of Technology
 \AND
 \Name{Daniel M. Roy} \Email{daniel.roy@utoronto.ca}\\
 \addr University of Toronto, Vector Institute
 \AND
 \Name{Gintare Karolina Dziugaite} \Email{gkdz@google.com}\\
 \addr Google Research, Mila, McGill
}